\documentclass[twoside]{article}

\usepackage[accepted]{aistats2025}

\usepackage[utf8]{inputenc}
\usepackage[dvipsnames]{xcolor} 
\usepackage{microtype}
\usepackage{scalerel}
\usepackage{graphicx}
\usepackage{subfig}
\usepackage{wrapfig}
\usepackage{float}
\usepackage{tikz}
\usetikzlibrary{3d}
\usepackage{pgfplots}
\pgfplotsset{compat=newest}
\usepackage{comment}
\usepackage{times}
\usepackage{lipsum}
\usepackage{enumitem}
\usepackage{algorithm}
\usepackage{algpseudocode}

\usepackage{amsmath}
\usepackage{amsfonts}
\usepackage{amssymb}
\usepackage{amsthm}
\usepackage{bbm}
\usepackage{mathabx}

\usepackage{thmtools}
\usepackage{thm-restate}

\usepackage{appendix}

\usepackage{todonotes}

\usepackage{enumitem}
\newlist{defenum}{enumerate}{1}
\setlist[defenum]{label=\upshape\theassumption.\arabic*}

\usepackage[round]{natbib}
\usepackage{breakcites}
\usepackage[hidelinks]{hyperref} %hidelinks
\usepackage[capitalize,noabbrev]{cleveref}
\hypersetup{colorlinks={true},linkcolor={BrickRed},citecolor={CadetBlue}}

\crefname{prop}{Proposition}{Propositions}
\crefname{ass}{Assumption}{Assumptions}
\crefname{ex}{Example}{Examples}
\crefname{thm}{Theorem}{Theorems}
\crefname{cor}{Corollary}{Corollaries}

\definecolor{fannyComment}{RGB}{150, 10, 250}
\newcommand{\noteby}[3]{{\colorbox{#2}{\bfseries\sffamily\scriptsize\textcolor{white}{#1}}}{\textcolor{#2}{\sf\small\textit{#3}}}}
\newcommand{\fy}[1]{\noteby{Fanny}{fannyComment}{ #1}}

\definecolor{juliaComment}{rgb}{0.99, 0.93, 0.0}
\definecolor{alexComment}{rgb}{0.120, 0.5, 0.26}
\definecolor{tobiComment}{rgb}{0.86, 0.14, 0.26}
\definecolor{filipComment}{rgb}{0.0, 0.00, 1.0}

\newcommand{\tw}[1]{\noteby{Tobi}{tobiComment}{ #1}}

\definecolor{tabblue}{rgb}{0.12156862745098039, 0.4666666666666667, 0.7058823529411765}
\definecolor{taborange}{rgb}{1.0, 0.4980392156862745, 0.054901960784313725}
\definecolor{tabgreen}{rgb}{0.17254901960784313, 0.6274509803921569, 0.17254901960784313}
\definecolor{tabred}{rgb}{0.8392156862745098, 0.15294117647058825, 0.1568627450980392}
\definecolor{tabpurple}{rgb}{0.5803921568627451, 0.403921568627451, 0.7411764705882353}

\definecolor{population_color}{rgb}{0.4, 0.4, 0.4}
\definecolor{ts_color}{rgb}{0.12156862745098039, 0.4666666666666667, 0.7058823529411765}
\definecolor{direct_color}{rgb}{1.0, 0.4980392156862745, 0.054901960784313725}

\newcommand{\inner}[2]{\left\langle #1,#2\right\rangle}

\newtheorem*{theorem*}{Theorem}

\newtheorem{lemma}{Lemma}

\newtheorem{definition}{Definition}

\newtheorem{assumption}{Assumption}
\newtheorem{claim}{Claim}

\crefname{assumption}{Assumption}{Assumptions}
\crefname{proposition}{Proposition}{Propositions}

\newcommand{\mathset}[1]{\left\{#1\right\}}
\newcommand{\pr}[1]{\left(#1\right)}
\newcommand{\br}[1]{\left[#1\right]}
\newcommand{\abs}[1]{\left\lvert#1\right\rvert}
\newcommand{\norm}[1]{\left\|#1\right\|}
\newcommand{\setmid}[0]{\;\middle|\;}
\newcommand{\EE}[0]{\mathbb{E}}
\newcommand{\NN}[0]{\mathbb{N}}
\newcommand{\RR}[0]{\mathbb{R}}
\newcommand{\PP}[0]{\mathbb{P}}
\newcommand{\SSS}[0]{\mathbb{S}}

\newcommand{\cD}[0]{\mathcal{D}}
\newcommand{\cE}[0]{\mathcal{E}}
\newcommand{\cF}[0]{\mathcal{F}}

\newcommand{\cN}[0]{\mathcal{N}}

\newcommand{\cP}[0]{\mathcal{P}}

\newcommand{\cS}[0]{\mathcal{S}}

\newcommand{\cX}[0]{\mathcal{X}}
\newcommand{\cY}[0]{\mathcal{Y}}

\newcommand{\I}[0]{\mathbf{I}}

\newcommand{\M}[0]{\mathbf{M}}

\newcommand{\X}[0]{\mathbf{X}}

\newcommand{\bSigma}[0]{\mathbf{\Sigma}}

\newcommand{\e}[0]{\mathbf{e}}
\newcommand{\rbold}[0]{\mathbf{r}}
\newcommand{\etabold}[0]{\boldsymbol{\eta}}

\newcommand{\fX}[0]{\mathfrak{X}}

\newcommand{\eps}[0]{\varepsilon}

\newcommand{\betahat}[0]{{\widehat{\beta}}}

\newcommand{\vprime}{v'}
\newcommand{\vhat}{\widehat{v}}

\newcommand{\R}{\mathbb{R}}

\newcommand{\muhat}[0]{\widehat\mu}

\newcommand{\Thetatilde}[0]{\widetilde{\Theta}}

\newcommand{\Thetabold}[0]{\boldsymbol{\Theta}}

\newcommand{\omegabold}[0]{\boldsymbol{\omega}}

\DeclareMathOperator{\trOp}{tr}
\newcommand{\tr}[1]{\trOp\pr{#1}}

\DeclareMathOperator{\ImaOp}{im}
\newcommand{\Ima}[1]{\ImaOp\pr{#1}}

\DeclareMathOperator*{\argmin}{arg\,min}

\newcommand{\inda}[1]{1}
\newcommand{\ind}[0]{1}

\newcommand{\expect}{\mathop{\EE}}

\DeclareMathOperator{\cov}{cov}

\DeclareMathOperator{\lawOp}{law}
\newcommand{\law}[0]{{\lawOp}}

\newcommand{\identity}[0]{\mathbf{I}}
\newcommand{\weight}[0]{\lambda}
\newcommand{\weightTuple}[0]{\boldsymbol{\lambda}}
\newcommand{\eigenvalue}[0]{\lambda}
\newcommand{\eigenvalueMax}[0]{\eigenvalue_{\max}}
\newcommand{\eigenvalueMin}[0]{\eigenvalue_{\min}}
\DeclareMathOperator{\volOp}{vol}
\newcommand{\vol}[0]{{\volOp}}

\newcommand{\PF}[0]{\mathfrak{F}}
\newcommand{\PFhat}[0]{\widehat{\PF}}

\newcommand{\Loss}{\mathcal{L}}
\newcommand{\EmpLoss}{\widehat{\Loss}}
\newcommand{\objectiveindex}[0]{k}
\newcommand{\objectivevector}[0]{\boldsymbol{\Loss}}
\newcommand{\objectiveindexed}[0]{\Loss_\objectiveindex}

\DeclareMathOperator{\fairOp}{fair}
\newcommand{\Lfair}[0]{\Loss_{\fairOp}}
\newcommand{\weightfair}[0]{\weight_{\fairOp}}

\DeclareMathOperator{\riskOp}{risk}
\newcommand{\Lrisk}[0]{\Loss_{\riskOp}}
\newcommand{\weightrisk}[0]{\weight_{\riskOp}}

\newcommand{\scalarization}[0]{s_{\weightTuple}}

\newcommand{\scalarizationObjectiveComposition}[0]{(\scalarization\circ \objectivevector)}

\newcommand{\excessscalarization}[0]{\cE_{\weightTuple}}

\newcommand{\minimax}[0]{\mathfrak{M}_{\weightTuple}}

\newcommand{\strongConvexityParam}[0]{\mu}
\newcommand{\strongConvexityParamTuple}[0]{\boldsymbol{\mu}}

\newcommand{\smoothnessParam}[0]{\nu}

\DeclareMathOperator{\HVOp}{HV}
\newcommand{\hypervolume}[0]{{\HVOp}}

\newlength\bshft
\def\fakebold#1{\ThisStyle{\ooalign{$\SavedStyle\protect#1$\cr%
  \kern-\bshft$\SavedStyle\protect#1$\cr%
  \kern\bshft$\SavedStyle\protect#1$}}}

\newcommand{\distributionSetTuple}[0]{\boldsymbol{\cP}}

\newcommand{\distributionXxY}[0]{\PP}
\newcommand{\distributionXxYTuple}[0]{{\fakebold{\distributionXxY}}}
\newcommand{\distributionXxYindexed}[0]{\PP^\objectiveindex}
\newcommand{\distributionX}[0]{\PP_X}
\newcommand{\distributionXindexed}[0]{\PP^\objectiveindex_X}

\newcommand{\empdistributionXxY}[0]{\widehat{\PP}}
\newcommand{\empdistributionXxYTuple}[0]{\widehat{\distributionXxYTuple}}
\newcommand{\empdistributionXxYindexed}[0]{\widehat{\PP}^\objectiveindex}

\newcommand{\minimizer}[0]{\vartheta}
\newcommand{\parameter}[0]{\theta}

\newcommand{\minimizerWeighted}[0]{\minimizer_{\weightTuple}}

\newcommand{\parameterTuple}[0]{\boldsymbol{\parameter}}

\newcommand{\estimatedMinimizer}[0]{\widehat{\minimizer}}
\newcommand{\estimatedParameter}[0]{\widehat{\parameter}}

\newcommand{\estimatedParameterTuple}[0]{\widehat{\parameterTuple}}

\newcommand{\parameterSpace}[0]{\Theta}

\newcommand{\parameterSpaceTuple}[0]{\Thetabold}

\newcommand{\parameterSpaceLarge}[0]{\Thetatilde}

\newcommand{\parameterSpaceTupleLarge}[0]{\widetilde{\Thetabold}}

\newcommand{\unlabeledTuple}[0]{\omegabold}

\newcommand{\covariance}[0]{\Sigma}
\newcommand{\covarianceTuple}[0]{\bSigma}
\newcommand{\estimatedCovariance}[0]{\widehat{\covariance}}
\newcommand{\estimatedCovarianceTuple}[0]{\widehat{\covarianceTuple}}

\newcommand{\estimatedMinimizerWeighted}[0]{\widehat{\minimizer}_{\weightTuple}}

\DeclareMathOperator{\twostageOp}{ts}
\newcommand{\tsEstimatedMinimizerWeighted}[0]{\estimatedMinimizerWeighted^{\twostageOp}}

\DeclareMathOperator{\regularizedpluginOp}{dr}
\newcommand{\drEstimatedMinimizerWeighted}[0]{\estimatedMinimizerWeighted^{\regularizedpluginOp}}

\newcommand{\parameterizedfunction}[0]{f_\minimizer}

\begin{document}

\twocolumn[

\aistatstitle{Learning Pareto manifolds in high dimensions: How can regularization help?}%Semi-Supervised Bi-level Regularization for Multi-Objective Learning}

\aistatsauthor{ Tobias Wegel \And Filip Kova\v{c}evi\'{c} \And Alexandru \c{T}ifrea \And Fanny Yang}

\aistatsaddress{ ETH Zurich \And IST Austria \And ETH Zurich \And ETH Zurich } 

]

\begin{abstract}
Simultaneously addressing multiple objectives is becoming increasingly important in modern machine learning. At the same time, data is often high-dimensional and costly to label.
For a single objective such as prediction risk, conventional regularization techniques are known to improve generalization when the data exhibits low-dimensional structure like sparsity. However, it is largely unexplored how to leverage this structure in the context of \emph{multi-objective learning (MOL)} with multiple competing objectives. 
In this work, we discuss how the application of vanilla regularization approaches can fail, and propose a two-stage MOL framework that can successfully leverage low-dimensional structure. 
We demonstrate its effectiveness experimentally for multi-distribution learning and fairness-risk trade-offs.

%In modern machine learning, we increasingly would like to account for multiple objectives at the same time.  "This" is for example important in scientific applications, where a further/additional challenge is that data is often high-dimensional and expensive to label.
%Accounting for multiple objectives is a challenging yet omnipresent task in modern machine learning. 
%A second challenge/problem is that in scientific applications, 
%At the same time, in practice, 
%data is often high-dimensional and expensive to label. 

%estimator that provably yields improved performance in the presence of sparsity and unlabeled data \fy{hm our framework more general? - }

%Modern machine learning methods often have to rely on high-dimensional data that is expensive to label, while unlabeled data is abundant. When the data exhibits low-dimensional structure such as sparsity, conventional regularization techniques are known to improve generalization for a single objective (e.g., prediction risk). However, it is largely unexplored how to leverage this structure in the context of \emph{multi-objective learning (MOL)} with multiple competing objectives. In this work, we discuss how the application of vanilla regularization approaches can fail, and propose the first MOL estimator that provably yields improved performance in the presence of sparsity and unlabeled data. We demonstrate its effectiveness experimentally for multi-distribution learning and fairness-risk trade-offs.
\end{abstract}

%\tableofcontents

\section{INTRODUCTION}
\label{sec:intro}
As machine learning systems are employed more and more broadly, they are expected to excel in different aspects: 
The models should not only be accurate, but also 
%be \emph{trustworthy} in numerous ways: Besides being accurate, they should also be 
robust against perturbations  \citep{Szegedy2014Intriguing} or distribution shifts \citep{Rojas2018invariant}, %\citep{Quinonero2022dataset}
fairness-aware \citep{Hardt2016Equality}, %Pedreshi2008discrimination,
private \citep{Dwork2006Differential}, %,Agrawal2000privac 
interpretable \citep{Belle2021Principles} and, more recently, aligned with diverse human preferences  \citep{Ji2024Aialignmentcomprehensivesurvey}, to name just a few. 
Similarly, in data-driven computational design such as drug discovery or materials science, the discovered compound or material should usually satisfy multiple desired properties \citep{Ashby2000multi,Luukkonen2023artificial}.
However, it is well understood that in many settings, doing well on all objectives simultaneously can be inherently impossible and that a trade-off between them is unavoidable (see, e.g., \citet{Menon2018costfairness,Zhang2019Theoretically,Cummings2019Compatibility,
Sanyal2022unfairprivate,Guo2024Controllablepreferenceoptimizationcontrollable}).  %Wei2024Jailbroken Chzhen2022 %Schneider2023Multi, %Bagdasaryan2019Differential %Wang2024Aleatoric % Raghunathan2020Understanding % Chakraborty2024maxmin

%When each goal is expressed as an optimization objective, 
In the presence of such inherent trade-offs, 
%instead of minimizers, 
we are usually interested in learning models that lie on the \emph{Pareto front}\textemdash 
%\fy{tried to directly say what we want, "inducing" is a little too abstract} The trade-off between the objectives induces what is known as a \emph{Pareto front}: 
a concept studied in \emph{multi-objective optimization (MOO)}\textemdash 
where improving in one objective must come at the expense of another. 
%In the optimization literature, it is well-known  that 
Under some conditions \citep{Ehrgott2005Multi}, a point on the Pareto front of $K$ objectives $\Loss_1,\dots,\Loss_K$ can be recovered  by minimizing a scalarized objective, such as 
\begin{equation}
\label{eq:linear-and-Chebyshev-scalarization}
    \sum_{\objectiveindex=1}^K \weight_\objectiveindex \objectiveindexed \quad \text{or} \quad \max_{\objectiveindex\in[K]}\weight_\objectiveindex\objectiveindexed
\end{equation}
%\fy{would probably put in display equation, not inline}
using the appropriate weight vector $\weightTuple$ from the simplex.

In the context of machine learning, however, the Pareto front cannot be computed since the objectives are population-level and hence unknown.
Instead, the Pareto front has to be learned
from data \citep{Jin2008Pareto}\textemdash a problem that falls under the
general \emph{multi-objective learning (MOL)} paradigm.
A standard approach \citep{Lin2019Pareto,Hu2023Revisiting} to learning Pareto optimal points %(or the entire Pareto frontier) 
is to use an MOO method on the empirical versions of the losses $\EmpLoss_1,\dots,\EmpLoss_K$, for example by minimizing the empirical scalarized objective $\sum_{\objectiveindex=1}^K \weight_\objectiveindex \EmpLoss_\objectiveindex$ or $\max_{\objectiveindex\in [K]}\weight_\objectiveindex \EmpLoss_\objectiveindex$. The focus in such works is usually the optimization algorithm, rather than evaluating the generalization of the estimated Pareto frontier to the ``true'' Pareto frontier on test data. To date, there is only little work that formally characterizes how well such methods estimate the population-level Pareto set.
%the generalization of MOO methods estimating the Pareto set, such as \citet{Sukenik2024generalization,Cortes2020Agnostic}.
Further, to the best of our knowledge, the proposed methods in existing theoretical works \citep{Sukenik2024generalization,Cortes2020Agnostic} only recover the true Pareto front if there is sufficiently much labeled data for training the models.
%\fy{(see, e.g. cite only most cited 2 or so? rest appears in rel. work)}
%Deist2021Multiobjectivelearningpredictpareto,Deist2023Multiobjective
%\fy{i'd probably just cite the few that introduced it as long as we can cite the others more specifically?}
% \tw{Not sure there is/are paper that did this first, as its such an easy idea}
%Usually, MOL is done by minimizing a scalarized empirical objective $\scalarization\circ\Lhat$ instead of the scalarized population objective \citep{Jin2008Pareto,Cortes2020Agnostic,Súkeník2022generalizationmultiobjectivemachinelearning,Lin2019Pareto,Deist2021Multiobjectivelearningpredictpareto,Deist2023Multiobjective,Hu2023Revisiting}.
%This seems to recover a standard learning problem: For any trade-off point, we can scalarize the population and empirical objectives, and then apply the machinery from single objective learning. 

%\fy{do we need - what does this citation say?} (e.g., \cite{Duh2012Learning}).

However, in modern overparameterized regimes with relatively little labeled data compared to the model complexity,  %but lots of large amounts of unlabeled data available, 
it is crucial to leverage low-dimensional structure.
% and that the estimator exhibits a ``correct'' inductive bias. 
%such that it can be close to the estimated. 
In single-objective learning, efficient estimators have been proposed that successfully leverage the structural simplicity of the estimand to enjoy sample-efficiency,
%this has been established in a long line of work on estimators with the right inductive bias, 
e.g., for recovering sparse ground truths 
%of the training process is crucial \fy{impmortant: estimator has inductive bias, training has implicit bias}
\citep{Bühlmann2011statistics,Wainwright2019}. 
%Only few works have explored how to leverage structured estimands \fy{how much detail... but "this" is a bit } in the context of multi-objective learning.
%and it is not clear how this can be achieved (cf.\ \cref{subsec:related-work}). \tw{I think we have to be careful with the last sentence}
%\fy{sth like: In the context of multi-objective learning, there are so far no theoretical guarantees on regularization / leveraging structure...}
For multi-objective learning, however, no such results exist to date.
%there are so far no theoretical guarantees on how regularization can counteract the curse of dimensionality in the presence of sparsity or other low-dimensional structure.
Hence, in this paper, we take a step towards addressing the following question:
\begin{center}
    \emph{How can we leverage a low-dimensional structure like sparsity in the presence of multiple competing objectives?}
\end{center}
%\newpage
Our main contributions are outlined below. 
%\tw{this needs to be rewritten! I think mentioning this counterintuitive fact here would be good}
\begin{itemize}[leftmargin=*]
    %\item In \cref{sec:estimators}, we demonstrate that common regularization techniques to learning Pareto optimal points can be insufficient in the high-dimensional regime, even in the presence of low-dimensional structure like sparsity.
    \item We introduce a new two-stage MOL framework that can successfully take advantage of the low-dimensional structure of some distributional parameters %(that parameterize the loss)
    %the objective-specific minimizers 
    to estimate the entire Pareto set
    (\cref{sec:estimators}).
    %, while leveraging unlabeled data .
    \item  We prove upper bounds that illustrate how the estimation error of such distributional parameters propagates to the estimation error of all points in the Pareto set, and show minimax optimality of the procedure under mild conditions using lower bounds (\cref{sec:theoretical-guarantees}). In doing so, we discover that unlabeled data plays a crucial role in MOL.
    %\item We also prove a corresponding lower bound that depends on the minimax estimation errors of the parameters and implies tightness of our estimator for parts of the Pareto set. Although seemingly counterintuitive, it in fact reveals an interesting peculiarities inherent to a large class of MOL problems 
    %optimality of our estimator under certain conditions 
    %(\cref{sec:theoretical-guarantees}).
    \item We demonstrate the effectiveness of our two-stage estimation procedure in several applications (\cref{subsec:examples-theory}), validated by experiments (\cref{sec:experiments}).
\end{itemize}

\section{SETTING AND NOTATION}
\label{sec:setting}

Throughout, we denote $K$-tuples and sets that contain $K$-tuples in bold, where $K$ is the number of objectives.
Let $\cF$ be a hypothesis space of functions $\parameterizedfunction:\cX\to\cY$ parameterized by $\vartheta\in\RR^m$.
%, which we treat as our hypothesis space.
%We consider objectives that are defined via distributions: 
Let $\objectivevector:\RR^m\times \distributionSetTuple \to \RR^K$ be a non-negative vector-valued function that consists of $K$
objectives
%separate non-negative objectives, defined by
\begin{equation*}
    \objectivevector(\vartheta,\distributionXxYTuple)=\pr{\Loss_1(\vartheta,\distributionXxY^1),\dots,\Loss_K(\vartheta,\distributionXxY^K)},
\end{equation*}
where the $\objectiveindex$-th objective depends on a joint distribution $\distributionXxYindexed$
defined on $\cX\times \cY$. We denote by $\distributionXindexed$ the marginal of $X$, and $\distributionSetTuple$ denotes the set of all possible $K$-tuples $\distributionXxYTuple=(\distributionXxY^1,\dots,\distributionXxY^K)$ of joint distributions. %that form the statistical model.
%Whenever $\distributionXxYTuple$ is fixed, 
We sometimes use the short-hand $\objectivevector(\minimizer)=\objectivevector(\minimizer, \distributionXxYTuple)$.
For any $k$, we define the single-objective population minimizers as 
\begin{equation}
\label{eq:individual-minimizers}
    \minimizer_\objectiveindex \in \argmin_{\vartheta\in\RR^m}\objectiveindexed(\vartheta,\distributionXxYindexed) \subset \RR^m.
\end{equation} 
%each joint distribution $\distributionXxYindexed$,
%of the random vector $(X,Y)$,
%$(X,Y)\sim\distributionXxYindexed$, 
%we denote the marginal distribution of $X$ 
%on $\cX$  \fy{this read a bit weird of ... on ..}
%as $\distributionXindexed$ and  %we denote by $\minimizer_\objectiveindex$. 
Our setting includes many well-known problems involving several objectives. For example, 
%To exemplify our setting, one can consider 
choosing objectives 
\begin{equation}
\label{eq:specific}
    \Loss_k(\vartheta, \distributionXxYindexed) = \expect_{(X,Y)\sim \distributionXxYindexed} \br{\ell_k(f_\vartheta(X),Y)},
\end{equation}
for some point-wise losses $\ell_k:\cY\times \cY\to \RR$, 
gives rise to the multi-distribution learning setting \citep{Haghtalab2022Ondemand,Zhang2024Optimal,Larsen2024derandomizing}.
%hence including 
%is some loss.
%In particular, our formulation includes the
%This setting includes 
%previously studied settings where different losses $\ell_1,\dots,\ell_K$ are evaluated on the same distribution $\distributionXxY$ \citep{Duh2012Learning}, or the same loss $\ell$ is evaluated on different distributions $\distributionXxY^1,\dots,\distributionXxY^K$ \citep{Haghtalab2022Ondemand}. \fy{maybe add names for settings e.g. multi-distribution...?}
Concretely, in this paper, we often consider the example of  multi-distribution learning for sparse linear regression, described below.
%and the conditional distribution of $Y$ given $X$ on $\cY$ as $\distributionYgivenXindexed$.
% \begin{restatable}{ass}{Parameterization}
% \label{ass:Parameterization}
%     Objective $\objectiveindexed(\cdot,\distributionXxYindexed)$ depends on $\distributionXxYindexed$ only through $\minimizer_\objectiveindex\in \RR^m$ defined as in \cref{eq:individual-minimizers} and a parameter $\unlabeled_\objectiveindex\in\RR^M$ that only depends on the marginal $\distributionXindexed$. 
% \end{restatable}
%For concreteness, we now present two simple examples that reappear in later sections of the paper.
\begin{restatable}[Sparse linear regression]{ex}{MultipleLinearRegression}
    \label{ex:multiple-linear-regression}
        Consider $\cX=\RR^d$, $\cY=\RR$ and $\distributionSetTuple$ that consists of distributions $\distributionXxYindexed$ induced by
    \begin{equation}
    \label{eq:sparse-multi-distribution-model}
        Y=\inner{X}{\beta_\objectiveindex}+\xi \quad \text{with} \quad \xi\sim\cN(0,\sigma^2),
    \end{equation}
    where $\beta_\objectiveindex\in\RR^d$ are $s$-sparse ground truths $\norm{\beta_\objectiveindex}_0\leq s$ with bounded norm $\norm{\beta_k}_2\leq 1$ for all $k$,
    and $X$ are $B^2$-sub-Gaussian covariate vectors with covariance matrices $\covariance_\objectiveindex = \EE_{X\sim \distributionXindexed}\br{X X^\top} \succeq b^2\identity_d$. 
    As objectives, consider 
    %We are interested in 
    the prediction risk measured by the expected squared-loss on each distribution that is defined as
    \begin{align*}\objectiveindexed(\vartheta,\distributionXxYindexed)&=\EE_{(X,Y)\sim \distributionXxYindexed}\br{(\inner{X}{\vartheta}-Y)^2}\\
        &=\norm{\covariance_\objectiveindex^{1/2}(\vartheta-\beta_\objectiveindex)}_2^2+\sigma^2.
    \end{align*}
\end{restatable}
More generally, our formulation also captures objectives that cannot be written in the specific form \eqref{eq:specific}.
%are of a more general form.
\begin{restatable}[Fairness and risk]{ex}{FairnessRiskTradeoff}
    \label{ex:fairness-risk-trade-off}
    Consider a family of joint distributions $\distributionXxY$ of the random variables $Y,X,A$,
    where $A\in\mathset{\pm1}$ is a Rademacher variable and $Y=\inner{X}{\beta}+\xi$ with an $s$-sparse ground-truth $\norm{\beta}_2\leq 1$, noise $\xi\sim \cN(0,\sigma^2)$ and $X|A \sim \cN(A\mu,\identity_d)$.
    %the conditional distribution of $X$ given $A$ is $\cN(A\mu,\identity_d)$ for some fixed $\mu\in \RR^d$. 
    %The covariates $X$ are distributed as follows: For a Rademacher variable $A$ uniformly distributed on $\mathset{-1,1}$, the conditional distribution of $X$ is $\cN(A\mu,\identity_d)$ for some fixed $\mu\in \RR^d$. 
    The variable $A$ may represent an observed protected group attribute.
    %distributed according to  with noise variables $\xi\sim \cN(0,\sigma^2)$, and $A$ is a Rademacher variable uniformly distributed on $\mathset{-1,1}$. $A$ represents an observed protected group attribute (such as gender or ethnicity) of two groups with different covariate distributions $(X|A=\pm 1)\sim \cN(\pm\mu,\identity_d)$ for some fixed $\mu\in \RR^d$.
    %As one objective, we again consider 
    In this setting we consider two objectives: the expected squared loss
    \begin{equation*}
        \Lrisk(\vartheta,\distributionXxY) = \EE_{(X,Y)\sim \distributionXxY}\br{(\inner{X}{\vartheta}-Y)^2}
    \end{equation*}
    %and as a second objective, we introduce 
    and demographic parity via the $2$-Wasserstein distance between the group-wise distributions of $(\inner{X}{\vartheta}|A=a)$ and their barycenter \citep{Gouic2020Projectionfairnessstatisticallearning,Chzhen2022,Fukuchi2024Demographic}, defined as
    \begin{multline*}
        \Lfair(\vartheta,\distributionXxY) = \min_{\nu\in \cP_2(\RR)} \Big\{\frac{1}{2} W_2^2\pr{\law(\inner{X}{\vartheta}\setmid A=1), \nu}\\
         + \frac{1}{2}W_2^2\pr{\law(\inner{X}{\vartheta}\setmid A=-1), \nu}  \Big\}.
    \end{multline*}
    More details on this setting %, including definitions, 
    can be found in Appendix \ref{sec:fairness-regression-appendix}, where we also demonstrate that under our assumptions, $\Lfair(\vartheta,\distributionXxY) =\inner{\mu}{\vartheta}^2$. Hence, unless $\inner{\mu}{\beta}=0$, there is a trade-off between fairness and risk.
\end{restatable}
%The notion of demographic parity in \cref{ex:fairness-risk-trade-off} can be found in prior work on fair regression, such as ,}.

\subsection{Multi-objective optimization}

In view of multiple objectives, one goal could be to find a parameter that simultaneously minimizes all objectives at once. However, often this is not possible because the sets of minimizers of the objectives do not intersect. 
In that case, 
%We formalize multi-objective learning in terms of 
\emph{multi-objective optimization} (MOO) \citep{Ehrgott2005Multi} aims to find \emph{Pareto-optimal solutions} as defined below.
%either find the set of all \emph{Pareto-optimal} solutions, or a set of points that have a small \emph{scalarized} loss, both defined below.
\begin{definition}[Pareto-optimality]
A parameter $\vartheta\in\RR^m$ is \emph{Pareto-optimal}, if for all $\vartheta'\in \RR^m$ and $\objectiveindex\in[K]$
\begin{equation*}
    \objectiveindexed(\vartheta',\distributionXxYindexed)<\objectiveindexed(\vartheta,\distributionXxYindexed) \Rightarrow \exists j: \Loss_j(\vartheta',\distributionXxY^j)>\Loss_j(\vartheta,\distributionXxY^j).
\end{equation*}
    The set $\PF=\mathset{\objectivevector(\vartheta,\distributionXxYTuple) \setmid \vartheta \text{ is Pareto-optimal}}$ is called the \emph{Pareto front} and the set $\mathset{\vartheta\in\RR^m\setmid \vartheta\text{ is Pareto-optimal}}$ is called \emph{Pareto set} of $\objectivevector$ and $\distributionXxYTuple$.
\end{definition}
%then $\minimizer_\objectiveindex \neq \minimizer_j$ for $k\neq j$ \fy{strictly speaking this doesn't say they don't intersect, but your choices of minimizers are not equal}. 
For a toy two-objective problem, \cref{fig:toy-pareto-front} 
depicts the Pareto set in the parameter space $\RR^m$ with ``end-points'' $\minimizer_\objectiveindex$ as the solid gray line on the left, and the Pareto front in the space of objective values on the right. 
%On the left, the solid gray line depicts the Pareto set in the parameter space $\RR^m$ with ``end-points'' $\minimizer_\objectiveindex$, while on the right, it traces the Pareto front in the two-dimensional space of objective function values. 
The gray-shaded region on the right corresponds to the set of all achievable value pairs of the two objective functions.
%Optimal trade-off points between the objectives 
The Pareto set can often be recovered using what is known as \emph{scalarization}.
%can be \fy{captured/enforced} by scalarizing the multiple objectives. %to a scalar quantity.
%, respectively. More generally, 
\begin{definition}
    A \emph{scalarization} of $\objectivevector$ is the composition $\scalarizationObjectiveComposition$ with a function $\scalarization:\RR^K\to\RR$, parameterized by $\weightTuple$ in the simplex $\Delta^{K}= \{\weightTuple \in \RR^K : \weight_\objectiveindex\geq 0 \text{ and } \sum_{\objectiveindex=1}^K\weight_\objectiveindex=1 \}$. 
\end{definition}
\cref{eq:linear-and-Chebyshev-scalarization} already introduced two important scalarizations, known as linear and Chebyshev scalarization.
We denote by $\minimizerWeighted$ the minimizer of a scalarization, i.e.,
\begin{equation}
    \label{eq:scalarization}
    \minimizerWeighted\in\argmin_{\vartheta\in\RR^m}\scalarizationObjectiveComposition(\vartheta,\distributionXxYTuple).
\end{equation}
For rest of the paper, when we use $\minimizerWeighted$ to denote a minimizer of an unspecified scalarization, we implicitly assume that the scalarization parameterizes the Pareto set, that is, $\weightTuple\mapsto \minimizerWeighted$ is a surjection from the simplex to the Pareto set\textemdash see \citet{Hillmeier2001Generalized,Roy2023optimizationparetosetstheory} for details on when this is true, and the manifold structure this map can induce.
%We explicitly write the 
%such a parameterization applies for the appropriately chosen scalarization.
%\tw{Find a better way to formulate this} \fy{i tried}

%In \cref{fig:toy-pareto-front}, we depict a toy example of the set of Pareto-optimal points (red line on the left) and the corresponding Pareto front (red line on the right).

\begin{figure}
    \centering
    \begin{tikzpicture}[scale=0.6]

% Pic 1 -----------------------------------------

\draw[-] (-6.5, 1) -- (-1.5, 1) node[below] {};
\draw[-] (-5, -0.2) -- (-5, 5.5) node[left] {};
\node[] (Rm) at (-2,5){$\RR^m$};

\draw[thick, population_color] plot[smooth, tension=0.6] coordinates {(-5,2) (-4.5,3) (-3.5,3) (-2.5,3.2) (-3,1)};
\fill[population_color] (-5,2) circle (2pt) node[left] {$\minimizer_2$};
\fill[population_color] (-3,1) circle (2pt) node[below] {$\minimizer_1$};
\fill[population_color] (-2.5,3.2) circle (2pt) node[above] {$\minimizerWeighted$};

\draw[thick, ts_color, dashed] plot[smooth, tension=0.6] coordinates {(-5,3) (-4.5,3.5) (-3,3.9) (-2,3.5) (-2,2.5) (-2.5,1)};
\fill[ts_color] (-5,3) circle (2pt) node[left] {$\estimatedMinimizer_2$};
\fill[ts_color] (-2.5,1) circle (2pt) node[below] {$\estimatedMinimizer_1$};
\fill[ts_color] (-2,3.5) circle (2pt) node[right] {$\estimatedMinimizerWeighted$};

% Pic 2 -----------------------------------------
\draw[-] (-0.2, 0) -- (5.5, 0);
\node[] (ax1) at (-0.3,5.5){\tiny $\Loss_1$};
\draw[-] (0, -0.2) -- (0, 5.5);
\node[] (ax1) at (5.5,-0.3){\tiny $\Loss_2$};

\fill[population_color!20, opacity=0.5] (5.5, 5.5) -- (1.1, 5.5) -- (1,4) -- (1.5,3.15) -- (2,2.5) -- (2.4, 2.2) -- (3,2) -- (4,1.8) -- (5,1) -- (5.5, 1.1) -- cycle;
\node[] () at (4.3,5){$\objectivevector(\RR^m)$};
\node[population_color] () at (1.5,2.5){$\PF$};

\draw[thick, population_color] plot[smooth, tension=0.6] coordinates {(1,4) (2,2.5) (3,2) (4,1.8) (5,1)};
\fill[population_color] (1,4) circle (2pt) node[left] {$\objectivevector(\minimizer_2)$};
\fill[population_color] (5,1) circle (2pt) node[below] {$\objectivevector(\minimizer_1)$};
\fill[population_color] (3,2) circle (2pt) node[below] {$\objectivevector(\minimizerWeighted)$};

\draw[thick, ts_color, dashed] plot[smooth, tension=0.6] coordinates {(1.2,4.3) (2.5,2.5) (3.1,2.3) (4.4,2.2) (5.1,1.5)};
\fill[ts_color] (1.2,4.3) circle (2pt) node[above] {$\objectivevector(\estimatedMinimizer_2)$};
\fill[ts_color] (5.1,1.5) circle (2pt) node[right] {$\objectivevector(\estimatedMinimizer_1)$};
\fill[ts_color] (3,2.3) circle (2pt) node[above right] {$\objectivevector(\estimatedMinimizerWeighted)$};

\draw[gray, ->,very thick, bend right=10] (-1.5,2) to (1,2.3);

\end{tikzpicture}
    \caption{\small{The parameter space $\RR^m$ (left) parameterizes the hypothesis set $\cF$ and contains the population Pareto set $\{\minimizerWeighted |  \weightTuple\in\Delta^{K}\}$ (gray line), and the set of the empirical estimators $\{\estimatedMinimizerWeighted | \weightTuple\in \Delta^K\}$ (dashed blue line). 
    %$\minimizerWeighted$ can be non-sparse even when $\minimizer_1,\minimizer_2$ are sparse. 
    The right figure depicts the region of all values that can be obtained by $\objectivevector(\vartheta)$ for some $\vartheta$ (gray shaded area),  the population Pareto front $\PF$ (gray line) and estimated Pareto front (dashed blue line).}}
    \label{fig:toy-pareto-front}
    \vspace{-0.1in}
\end{figure}
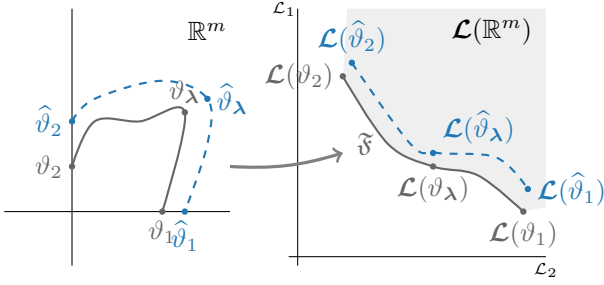

\subsection{Multi-objective learning}
In practice, we may only observe finite samples from the distributions $\distributionXxYindexed$.
%but finite samples from it. 
We consider a semi-supervised setting, where we observe 
a set $\cD$ consisting of $n_\objectiveindex$ 
i.i.d.\ labeled samples from $\distributionXxYindexed$, denoted
$\mathset{(X^{\objectiveindex}_{i},Y^{\objectiveindex}_{i})}_{i=1}^{n_\objectiveindex}$, as well as $N_\objectiveindex\in\NN$ unlabeled i.i.d.\ samples %contained in the dataset 
$\mathset{X^{\objectiveindex}_{i}}_{i=n_\objectiveindex+1}^{n_\objectiveindex+N_\objectiveindex}$
from each marginal distribution $\distributionXindexed$. 
%We let $\cD$ denote the entire dataset of labeled and unlabeled datapoints.
Then, the empirical measure $\empdistributionXxYindexed=n_\objectiveindex^{-1}\sum_{i=1}^{n_\objectiveindex}\ind_{(X^{\objectiveindex}_{i},Y^{\objectiveindex}_{i})}$ is an estimate of the distribution $\distributionXxYindexed$ and we denote $\empdistributionXxYTuple=(\empdistributionXxY^1,\dots,\empdistributionXxY^K)$.
%Based on the labeled samples, we can estimate $\distributionXxYindexed$ using the empirical measure $\empdistributionXxYindexed=n_\objectiveindex^{-1}\sum_{i=1}^{n_\objectiveindex}\ind_{(X^{\objectiveindex}_{i},Y^{\objectiveindex}_{i})}$ and define $\empdistributionXxYTuple=(\empdistributionXxY^1,\dots,\empdistributionXxY^K)$ accordingly. 
%We again denote $\empdistributionXxY=(\empdistributionXxYindexed)_{\objectiveindex=1}^K$. 
%Moreover, we can estimate the marginal $\distributionXindexed$ via the empirical marginal using  the entire dataset $\empdistributionXindexed=(n_\objectiveindex+N_\objectiveindex)^{-1}\sum_{i=1}^{n_\objectiveindex+N_\objectiveindex}\ind_{X^{\objectiveindex}_{i}}$.
%The aim in this setting is to 
The aim of \emph{multi-objective learning} (MOL) is to use the data $\cD$ to recover the Pareto set $\{\minimizerWeighted | \weightTuple\in\Delta^{K}\}\subset\RR^m$, that is, to find a set of estimators $\{\estimatedMinimizerWeighted | \weightTuple\in\Delta^{K}\}$ with small estimation errors %$\ell_2$-norm, 
$\|\estimatedMinimizerWeighted-\minimizerWeighted\|_2$. 
In our paper, we treat the estimation of each point $\minimizerWeighted$ as a separate problem, although our statements hold for all $\weightTuple \in\Delta^{K}$.  
%\fy{i always wondered here, whether this has to be known in advance, this lambda - or we just defined it as the one in the set with smallest distance of $\minimizerWeighted$}
%Note that here, we denote by $\estimatedMinimizerWeighted$ the estimator for the Pareto-optimal point $\minimizerWeighted$ parameterized by $\weightTuple$.
%\fy{now here i think you might immediately translate into the empirical version and think is it clear that you'd use the same lambda etc.} 
For each $\weightTuple$, we then compare our estimation procedures with the information-theoretically optimal
%expected squared $\ell_2$-norm.We denote the minimax 
estimation error (i.e., minimax error)
\begin{equation}
\label{eq:minimax-estimation}
    \minimax(\distributionSetTuple)=\inf_{\estimatedMinimizerWeighted}\sup_{\distributionXxYTuple\in\distributionSetTuple} \EE_{\cD}\br{\norm{\estimatedMinimizerWeighted(\cD)-\minimizerWeighted}_2},
\end{equation}
where the infimum is taken over all estimators $\estimatedMinimizerWeighted$ that map $\cD$ to $\RR^m$, and the expectation is taken over draws of $\cD$. Note that later we adopt the common abuse of notation and drop the dependence on $\cD$, that is, we write $\estimatedMinimizerWeighted=\estimatedMinimizerWeighted(\cD)$.

%\fy{if we never refer to the first displayequation below, then we don't need to display it here ;)}
This paper focuses on bounding the estimation error, since\textemdash under mild regularity assumptions\textemdash it bounds other metrics of interest, such as the  
%We consider estimation error, because any bounds on $\|\estimatedMinimizerWeighted-\minimizerWeighted\|_2$ are strong in the sense that under weak assumptions such as smoothness, they also imply bounds on the excess in each individual objective, $\objectiveindexed(\estimatedMinimizerWeighted,\distributionXxYindexed)-\objectiveindexed(\minimizerWeighted,\distributionXxYindexed)$, as well as the 
excess scalarized objective 
\begin{equation}
\label{eq:excess-scalarized-loss}
    \excessscalarization(\estimatedMinimizerWeighted):=\scalarizationObjectiveComposition(\estimatedMinimizerWeighted,\distributionXxYTuple)-\min_{\vartheta\in\RR^m}\scalarizationObjectiveComposition(\vartheta,\distributionXxYTuple),
\end{equation}
and the \emph{hypervolume} \citep{Zitzler1999multiobjective} of the estimated Pareto front, defined for $\cS \subset [0,r]^K$, $r\geq 0$, as\footnote{We write $s\preceq x$ if $s_i\leq x_i$ for all $i$.}
\begin{equation}
\label{eq:hypervolume-definition}
    \hypervolume_r(\cS) := \vol(\mathset{x\in [0,r]^K \setmid\exists s\in \cS: s \preceq x}).
\end{equation}
Here $\vol(\cdot)$ denotes the Lebesgue measure on $\RR^K$. We now formalize this statement using the function
\begin{equation*}
    \eps(G,\smoothnessParam,\weightTuple) := G\norm{\estimatedMinimizerWeighted-\minimizerWeighted}_2 + \frac{\smoothnessParam}{2}\norm{\estimatedMinimizerWeighted-\minimizerWeighted}_2^2.
\end{equation*}
\vspace{-0.1in}
\begin{restatable}{prop}{HypervolumeBound}
\label{prop:HypervolumeBound}
Let $G_\objectiveindex:=\sup_{\weightTuple\in\Delta^K}\norm{\nabla_\vartheta\objectiveindexed(\minimizerWeighted)}_2$, assume $\vartheta\mapsto \objectiveindexed(\vartheta)$ is $\smoothnessParam_k$-smooth,
%Define the function $\eps$ as
and define $\eps_{\max} := \max_{k\in[K],\weightTuple\in\Delta^K}\eps(G_k,\smoothnessParam_k,\weightTuple)$. It then holds that
%Then the following inequalities hold:
\begin{enumerate}
\vspace{-0.2cm}
\setlength\itemsep{0em}
    \item for linear scalarization, $\excessscalarization(\estimatedMinimizerWeighted) \leq \eps(0,\scalarization(\boldsymbol{\smoothnessParam}),\weightTuple)$, 
    \item $\objectiveindexed(\estimatedMinimizerWeighted)-\objectiveindexed(\minimizerWeighted)\leq \eps(G_k,\smoothnessParam_k,\weightTuple)$,
    \item $\hypervolume_r(\PFhat) \geq \pr{1-2\eps_{\max}/r}^K\hypervolume_r(\PF)$, 
\vspace{-0.2cm}
\end{enumerate}
for $\PFhat = \{\objectivevector(\estimatedMinimizerWeighted)| \weightTuple\in\Delta^K\}$, $\PF = \mathset{\objectivevector(\minimizerWeighted)\setmid \weightTuple\in\Delta^K}$ and any constant $r\geq 2\sup_{\weightTuple \in \Delta^K}\norm{\objectivevector(\minimizerWeighted)}_\infty$. 
\end{restatable}
The proof of \cref{prop:HypervolumeBound} can be found in \cref{subsec:proof-HypervolumeBound}.

\section{TWO ESTIMATORS}
\label{sec:estimators}
%A FAILURE MODE OF PLUG-IN ESTIMATION}

A commonly used approach to recover $\minimizerWeighted$
%in the multi-objective learning literature 
is to consider a \emph{plug-in} estimator that is
the minimizer of the objective \eqref{eq:scalarization} where $\distributionXxYTuple$ is replaced by $\empdistributionXxYTuple$ (see, e.g., \citep{Jin2007multi}).
Albeit simple, this approach suffers from the curse of dimensionality when $m$ is large relative to the sample size (i.e., in the high-dimensional regime). We now discuss how low-dimensional structure can be leveraged for MOL.
%the high-dimensional regime requires introducing regularization to this approach. But no work so far effectively deals with structure in MOL.

\subsection{Naive approach: direct regularization}
For a single objective, a common practice in high-dimensional statistics is to use a regularizer that reflects the structural simplicity (such as sparsity) of the single objective minimizer. 
Since the scalarized objective in \eqref{eq:scalarization} can be viewed as a generic scalar loss, one may be tempted
to analogously add 
%to follow common practice for high-dimensional (single-objective) learning by adding 
a penalty term $\rho:\RR^m\to \RR$ to the empirical objective and find a minimizer of a directly regularized ($\regularizedpluginOp$) scalarization, that it, to use\footnote{Alternatively, one may also add a penalty in each objective separately or view the penalty as an additional objective. For linear scalarization, any of these alternative strategies would result in a final estimator that is equivalent to \eqref{eq:penalized-scalarization}.}
%solve a directly regularized ($\regularizedpluginOp$) scalarization
\begin{equation}
\label{eq:penalized-scalarization}
    \drEstimatedMinimizerWeighted \in \argmin_{\vartheta\in\RR^m} \scalarizationObjectiveComposition(\vartheta,\empdistributionXxYTuple)+\rho(\vartheta).
\end{equation}
%or constrain the parameter space to a subset of $\RR^m$. \fy{not necessary}
This approach can be effective 
%is often used 
if the inductive bias on the multi-objective solution \eqref{eq:scalarization} is the same across the Pareto set \citep{Jin2008Pareto,Cortes2020Agnostic,Mierswa2007Controlling,Bieker2022Treatment,Hotegni2024multiobjectiveoptimizationsparsedeep}.
%The rationale behind regularization in high-dimensional learning is the assumption that the minimizer exhibits a certain structure that is captured by the penalty $\rho$ (e.g., sparsity by the $\ell_1$-norm). Hence, estimators of the form \eqref{eq:penalized-scalarization} could work well if the corresponding Pareto-optimal points adhere to this structure. %\fy{abstractly speaking, we try to learn whole pareto front from a few easy points,i.e. sth like (rephrase!!)} 
However, assuming that all points $\minimizerWeighted$ in the Pareto set have the same simple structure is less justified for a generic problem.
%unrealistic depending on the problem at hand.  %(e.g., the entire Pareto-optimal set is sparse).
In particular, we typically assume structural simplicity of distributional parameters, such as sparsity of $\beta_k$ in \cref{ex:multiple-linear-regression}.
Direct regularization penalties such as \eqref{eq:penalized-scalarization} then help in single-objective learning 
because the \emph{minimizers $\minimizer_k$ happen to coincide with distributional parameters}.
%only because the \emph{minimizers $\minimizer_k$ happen to coincide with distributional parameters}, like $\beta_k$ in \cref{ex:multiple-linear-regression}, that exhibit sparsity.
But even if some distributional parameters are sparse, large parts of the Pareto set will \emph{not} coincide with them, and hence can still be non-sparse (cf.\ \cref{fig:toy-pareto-front}).
Therefore, applying a regularization penalty that works for the individual objectives (such as an $\ell_1$-norm penalty) does not generally improve the estimate from \eqref{eq:penalized-scalarization} for all $\weightTuple$, except when $\weight_\objectiveindex=1$ and we are estimating $\minimizer_\objectiveindex$. 
%Let us exemplify this effect in a simple example.
\begin{restatable}[Sparse fixed-design linear regression]{ex}{FixedDesignLinearRegression}
\label{ex:fixed-design-linear-regression}
    %Consider the setting where 
    Let $n\geq d$. For each $k\in\mathset{1,2}$, we observe a fixed design matrix $ \X_k \in \mathfrak{X}(\gamma)$, where $\mathfrak{X}(\gamma)$ is defined as 
    \begin{equation*}
    \mathfrak{X}(\gamma) = \mathset{\X\in \RR^{n\times d}\setmid \gamma^{-1}\identity_d \preceq n^{-1}\X^\top \X \preceq \gamma \identity_d}
\end{equation*}
for some $\gamma >1$.
Further, we observe noisy responses
\begin{equation*}y^\objectiveindex=\X_\objectiveindex\beta_\objectiveindex+\xi^{\objectiveindex} \quad\text{with}\quad \xi^\objectiveindex\overset{\text{i.i.d.}}{\sim} \cN(0,\sigma^2\identity_n),
\end{equation*}
where $\beta_k$ is from the set $\Gamma\subset \RR^d$ of $1$-sparse vectors.
%with additive i.i.d.\ Gaussian noise with variance $\sigma^2$.
With slight abuse of notation, we define the population and empirical objectives for $\objectiveindex\in\mathset{1,2}$ as
\begin{align*}
    \objectiveindexed (\vartheta,\distributionXxYindexed) &= n^{-1}\norm{\X_\objectiveindex(\vartheta-\beta_\objectiveindex)}_2^2, \\ 
    \objectiveindexed (\vartheta,\empdistributionXxYindexed) &= n^{-1}\norm{\X_\objectiveindex\vartheta-y^k}_2^2.
\end{align*}
%Let $n\geq d$, 
%and define for any $\gamma>1$
%\begin{equation*}
%    \mathfrak{X}(\gamma) = \mathset{\X\in \RR^{n\times d}\setmid \frac{1}{\gamma}\identity_d \preceq \frac{1}{n}\X^\top \X \preceq \gamma \identity_d}.
%\end{equation*}
\end{restatable}
\vspace{-0.2cm}
Note that here $\minimizer_k=\beta_k$ are sparse 
and direct regularization with $\ell_1$-norm (i.e., the LASSO \citep{Tibshirani1996lasso}) would mitigate the curse of dimensionality for $\weightTuple$ with $\weight_k=1$, leading to an estimation error of order $\sigma^2\log (d) / n$ \citep{Bickel2009simultaneous}. 
The following proposition shows, however, that for general $\weightTuple$, any estimator with direct regularization \eqref{eq:penalized-scalarization} \emph{cannot} leverage the sparsity of $\beta_\objectiveindex$ and incurs a sample complexity that is linear in the dimension $d$. 
%\emph{cannot} leverage the sparsity of $\beta_\objectiveindex$ to mitigate the curse of dimensionality when estimating $\minimizerWeighted$.
\begin{restatable}[Insufficiency of direct regularization]{prop}{InsufficiencyPluginRegularization}
\label{prop:InsufficiencyPluginRegularization}
Consider \cref{ex:fixed-design-linear-regression} and linear scalarization. For any $\weightTuple$ with $\weight_1,\weight_2>0$, $\sigma^2\leq 2n\gamma^2/(d+1)$ %and if we use linear scalarization, then, for 
    and any regularizer $\rho$, an estimator $\drEstimatedMinimizerWeighted$ from \eqref{eq:penalized-scalarization} satisfies
    \begin{equation*}
        \sup_{\substack{\beta_1,\beta_2\in \Gamma \\\X_1,\X_2\in \mathfrak{X}(\gamma)}} \EE\br{\norm{\drEstimatedMinimizerWeighted-\minimizerWeighted}_2^2}\gtrsim \frac{\sigma^2 d}{n \gamma}. 
    \end{equation*}
\end{restatable}
See \cref{subsec:proof-InsufficiencyPluginRegularization} for the proof. As we can see, 
even though the individual minimizers $\minimizer_k = \beta_k$ are $1$-sparse, in a worst-case sense, direct regularization does not mitigate the curse of dimensionality for all other points in the Pareto set.

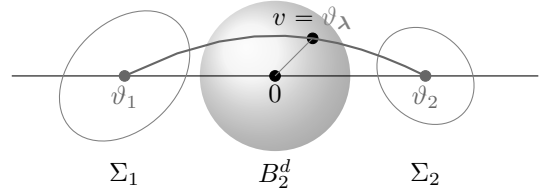
\begin{figure}
    \centering
    \newcommand{\drawellipsenew}[5]{
    \pgfmathsetmacro{\a}{#1}  % Covariance matrix element a
    \pgfmathsetmacro{\b}{#2}  % Covariance matrix element b
    \pgfmathsetmacro{\c}{#3}  % Covariance matrix element c
    
    % Compute the eigenvalues (semi-axes lengths)
    \pgfmathsetmacro{\discriminant}{sqrt((\a-\c)^2 + 4*(\b)^2)}
    \pgfmathsetmacro{\lambdaone}{(\a + \c + \discriminant)/2}
    \pgfmathsetmacro{\lambdatwo}{(\a + \c - \discriminant)/2}
    
    % Calculate the angle of rotation (eigenvector direction)
    \pgfmathsetmacro{\angle}{atan2(2*\b, \a-\c)/2}  % atan2 ensures correct quadrant
    
    % Draw the ellipse: center #4, semi-axes lengths sqrt(lambdaone) and sqrt(lambdatwo)
    \draw[rotate around={\angle:#4}, thin,#5] #4 ellipse ({sqrt(\lambdaone)} and {sqrt(\lambdatwo)});
}

\begin{tikzpicture}

    % Draw the sphere in the middle
    \shade[ball color=black!10!white, opacity=0.5] (0,0) circle (1);

    % Draw the points on the x-axis
    \draw[-] (-3.5,0) -- (3.5,0); 
    \filldraw[black] (0, 0) circle (2pt) node[below] {$0$};

    % Labels
    \node at (0, -1.3) {$B_2^d$}; % Sphere label
    \node at (-2, -1.3) {$\covariance_1$}; % Covariance matrix label P1
    \node at (2, -1.3) {$\covariance_2$};
    
    %line 1
    \filldraw[population_color] (-2,0) circle (2pt) node[below] {$\minimizer_1$}; 
    \filldraw[population_color] (2,0) circle (2pt) node[below] {$\minimizer_2$};
    %\filldraw[population_color] (0.5, 0.5) circle (2pt) node[below] {$\ParameterWeighted$};
    \filldraw[black] (0.5,0.5) circle (2pt) node[above] {$v={\color{gray}\minimizerWeighted}$};
    \draw[thin,gray] (0,0) -- (0.5,0.5);
    %[[ 0.41666667 -0.08333333][-0.08333333  0.41666667]] 
    \drawellipsenew{0.4166}{-0.0833}{0.4166}{(2, 0)}{gray};
    %[[0.75 0.25][0.25 0.75]]
    \drawellipsenew{0.75}{0.25}{0.75}{(-2, 0)}{gray};

    \draw[-,population_color,thick] {
    (-2.0, -0.0)--(-1.3, 0.3)--(-0.743, 0.457)--(-0.275, 0.525)--(0.133, 0.533)--(0.5, 0.5)--(0.836, 0.436)--(1.15, 0.35)--(1.446, 0.246)--(1.729, 0.129)--(2.0, 0.0)
    };

\end{tikzpicture}
    \caption{\small Illustration of the intuition for \cref{prop:InsufficiencyPluginRegularization,prop:NecessityUnlabeledData} in linear regression with squared loss and linear scalarization: For any $v\in B_2^d$, we can find covariance matrices $\covariance_1,\covariance_2$ with constrained condition number, and $1$-sparse $\beta_1,\beta_2$, so that the minimizer $\minimizerWeighted$ of \eqref{eq:scalarization} satisfies $v =  \minimizerWeighted$. This makes learning with direct regularization and without enough unlabeled data infeasible. %That is where the hardness of learning $\minimizerWeighted$ comes from in \cref{prop:InsufficiencyPluginRegularization,prop:NecessityUnlabeledData}.
    %From the perspective of a directly regularized estimator, this turns the problem into a Gaussian sequence model on $B_2^d$. The gray line represents the Pareto set. \vspace{-0.1in}
    }
    \label{fig:proof_illustration}
\end{figure}

The proof of \cref{prop:InsufficiencyPluginRegularization} uses the fact that we can choose the covariance matrices $\frac{1}{n}\X_k^\top \X_k$ adversarially within the eigenvalue constraints, so that the Pareto-optimal $\minimizerWeighted$ lies anywhere in an $\ell_2$-ball of fixed radius, see \cref{fig:proof_illustration}. If we did not allow for this (e.g., if the covariance matrices are scaled identities), $\minimizerWeighted$ would also be sparse, and the directly regularized estimator could achieve a fast rate. 
We will revisit this point later in \cref{subsec:necessity-unlabeled-data,prop:NecessityUnlabeledData}.

%any directly penalized estimator may not be able to overcome the curse of being in high dimensions, e.g., when $d\gtrsim n$, even though the individual minimizers are $1$-sparse.

\subsection{A new two-stage estimator}
The previous example suggests that in contrast to the single-objective case, learning points in the Pareto set sample-efficiently requires explicitly estimating sparse distributional parameters separately. Building on this intuition, we indeed propose such a two-stage estimator in this section. 
%that are reasonable to be assumed to be sparse.
%This motivates us to exploit distributional parameters explicitly. 
First, to formalize this, we assume that each objective depends on the distributions $\distributionXxYindexed$ via some parameter $\parameter_k \in\RR^p$.
\begin{restatable}{ass}{Parameterization}
\label{ass:Parameterization}
    For each $k\in[K]$, the objective $\objectiveindexed(\cdot,\distributionXxYindexed)$ depends on $\distributionXxYindexed$ only through $\parameter_k \equiv \parameter_k(\distributionXxYindexed) \in \RR^p$, so that we can abuse notation and write $\objectiveindexed(\vartheta,\parameter_k)=\objectiveindexed(\vartheta,\distributionXxYindexed)$.
\end{restatable}
Denoting $\parameterTuple= (\parameter_1,\dots,\parameter_K)$, 
%and $\unlabeledTuple = (\unlabeled_1,\dots,\unlabeled_K)$
we can write 
\begin{equation*}
    \objectivevector(\vartheta,\parameterTuple) := (\Loss_1(\vartheta,\parameter_1),\dots,\Loss_K(\vartheta,\parameter_K)),
\end{equation*}
and define $\parameterSpaceTuple \subset \RR^{K\cdot p}$ to be the set of all possible $\parameterTuple$ for $K$-tuples of distributions in $\distributionSetTuple$. Throughout, we do not distinguish between matrices and their vectorizations, unless necessary.
We argue that the re-parameterization from \cref{ass:Parameterization} can be found in many cases.
For instance, in \cref{ex:multiple-linear-regression}, the parameter $\theta_\objectiveindex$ corresponds to the tuple $(\beta_k,\covariance_k)$. In \cref{ex:fairness-risk-trade-off}, we have $\theta_{\fairOp}=\mu$ and $\parameter_{\riskOp}=(\beta,\mu)$. And finally, in \cref{ex:fixed-design-linear-regression}, we have $\parameter_k= \beta_k$. 
Other objectives that fit this framework include the robust risk \citep{Yin2019Rademacher} and a variety of fairness losses \citep{Berk2017Convexframeworkfairregression}.
Note that an important case of \cref{ass:Parameterization} is when the individual minimizers $\minimizer_k$ from \eqref{eq:individual-minimizers} and part of the parameters $\parameter_k$ coincide. For instance, in \cref{ex:multiple-linear-regression}, $\minimizer_k$ is a component of $\parameter_k=(\beta_k,\covariance_k)$, and in \cref{ex:fixed-design-linear-regression}, $ \minimizer_k = \beta_k=\parameter_k$.

%\tw{We have to rewrite this definition, seems to be a major point of confusion}
Under \cref{ass:Parameterization}, we can now
introduce the two-stage estimation framework for learning Pareto-optimal solutions for any $\weightTuple\in\Delta^K$ in high dimensions. 
%define our \emph{two-stage} ($\twostageOp$) regularized multi-objective estimator $\tsEstimatedMinimizerWeighted$ as the outcome of a two-stage procedure. 
\begin{definition}
\label{def:two-stage-estimator}
%The \emph{bi-level (bl)} estimator is defined as
We define $\tsEstimatedMinimizerWeighted$ as the final solution of the following two-stage optimization procedure. 

\textbf{Stage 1: Estimation.} Use the data $\cD$ to estimate the parameters $\estimatedParameterTuple=(\estimatedParameter_1,\dots,\estimatedParameter_K)^\top$ in any way.\footnote{Note that the estimator $\estimatedParameter_k$ of $\parameter_k$ does \emph{not} need to be the plug-in estimator of the empirical distribution $\empdistributionXxYindexed$.}
%, possibly but not necessarily using \eqref{eq:two-stage-penalty}.

\textbf{Stage 2: Optimization.} Minimize the scalarized objective
\begin{equation}
    \tsEstimatedMinimizerWeighted \in \argmin_{\vartheta\in\RR^m} \scalarizationObjectiveComposition(\vartheta,\estimatedParameterTuple). \label{eq:stage-2}
\end{equation}
\end{definition}
%Note that when $\weight_\objectiveindex=1$, we recover the directly regularized estimator from \eqref{eq:penalized-scalarization}.
%The results in \cref{sec:theoretical-guarantees} shed light on the relevance and unconventional benefits of using unlabeled data to estimate $\unlabeledTuple$ for MOL.
%Its most general formulation is as a wrapper: In the first stage, it takes as input \emph{any} estimator $\estimatedParameterTuple$ for $\parameterTuple$. In the second stage, it computes the multi-objective estimator based on these estimates. 
In its general form, 
this estimator first learns a \emph{probabilistic model} of the distributions \citep{Ng2001Discriminative} which it then plugs into the scalarized objective to estimate the Pareto set. 
%is similar to \emph{probabilistic modeling} pipelines \citep{Ng2001Discriminative} in that it learns all (necessary) parameters first, and then the estimation of any Pareto-optimal point $\minimizerWeighted$ can reap the benefits of the efficiency of the parameter estimation of $\parameterTuple$. 
Further, if the parameters $\parameter_k$ coincide with the individual minimizers $\minimizer_k$, the estimator resembles a form of \emph{Mixture of Experts} \citep{Dimitriadis2023pareto,Chen2024efficient,Tang2024towards} with expert models $\estimatedMinimizer_k = \estimatedParameter_k$.
%: It first learns one expert model per objective, namely $\estimatedMinimizer_k = \estimatedParameter_k$, and then uses these experts to estimate the rest of the Pareto set.

Naturally, for $\tsEstimatedMinimizerWeighted$ to be a 
sample-efficient estimator, the estimators $\estimatedParameterTuple$ for $\parameterTuple$ themselves
need to be efficient.
%have to be chosen well and have the correct inductive bias towards the parameters. 
For instance, if $\parameter_\objectiveindex$ is sparse, one could choose an $\ell_1$-norm penalty \citep{Tibshirani1996lasso} with appropriate regularization strength \citep{Bickel2009simultaneous,Chatterjee2011bootstrapping}. 
We now show how for \cref{ex:fixed-design-linear-regression}, the two-stage estimator with the $\ell_1$-norm penalty performs much better than any directly regularized estimator (\cref{prop:InsufficiencyPluginRegularization}).

\begin{restatable}{prop}{TwostageRateFixedDesign}
\label{prop:TwostageRateFixedDesign}
    In the setting of \cref{ex:fixed-design-linear-regression,prop:InsufficiencyPluginRegularization}, consider
    $\estimatedParameterTuple$ with $\estimatedParameter_k= \betahat_k$, and the corresponding two-stage estimator $\tsEstimatedMinimizerWeighted$ for linear scalarization, with
    \begin{align*}
        \betahat_k &\in \argmin_{\beta\in \RR^d} \frac{1}{n}\norm{\X_k\beta-y^k}_2^2 + 6\gamma \sigma \sqrt{\frac{2\log d}{n}}\norm{\beta}_1, \\
        \tsEstimatedMinimizerWeighted &\in \argmin_{\minimizer\in \RR^d} \frac{\weight_1}{n}\norm{\X_1(\minimizer-\betahat_1)}_2^2+ \frac{\weight_2}{n}\norm{\X_2(\minimizer-\betahat_2)}_2^2.
    \end{align*}
    This two-stage estimator achieves estimation error
    \vspace{-0.1cm}
    \begin{equation*}
        \sup_{\substack{\beta_1,\beta_2\in \Gamma \\ \X_1,\X_2\in \mathfrak{X}(\gamma)}}\norm{\tsEstimatedMinimizerWeighted-\minimizerWeighted}_2^2 \lesssim \frac{\gamma^7 \sigma^2\log d}{n}
    \end{equation*}
    \vspace{-0.15cm}
    with probability at least $1-4d^{-4}$.
\end{restatable}
See \cref{subsec:proof-TwostageRateFixedDesign} for the proof. We can see that our estimator recovers the well-known rates of the LASSO \citep{Tibshirani1996lasso,Bickel2009simultaneous} along the entire Pareto set.

%\section{SEMI-SUPERVISED BI-LEVEL ESTIMATION}
%\label{sec:our-estimator}
%\input{AISTATS/sections/ourestimator}

\section{THEORETICAL GUARANTEES}
\label{sec:theoretical-guarantees}

%So far, our considerations have been based on examples.
We now prove upper and lower bounds for more general problems. 
First, \cref{thm:MinimaxRateStrongConvexity,prop:LipschitzBound} show how the error in estimating the parameters $\parameterTuple$ propagates to the estimation error of the two-stage estimator. We then establish a minimax lower bound in \cref{thm:GeneralLowerBound} that is tight in many cases. We instantiate these bounds in concrete examples to obtain explicit statistical errors for learning Pareto sets in high dimensions. Finally, we discuss the important role of unlabeled data for multi-objective learning.

\subsection{Main results}
\label{subsec:main-results}
Relating the error of estimating $\parameterTuple$ with the error of estimating $\minimizerWeighted$ 
relies on some assumptions about the objectives. To formulate them, we introduce a new set $\parameterSpaceTupleLarge\supset \parameterSpaceTuple$ on which the objectives should satisfy these assumptions. The set $\parameterSpaceTupleLarge$ should be large enough so that the estimators satisfy $\estimatedParameterTuple \in \parameterSpaceTupleLarge$ with high probability. We rely on two different sets of regularity assumptions: strongly convex objectives and objectives with Lipschitz parameterization.
%To formulate our main results, we first state sufficient assumptions.
%While for distributions in the statistical model $\distributionSetTuple$ it holds by definition that $\parameterTuple\in\parameterSpaceTuple$, the corresponding estimators from \cref{def:two-stage-estimator} may not necessarily fall into the same set $\parameterSpaceTuple$. Hence, we
%formulate our assumptions on an enlarged set $\parameterSpaceTupleLarge\supset \parameterSpaceTuple$ (which can be chosen depending on the context and estimator $\estimatedParameterTuple$),
%and then, when applying our results, we  need to show that the estimators lie within $\parameterSpaceTupleLarge$ with high probability.
%In this section we state theoretical guarantees for objectives that satisfy the following regularity assumptions.
%We recall basic definitions from the convex optimization literature in \cref{sec:basics-convex-optimization} for completeness.
\paragraph{Strongly convex objectives.} We start by first stating a bound under the following two assumptions.
\begin{restatable}[Strongly convex objectives]{ass}{StrongConvexitySmoothness}
\label{ass:StrongConvexitySmoothness}
%\leavevmode
For all $\parameterTuple\in \parameterSpaceTupleLarge$ and $\objectiveindex\in[K]$, the map $\vartheta\mapsto \objectiveindexed(\vartheta,\parameter_\objectiveindex)$ is differentiable and 
$\strongConvexityParam_\objectiveindex$-strongly convex with $\strongConvexityParam_\objectiveindex\geq 0$
%\vspace{-0.2cm}
%\begin{itemize}
%\setlength\itemsep{-0.2em}
%    \item $\nu_\objectiveindex$-smooth w.r.t.\ the $\ell_2$-norm with $\nu_k < \infty$,
%    \item 
%\end{itemize}
%\vspace{-0.2cm}
and $\strongConvexityParam_j>0$ for at least one objective $j\in[K]$. We denote $\strongConvexityParamTuple=(\strongConvexityParam_1,\dots,\strongConvexityParam_K)$.
\end{restatable}
Note that only one of the objectives is assumed to be \emph{strongly} convex, and $\strongConvexityParam_\objectiveindex=0$ corresponds to regular convexity, see \cref{sec:basics-convex-optimization} for a reminder of definitions. 
\begin{restatable}[Locally Lipschitz-continuous gradients]{ass}{LocallyLipschitz}
\label{ass:LocallyLipschitz}
For all $\objectiveindex\in[K]$, $\vartheta\in\RR^m$ and all $\parameterTuple,\parameterTuple'\in \parameterSpaceTupleLarge$ it holds that
\begin{equation*}
    \norm{\nabla_\vartheta \objectiveindexed(\vartheta,\parameter_\objectiveindex) - \nabla_\vartheta \objectiveindexed(\vartheta,\parameter_\objectiveindex')}_2 \leq \zeta_k(\vartheta)\norm{\parameter_k-\parameter_k'}
\end{equation*}
with $\zeta_k(\vartheta)\geq 0$ depending on $\parameterSpaceTupleLarge$ and $\vartheta\in\RR^m$, and where $\norm{\cdot}$ is some norm. We denote $\zeta(\minimizer) = \max_{k\in[K]}\zeta_k(\minimizer)$.
\end{restatable}
%\cref{ass:StrongConvexitySmoothness,ass:LocallyLipschitz} are
Note that both \cref{ass:StrongConvexitySmoothness,ass:LocallyLipschitz} are common in the (multi-objective) optimization literature \citep{Hillmeier2001Generalized,Roy2023optimizationparetosetstheory,Ehrgott2005Multi,Bubeck2015Convex,Boyd2004Convex}, and both hold for several standard settings in statistics and machine learning\textemdash including \cref{ex:multiple-linear-regression,ex:fairness-risk-trade-off,ex:fixed-design-linear-regression}.
%naturally satisfied for many standard losses in machine learning \fy{would be careful to say that, especially strong convexity is rare - see e.g. crossentropy which is also maybe the standard loss}. For example, it is easily verified that the objectives from \cref{ex:multiple-linear-regression,ex:fairness-risk-trade-off} satisfy \cref{ass:StrongConvexitySmoothness}.
The next theorem %which is our main result, 
provides upper bounds on the estimation error $\|\tsEstimatedMinimizerWeighted-\minimizerWeighted\|_2$ 
in terms of the estimation error of the parameters $\parameterTuple$.
\begin{restatable}{thm}{MinimaxRateStrongConvexity}
\label{thm:MinimaxRateStrongConvexity}
    Let \cref{ass:StrongConvexitySmoothness,ass:LocallyLipschitz} hold with $\strongConvexityParamTuple$ and $\zeta$, respectively.  Let $j$ be the index of the strongly convex objective ($\mu_j>0$), and $\tsEstimatedMinimizerWeighted$ be the minimizer of \eqref{eq:stage-2} with linear scalarization, i.e., $\scalarizationObjectiveComposition=\sum_{\objectiveindex=1}^K\weight_\objectiveindex\objectiveindexed$. Then, for all $\weightTuple\in\Delta^K$ with $\weight_j>0$ and $\estimatedParameterTuple\in\parameterSpaceTupleLarge$
    %of the strongly convex function be positive,
    it holds that
    \begin{equation*}
        \norm{\tsEstimatedMinimizerWeighted-\minimizerWeighted}_2\leq \frac{\zeta(\minimizerWeighted)}{\scalarization(\strongConvexityParamTuple)} \sum_{\objectiveindex=1}^K \weight_\objectiveindex\|\estimatedParameter_\objectiveindex-\parameter_\objectiveindex\|.
    \end{equation*}
\end{restatable}
We prove \cref{thm:MinimaxRateStrongConvexity} in \cref{subsec:proof-MinimaxRateStrongConvexity}. Notice that if $\distributionSetTuple$ and $\zeta(\minimizerWeighted)$ are such that $\zeta:=\sup_{\parameterTuple\in\parameterSpaceTuple, \weightTuple\in \Delta^K} \zeta(\minimizerWeighted)<\infty$ is a constant independent of all parameters, %we can plug this into the bound an 
we obtain a bound uniformly over $\distributionSetTuple$\textemdash assumed to be  true in the subsequent discussion. 
%We assume that this is the case for the rest of the discussion.
Note that under \cref{ass:StrongConvexitySmoothness}, linear scalarization is sufficient to reach the entire Pareto front \citep[Theorem 4.1]{Ehrgott2005Multi}, but the proof of
\cref{thm:MinimaxRateStrongConvexity} may also be extended to other scalarizations, such as (smoothed) Chebyshev scalarization \citep{Lin2024smooth} under other assumptions. %conditions.

The upper bound in \cref{thm:MinimaxRateStrongConvexity} has a very straight-forward interpretation: For fixed choices of $\weightTuple$, the statistical rate of $\| \tsEstimatedMinimizerWeighted-\minimizerWeighted \|_2$ 
inherits the rates of estimating $\parameterTuple$. 
%improve whenever the $\parameter_\objectiveindex$ has a low-dimensional structure like sparsity, or can be estimated using unlabeled data. 
Further, using \cref{prop:HypervolumeBound}, we note that \cref{thm:MinimaxRateStrongConvexity} also implies bounds on the excess scalarized objective and the hypervolume.

%\paragraph{Discussion of the upper bound.}% and Assumptions \ref{ass:StrongConvexitySmoothness} and \ref{ass:LocallyLipschitz}}
%\label{subsec:discussion-assumptions}

%We now discuss the Assumptions \ref{ass:StrongConvexitySmoothness} and \ref{ass:Injectivity}, beginning with \cref{ass:Injectivity}, as its interpretation was already discussed in \cref{subsec:discussion-lower-bound}. It includes many interesting cases, and is similar to many conditions found in inverse optimization \citep{Aswani2018Inverse,Gebken2021Inverse}.
%Relating the error of estimating $\parameterTuple$
%with the error of estimating the Pareto-optimal $\minimizerWeighted$ 
The proof of \cref{thm:MinimaxRateStrongConvexity}
relies on studying how the minimizer of an optimization problem changes with respect to the parameters of that problem\textemdash a problem extensively discussed in the optimization stability literature
\citep{Ito1992sensitivity,Gfrerer2016Lipschitz,Dontchev1995characterizations,Bonnans2000perturbation,Shvartsman2012stability}. 
Results in that literature often rely on properties akin to \cref{ass:StrongConvexitySmoothness,ass:LocallyLipschitz} for the Implicit Function Theorem to apply (see \citet[\S 1]{Bonnans2000perturbation} and \citet[\S I.3.4]{Miettinen1999nonlinear}). Our proof of   \cref{thm:MinimaxRateStrongConvexity} similarly 
relies on \cref{ass:StrongConvexitySmoothness,ass:LocallyLipschitz} 
to show that 
the implicitly defined function 
\begin{equation*}
    \parameterTuple \mapsto \minimizerWeighted(\parameterTuple)=\argmin_{\vartheta\in\RR^m}\scalarizationObjectiveComposition(\vartheta,\parameterTuple)
\end{equation*}
is Lipschitz continuous.

%The key ingredient for deriving guarantees for the two-stage estimator is to understand how improvements in estimating the parameters $\parameterTuple$ translate into improvements in estimating the Pareto-optimal $\minimizerWeighted$.
% \cref{ass:StrongConvexitySmoothness,ass:LocallyLipschitz} is used in \cref{thm:MinimaxRateStrongConvexity} to prove that the implicitly defined function 
% \begin{equation*}
%     \parameterTuple \mapsto \minimizerWeighted(\parameterTuple)=\argmin_{\vartheta\in\RR^m}\scalarizationObjectiveComposition(\vartheta,\parameterTuple)
% \end{equation*}
% is Lipschitz continuous. 
%and the proof follows ideas from the optimization stability literature
%\citep{Ito1992sensitivity,Gfrerer2016Lipschitz,Dontchev1995characterizations,Bonnans2000perturbation,Shvartsman2012stability}, which is the study of how the minimizer of an optimization problem changes with respect to the parameters of that problem.

\paragraph{Lipschitz parameterization.} \cref{ass:StrongConvexitySmoothness} excludes examples where the objectives are (globally) Lipschitz continuous, such as the Huber loss \citep{Huber1964Robust}, only convex, such as logistic loss on separable data \citep{Ji2019implicit}, or even non-convex.
Fortunately, these cases can be addresses if the objectives are Lipschitz in their parameters using standard arguments on the excess scalarized objective. This offers an alternative to \cref{thm:MinimaxRateStrongConvexity}. 

\begin{restatable}[Lipschitz parameterization]{prop}{LipschitzBound}
\label{prop:LipschitzBound}
    Assume that the parameterization $\parameter_\objectiveindex\mapsto \objectiveindexed(\cdot, \parameter_\objectiveindex)$ is $1$-Lipschitz continuous with respect to the function 
    $\Phi:\RR^p\times \RR^p\to\RR$, in the sense that for all $\parameterTuple,\parameterTuple' \in \parameterSpaceTupleLarge$ it holds
    \begin{equation*}
    %\label{eq:Lipschitz-parameterization}
        \sup_{\vartheta\in\RR^m} \abs{\objectiveindexed(\vartheta,\parameter_\objectiveindex)-\objectiveindexed(\vartheta,\parameter_\objectiveindex')}\leq \Phi(\parameter_\objectiveindex,\parameter_\objectiveindex').
    \end{equation*}
    Then, for any scalarization of the form $\scalarization(x)=\norm{\weightTuple\odot x}$ with some norm $\norm{\cdot}$, the excess scalarized loss of $\tsEstimatedMinimizerWeighted$---as defined in \cref{eq:excess-scalarized-loss}---is bounded by
    \begin{equation*}
        \excessscalarization(\tsEstimatedMinimizerWeighted) \leq 2 \scalarization\pr{(\Phi(\estimatedParameter_\objectiveindex,\parameter_\objectiveindex))_{\objectiveindex=1}^K}.
    \end{equation*}
\end{restatable}
Notably, \cref{prop:LipschitzBound} can apply to non-convex objectives, and allows the use of Chebyshev scalarization (\cref{eq:linear-and-Chebyshev-scalarization}). 
The proof, found in \cref{subsec:proof-LipschitzBound}, follows the standard uniform learning decomposition, similar to the bounds found in \cite{Sukenik2024generalization}. The difference here is that we may still observe the benefits from regularization and unlabeled data for the two-stage estimator.

\subsection{Tightness and a minimax lower bound}

We now provide a lower bound on the minimax multi-objective estimation errors from \cref{eq:minimax-estimation}
%characterize the multi-objective minimax optimality in terms of
in terms of the minimax parameter estimation error
\begin{equation*}
\delta_\objectiveindex :=\inf_{\estimatedParameter_k}\sup_{\distributionXxYTuple\in\distributionSetTuple} \EE_{\cD}\br{\norm{\estimatedParameter_k(\cD)-\parameter_\objectiveindex}},
\end{equation*}
where the infimum is taken over all estimators 
%$\estimatedParameter(\cS_\objectiveindex,\cU_\objectiveindex)$ \fy{for i 1 to K} 
that have access to the unlabeled and labeled datasets. 
Our arguments rely on the 
%We show a general lower bound for objectives that satisfy the 
following \emph{identifiability} assumption. 
\begin{restatable}[Lipschitz identifiability]{ass}{Injectivity}
    \label{ass:Injectivity}
For all $\vartheta\in\RR^m$, the mapping $g_\objectiveindex(\cdot;\vartheta):\parameter_\objectiveindex\mapsto \nabla_\vartheta \objectiveindexed(\vartheta,\parameter_\objectiveindex)$
is injective on $\parameterSpaceTupleLarge$, and for any $\parameterTuple,\parameterTuple'\in \parameterSpaceTupleLarge$, $\minimizer,\minimizer'\in \RR^m$ and $u\in \Ima{g_k(\cdot;\minimizer)}, u'\in \Ima{g_k(\cdot;\minimizer')}$, we have that
\begin{align*}
\norm{g_\objectiveindex(\parameter_k; \vartheta)-g_\objectiveindex(\parameter_k'; \vartheta')}_2 &\leq \eta_k\pr{\norm{\parameter_k-\parameter_k'}+\norm{\vartheta-\vartheta'}_2},\\
    \norm{g_\objectiveindex^{-1}(u; \vartheta)-g_\objectiveindex^{-1}(u'; \vartheta')}_2 &\leq \eta_k'\pr{\norm{u-u'}_2+\norm{\vartheta-\vartheta'}_2}.
\end{align*}
\end{restatable}
%Notice how \cref{ass:Injectivity} implies $\eta_k$-smoothness of $\vartheta\mapsto\objectiveindexed(\vartheta,\parameter_\objectiveindex)$. \fy{why is that relevant? other way around would be more useful?}
Assumptions of this type are common in the inverse optimization literature, which studies the identification of optimization parameters from a minimizer; see \cite{Aswani2018Inverse,Gebken2021Inverse} and references therein. In particular, \cref{ass:Injectivity} is, e.g., satisfied by \cref{ex:fixed-design-linear-regression}, and
\cref{ex:multiple-linear-regression,ex:fairness-risk-trade-off} under some conditions (see \cref{subsec:examples-theory}). The intuition behind \cref{ass:Injectivity} is that it ensures the identifiability of optimization parameters\textemdash in our case $\parameter_k$\textemdash from the minimizer of an optimization problem\textemdash in our case $\minimizerWeighted$\textemdash in some Lipschitz manner. This allows us to lower bound the minimax estimation error.

\begin{restatable}{thm}{GeneralLowerBound}
\label{thm:GeneralLowerBound}
    If \cref{ass:Injectivity} holds and we use linear scalarization, %$N_\objectiveindex=\infty$ for all $k$, 
    the minimax rate is lower bounded as
    \begin{equation*}
        \minimax(\distributionSetTuple)\geq \max_{k\in [K]} \pr{1+\scalarization(\boldsymbol{\eta})}^{-1}\pr{\frac{\weight_k}{\eta_{k}'}\delta_k-\sum_{i\neq k}\eta_{i}\weight_i \delta_i}_+
    \end{equation*}
    where $\pr{\cdot}_+:=\max\mathset{0,\cdot}$.
\end{restatable}
\cref{thm:GeneralLowerBound} is proved in \cref{subsec:proof-GeneralLowerBound}.

%On the other hand, the lower bound characterizes the ``limits" of Pareto-optimal set estimation in the best case when we have a lot of unlabeled data (see also discussion in \cref{subsec:examples-theory}), in which case we can essentially think of $\estimatedUnlabeled_\objectiveindex=\unlabeled_\objectiveindex$.
%We discuss some intuition on why unlabeled data helps in \cref{subsec:examples-theory}.
%In the lower bound we assume that $\estimatedUnlabeledTuple=\unlabeledTuple$, corresponding to infinitely many unlabeled samples. 
%This is for simplicity of exposition only, and could be readily adapted to capture the estimation error of $\estimatedUnlabeledTuple$. 

\paragraph{Minimax optimality of the two-stage estimator.} We can now use the  lower bound from \cref{thm:GeneralLowerBound} to discuss the minimax optimality of our two-stage procedure. 
%compare the upper bound from \cref{thm:MinimaxRateStrongConvexity} and lower bound from \cref{thm:GeneralLowerBound}. 
In particular, we first note that the lower bound is tight for all $\weightTuple \in \Delta^K$ that satisfy for some large enough constant $C=C(\strongConvexityParamTuple,\zeta,\boldsymbol{\eta},\boldsymbol{\eta}')>0$ and some $k\in [K]$ that 
\begin{equation}
\label{eq:sufficient-condition-minimax}
    \weight_\objectiveindex\delta_\objectiveindex\geq C\sum_{i\neq k}\weight_i\delta_i.
\end{equation}
To see this, first observe that when inequality \eqref{eq:sufficient-condition-minimax} holds,
the lower bound (neglecting dependence on $\strongConvexityParamTuple,\boldsymbol{\eta},\boldsymbol{\eta}'$) reduces to 
$\minimax(\distributionSetTuple)\gtrsim \max_{k\in[K]}\weight_k\delta_k$.
Supposing $\estimatedParameter_\objectiveindex$ estimates $\parameter_\objectiveindex$ in a minimax optimal manner so that $\sup_{\distributionXxYTuple\in \distributionSetTuple}\EE\|\estimatedParameter_\objectiveindex-\parameter_\objectiveindex\|\asymp \delta_\objectiveindex$, under \eqref{eq:sufficient-condition-minimax}, the upper bound in \cref{thm:MinimaxRateStrongConvexity} also reduces to $\sup_{\distributionXxYTuple\in \distributionSetTuple}\EE\|\tsEstimatedMinimizerWeighted-\minimizerWeighted\|_2 \lesssim \max_{k\in [K]} \weight_k \delta_k$. Hence, up to constants depending on $\strongConvexityParamTuple,\zeta,\boldsymbol{\eta},\boldsymbol{\eta}'$, we obtain that
\begin{equation*}
     \sup_{\distributionXxYTuple\in \distributionSetTuple}\EE\br{\norm{\tsEstimatedMinimizerWeighted-\minimizerWeighted}_2}\asymp \minimax(\distributionSetTuple)\asymp \max_{k\in[K]} \weight_k\delta_k,
\end{equation*}
and the two-stage estimator is minimax optimal. %under the aforementioned assumptions. 

We now demonstrate that \eqref{eq:sufficient-condition-minimax} holds for essentially all $\weightTuple\in \Delta^K$, that is, across the majority of the Pareto front, for large enough sample size and fixed $K$. 
%In the sequel, ``$\lesssim$'' suppresses constants that may depend on the parameters $\etabold,\etabold'$ and $\zeta$.
To that end, consider the simplified problem for $K=2$ objectives with minimax rates $\delta_1,\delta_2$, so that \eqref{eq:sufficient-condition-minimax} reduces to $\weight_1\delta_1 \geq C \weight_2\delta_2$ (or vice versa). We can make the following case distinction. 

\begin{itemize}[leftmargin=*]
\setlength\itemsep{0em}
\vspace{-0.1cm}
    \item Case 1: $\delta_2 =o(\delta_1)$ (or vice versa) 
    as $n,d\to\infty$, that is, estimating the two parameters is unequally hard. In that case, for every fixed $\weightTuple \in \Delta^2$, \eqref{eq:sufficient-condition-minimax} holds when $n,d$ are large enough. For example, in linear regression with the parameters $\parameter_k$ being one sparse and one dense ground truth, the minimax rates are $\delta_1 \asymp \sqrt{d / n}$ and $\delta_2 \asymp \sqrt{\log (d)/n}$. Then \eqref{eq:sufficient-condition-minimax} holds if $\weight_1/\weight_2 \geq C \sqrt{\log (d) / d}$, which constitutes a large subset of $\Delta^2$ that expands to the entire simplex as $d\to \infty$. Therefore, for most $\weightTuple\in\Delta^2$, the estimation error eventually scales with the ``slower'' minimax rate, and the upper bound is tight. 
    \item Case 2: $\delta_1\approx \delta_2$, that is, estimating the distributional parameters in both problems is approximately equally hard. In that case, \eqref{eq:sufficient-condition-minimax} holds for all $\weightTuple\in\Delta^2$ where either $\weight_1\geq C\weight_2$ or $\weight_2\geq C\weight_1$, which constitutes a large part of the simplex, unless $C$ is very large.
    For example, in linear regression when both ground-truths are equally dense so that $\delta_1=\delta_2 \asymp \sqrt{d/n}$, then for almost all $\weightTuple$, we have that $\minimax(\distributionSetTuple)\asymp \max\{\weight_1,\weight_2\} \sqrt{d/n}$. 
\vspace{-0.1cm}
\end{itemize}
A similar argument works for any fixed $K$\textemdash but \eqref{eq:sufficient-condition-minimax} has to be evaluated more carefully if $K$ grows with $n,d$.
We conclude that, under \cref{ass:Parameterization,ass:StrongConvexitySmoothness,ass:LocallyLipschitz,ass:Injectivity}, the difficulty of MOL is dominated by the hardest individual learning problem, if \emph{all other} individual learning tasks are easier in the sense of \eqref{eq:sufficient-condition-minimax}. Then the two-stage estimator is optimal.

\subsection{Application to Examples \ref{ex:multiple-linear-regression} and \ref{ex:fairness-risk-trade-off}}
\label{subsec:examples-theory}

We now apply \cref{thm:MinimaxRateStrongConvexity,thm:GeneralLowerBound} to derive explicit bounds for the two problems from \cref{ex:multiple-linear-regression,ex:fairness-risk-trade-off}.

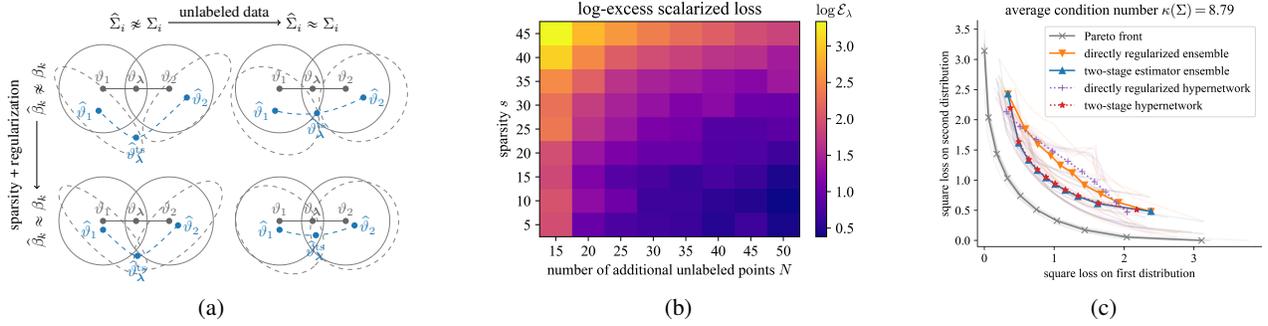
\begin{figure*}
    \centering
    \subfloat[]{
    \resizebox{0.33\linewidth}{!}{%
    %\input{macros}

% To draw ellipses using covariance matrices
\newcommand{\drawellipse}[5]{
    \pgfmathsetmacro{\a}{#1}
    \pgfmathsetmacro{\b}{#2}
    \pgfmathsetmacro{\c}{#3}
    
    % Corrected eigenvalue calculation with positive b^2
    \pgfmathsetmacro{\discriminant}{sqrt((\a-\c)^2 + 4*(\b)^2)}
    \pgfmathsetmacro{\lambdaone}{(\a + \c + \discriminant)/2}
    \pgfmathsetmacro{\lambdatwo}{(\a + \c - \discriminant)/2}
    
    % Special case: if a == c, we manually handle the rotation
    \ifdim \a pt=\c pt
        \pgfmathsetmacro{\angle}{45 * (\b > 0 ? 1 : -1)} % Set rotation to 45 degrees if b != 0
    \else
        \pgfmathsetmacro{\angle}{atan(2*\b/(\a-\c))/2} % Normal case
    \fi
    
    % Draw ellipse
    \draw[rotate around={\angle:#4}, thin, #5] #4 ellipse ({sqrt(\lambdaone)} and {sqrt(\lambdatwo)});
}

\begin{tikzpicture}

\node[rotate=90] (t1) at (-1.5,0){$\betahat_k\not\approx\beta_k$};
\node[rotate=90] (t2) at (-1.5,-3){$\betahat_k\approx\beta_k$};
\node (S1) at (0.75,1.5){$\estimatedCovariance_i\not\approx\covariance_i$};
\node (S2) at (4.75,1.5){$\estimatedCovariance_i\approx\covariance_i$};

\draw[->] (S1) -- (S2) node[midway, above]{unlabeled data};
\draw[->] (t1) -- (t2);

\node[rotate=90] (labeleddata) at (-2,-1.5){sparsity + regularization};

% top left figure ------------------------------------------

\coordinate (centertopleft1) at (0,0);
\coordinate (centertopleft2) at (1.5,0);
\coordinate (optimalparamtopleft) at (0.75,0);

\coordinate (esttopleft1) at (-0.1,-0.5);
\coordinate (esttopleft2) at (1.9,-0.2);
\coordinate (estparamtopleft) at (0.75, -1.108);

\def\corrone{-0.8}
\def\corrtwo{0.7}

% Ellipses
\drawellipse{1}{0}{1}{(centertopleft1)}{gray}; % for S1
\drawellipse{1}{\corrone}{1}{(esttopleft1)}{dashed,gray}; % for S2

\drawellipse{1}{0}{1}{(centertopleft2)}{gray}; % for S3
\drawellipse{1}{\corrtwo}{1}{(esttopleft2)}{dashed,gray};  % for S4

\fill[population_color] (centertopleft1) circle (2pt) node[above] {$\minimizer_1$};
\fill[population_color] (centertopleft2) circle (2pt) node[above] {$\minimizer_2$};

\fill[ts_color] (esttopleft1) circle (2pt) node[left] {$\estimatedMinimizer_1$};
\fill[ts_color] (esttopleft2) circle (2pt) node[right] {$\estimatedMinimizer_2$};

\fill[population_color] (optimalparamtopleft) circle (2pt) node[above] {$\minimizerWeighted$};
\fill[ts_color] (estparamtopleft) circle (2pt) node[below] {$\tsEstimatedMinimizerWeighted$};

\draw[dashed,ts_color]  {
    (1.9, -0.2)--(1.326, -0.743)--(1.084, -0.948)--(0.941, -1.046)--(0.838, -1.092)--(0.75, -1.108)--(0.664, -1.096)--(0.567, -1.057)--(0.443, -0.977)--(0.255, -0.825)--(-0.1, -0.5)
};

\draw[thin,population_color]  {
    (centertopleft1)--(centertopleft2)
};

% top right figure ------------------------------------------

\coordinate (centertopleft1) at (4,0);
\coordinate (centertopleft2) at (5.5,0);
\coordinate (optimalparamtopleft) at (4.75,0);

\coordinate (esttopleft1) at (3.9,-0.5);
\coordinate (esttopleft2) at (5.9,-0.2);
\coordinate (estparamtopleft) at (4.849, -0.555);

\def\corrone{-0.3}
\def\corrtwo{0.1}

% Ellipses
\drawellipse{1}{0}{1}{(centertopleft1)}{gray}; % for S1
\drawellipse{1}{\corrone}{1}{(esttopleft1)}{dashed,gray}; % for S2

\drawellipse{1}{0}{1}{(centertopleft2)}{gray}; % for S3
\drawellipse{1}{\corrtwo}{1}{(esttopleft2)}{dashed,gray};  % for S4

\fill[population_color] (centertopleft1) circle (2pt) node[above] {$\minimizer_1$};
\fill[population_color] (centertopleft2) circle (2pt) node[above] {$\minimizer_2$};

\fill[ts_color] (esttopleft1) circle (2pt) node[left] {$\estimatedMinimizer_1$};
\fill[ts_color] (esttopleft2) circle (2pt) node[right] {$\estimatedMinimizer_2$};

\fill[population_color] (optimalparamtopleft) circle (2pt) node[above] {$\minimizerWeighted$};
\fill[ts_color] (estparamtopleft) circle (2pt) node[below] {$\tsEstimatedMinimizerWeighted$};

\draw[dashed,ts_color]  {
    (5.9, -0.2)--(5.668, -0.31)--(5.45, -0.399)--(5.243, -0.468)--(5.043, -0.52)--(4.849, -0.555)--(4.66, -0.574)--(4.471, -0.579)--(4.283, -0.568)--(4.093, -0.542)--(3.9, -0.5)
};

\draw[thin,population_color]  {
    (centertopleft1)--(centertopleft2)
};

% bottom left figure ---------------------------------------------

\coordinate (centertopleft1) at (0,-3);
\coordinate (centertopleft2) at (1.5,-3);
\coordinate (optimalparamtopleft) at (0.75,-3);

\coordinate (esttopleft1) at (0,-3.2);
\coordinate (esttopleft2) at (1.7,-3.1);
\coordinate (estparamtopleft) at (0.78, -3.791);

\def\corrone{-0.8}
\def\corrtwo{0.7}

% Ellipses
\drawellipse{1}{0}{1}{(centertopleft1)}{gray}; % for S1
\drawellipse{1}{\corrone}{1}{(esttopleft1)}{dashed,gray}; % for S2

\drawellipse{1}{0}{1}{(centertopleft2)}{gray}; % for S3
\drawellipse{1}{\corrtwo}{1}{(esttopleft2)}{dashed,gray};  % for S4

\fill[population_color] (centertopleft1) circle (2pt) node[above] {$\minimizer_1$};
\fill[population_color] (centertopleft2) circle (2pt) node[above] {$\minimizer_2$};

\fill[ts_color] (esttopleft1) circle (2pt) node[left] {$\estimatedMinimizer_1$};
\fill[ts_color] (esttopleft2) circle (2pt) node[right] {$\estimatedMinimizer_2$};

\fill[population_color] (optimalparamtopleft) circle (2pt) node[above] {$\minimizerWeighted$};
\fill[ts_color] (estparamtopleft) circle (2pt) node[below] {$\tsEstimatedMinimizerWeighted$};

\draw[dashed,ts_color]  {
    (1.7, -3.1)--(1.248, -3.523)--(1.056, -3.68)--(0.941, -3.753)--(0.855, -3.785)--(0.78, -3.791)--(0.705, -3.775)--(0.619, -3.733)--(0.506, -3.654)--(0.332, -3.508)--(0.0, -3.2)
};

\draw[thin,population_color]  {
    (centertopleft1)--(centertopleft2)
};

% bottom right figure ---------------------------------------------

\coordinate (centertopleft1) at (4,-3);
\coordinate (centertopleft2) at (5.5,-3);
\coordinate (optimalparamtopleft) at (4.75,-3);

\coordinate (esttopleft1) at (4,-3.2);
\coordinate (esttopleft2) at (5.7,-3.1);
\coordinate (estparamtopleft) at (4.823, -3.323);

\def\corrone{-0.3}
\def\corrtwo{0.1}

% Ellipses
\drawellipse{1}{0}{1}{(centertopleft1)}{gray}; % for S1
\drawellipse{1}{\corrone}{1}{(esttopleft1)}{dashed,gray}; % for S2

\drawellipse{1}{0}{1}{(centertopleft2)}{gray}; % for S3
\drawellipse{1}{\corrtwo}{1}{(esttopleft2)}{dashed,gray};  % for S4

\fill[population_color] (centertopleft1) circle (2pt) node[above] {$\minimizer_1$};
\fill[population_color] (centertopleft2) circle (2pt) node[above] {$\minimizer_2$};

\fill[ts_color] (esttopleft1) circle (2pt) node[left] {$\estimatedMinimizer_1$};
\fill[ts_color] (esttopleft2) circle (2pt) node[right] {$\estimatedMinimizer_2$};

\fill[population_color] (optimalparamtopleft) circle (2pt) node[above] {$\minimizerWeighted$};
\fill[ts_color] (estparamtopleft) circle (2pt) node[below] {$\tsEstimatedMinimizerWeighted$};

\draw[dashed,ts_color]  {
    (5.509, -3.177)--(5.328, -3.236)--(5.155, -3.279)--(4.987, -3.308)--(4.823, -3.323)--(4.661, -3.324)--(4.499, -3.313)--(4.336, -3.289)--(4.17, -3.251)--(4.0, -3.2)
};

\draw[thin,population_color]  {
    (centertopleft1)--(centertopleft2)
};

\end{tikzpicture}
    \label{fig:covariance-intuition}
    }
    }\hfill
    \subfloat[ ]{
    \includegraphics[width=0.29\linewidth]{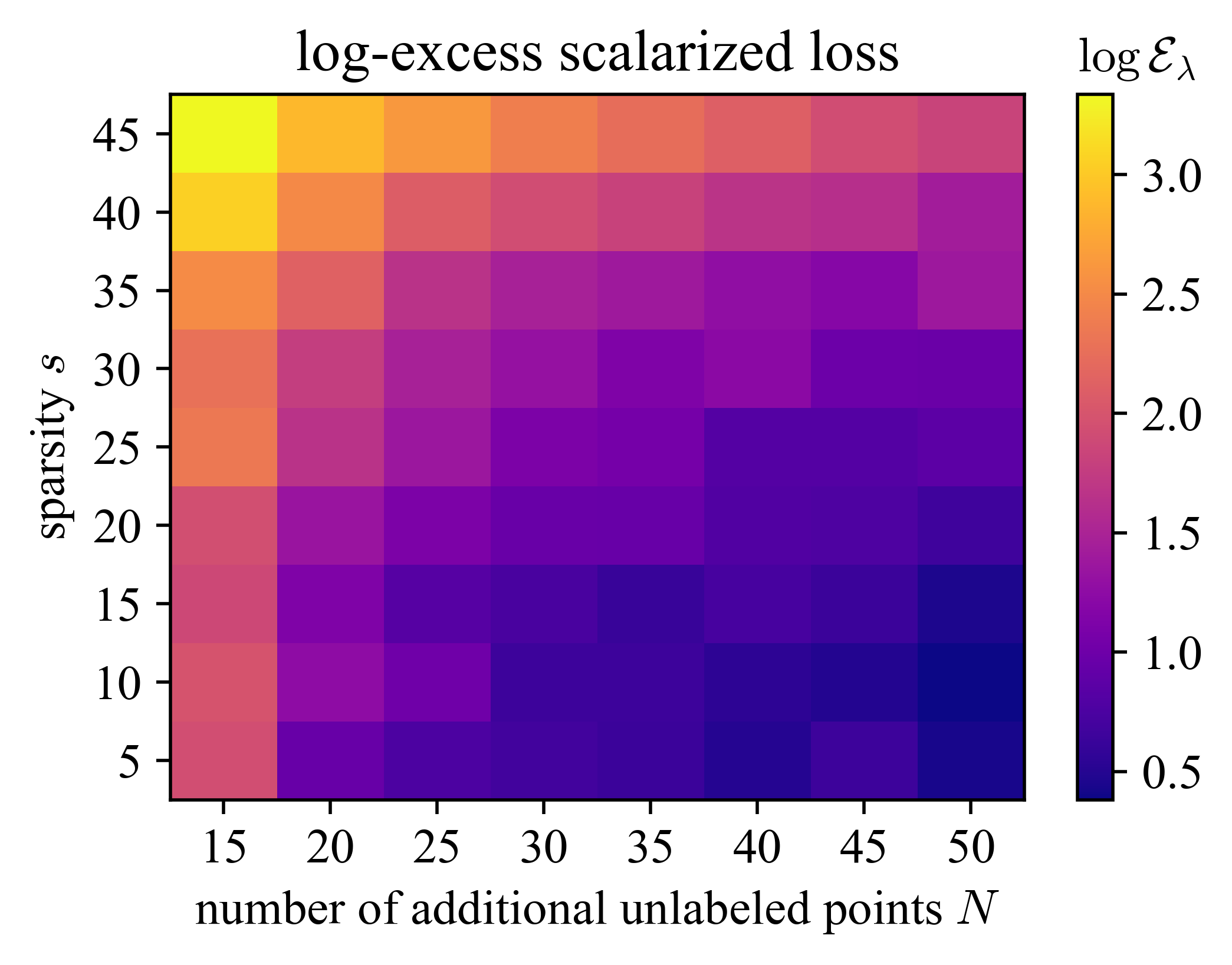}
    \label{fig:covariance-simulation}
    }\hfill
    \subfloat[ ]{
    \includegraphics[width=0.26\linewidth]{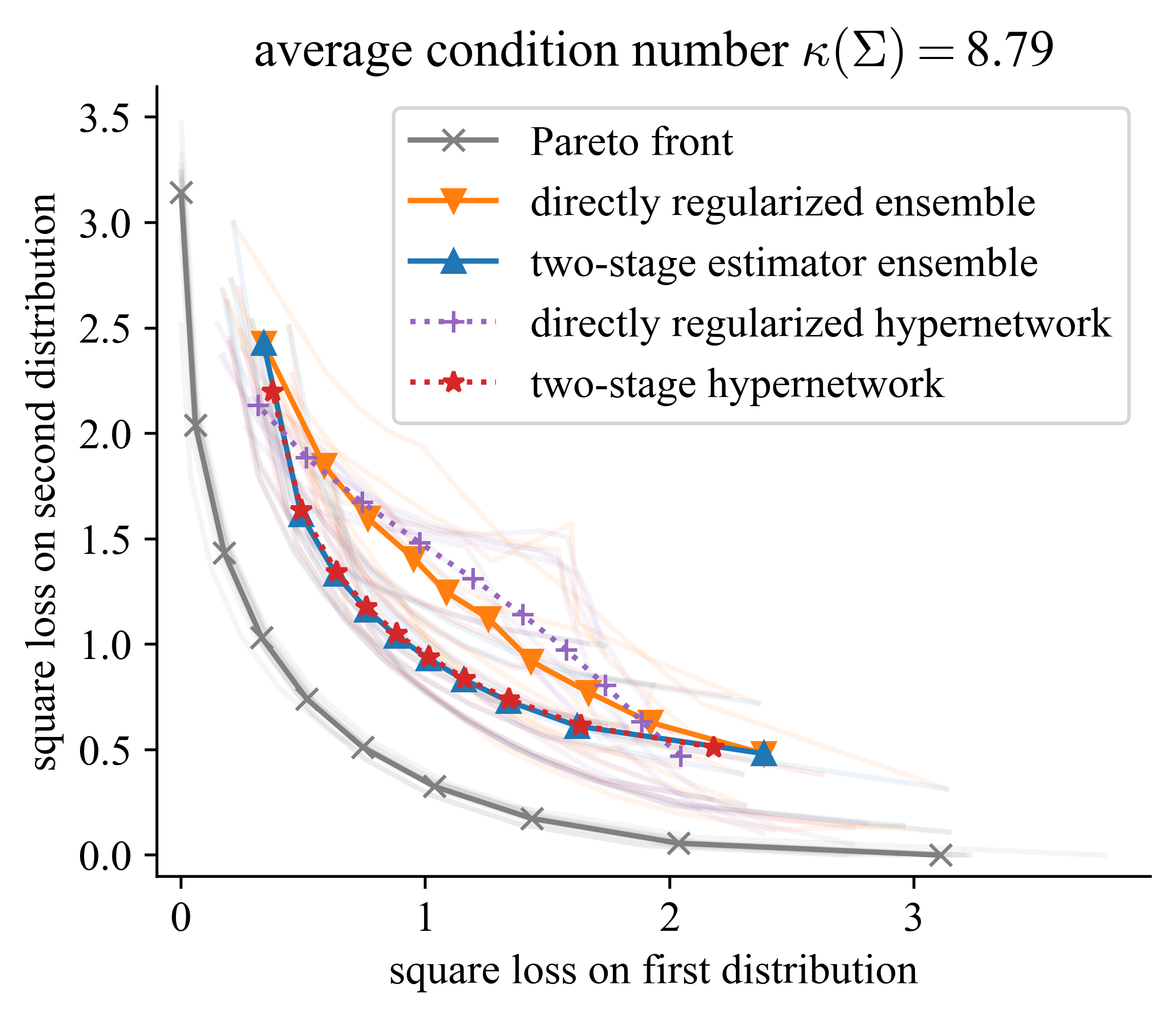}
    \label{fig:multiple_regression}
    }
    \caption{\small{The important roles of both regularization and additional unlabeled data for \cref{ex:multiple-linear-regression} illustrated on an intuitive level \protect\subref{fig:covariance-intuition}, and by evaluating the excess scalarized loss in simulations \protect\subref{fig:covariance-simulation}:
    %The two dimensions in which our estimator improves: 
    Increasing sparsity together with appropriate regularization
    improves the estimate of the parameters $\beta_\objectiveindex$, while 
    %the number of additional 
    an increasing number of unlabeled datapoints $N_\objectiveindex$ improves the estimate of the covariance matrices $\covariance_\objectiveindex$, both improving the estimation of the Pareto front. \protect\subref{fig:multiple_regression}: Pareto fronts for two sparse linear regression problems, using direct regularization and the two-stage approach (\cref{subsec:multi-dist-regression}). We also plot the hypernetwork implementation.
    }} 
    \vspace{-0.1in}
    %Putting both together improves the excess scalarized loss.}
    \label{fig:covariance-effect}
\end{figure*}

\paragraph{Multiple sparse linear regression.} 

Recall \cref{ex:multiple-linear-regression} and consider our two-stage estimator
$\tsEstimatedMinimizerWeighted$, where in stage 1 we estimate $\theta_k=(\beta_k,\Sigma_k)$ with the sample covariance $\estimatedCovariance_k$ (using both the labeled and unlabeled data) and
\begin{align}
    \widehat \beta_k &\in \argmin_{\beta\in \RR^d} \frac{1}{n}\norm{\X_k \beta -y^k}_2^2 +136B\sigma \sqrt{\frac{\log d}{n}} \norm{\beta}_1, \notag \\
     \tsEstimatedMinimizerWeighted & \in \argmin_{\minimizer\in \RR^d} \sum_{k=1}^K \weight_k \norm{\estimatedCovariance_k^{1/2}(\vartheta-\betahat_k)}_2^2. \label{eq:ts-MultiDistribution} 
\end{align}
%and we use the sample covariance matrix  using both labeled and unlabeled data to estimate the covariance $\covariance_k$.
\begin{restatable}{cor}{MultiDistributionOurs}
\label{cor:MultiDistributionOurs}
Let \cref{ass:technical-assumptions-multi-distribution} (in \cref{subsec:proof-MultiDistributionOurs}) hold. Then $\tsEstimatedMinimizerWeighted$ from \eqref{eq:ts-MultiDistribution} achieves for all $\weightTuple\in\Delta^K$ and $\distributionXxYTuple \in \distributionSetTuple$
\begin{equation*}
\norm{\tsEstimatedMinimizerWeighted-\minimizerWeighted}_2\lesssim \frac{B^4}{b^4} \sum_{\objectiveindex=1}^K \weight_k\pr{\frac{\sigma }{b^2} \sqrt{\frac{s\log d}{n_k}} + \sqrt{\frac{d}{n_k+N_k}}}
\end{equation*}
with probability at least $1-c_1K(d^{-3}+\exp(-c_2B^4 d))$, where $c_1,c_2>0$ are some universal constants.
If further $\estimatedCovarianceTuple = \covarianceTuple$ is fixed and known, then
\begin{equation*}
        \minimax(\distributionSetTuple) \gtrsim  \max_{k\in[K]}\weight_k\frac{b^2\sigma}{B^3} \sqrt{\frac{s\log d}{n_k}}.
    \end{equation*}
\end{restatable}
The proof can be found in \cref{subsec:proof-MultiDistributionOurs}.
We can see that when there are enough unlabeled samples, $N_k \gg dn_k /\log d$, both bounds match up to constants.

\paragraph{Fairness-risk tradeoff in linear regression.}
Recall \cref{ex:fairness-risk-trade-off}. As the fairness objective $\Lfair$ violates strong convexity, we have to restrict ourselves to the case that $\weightrisk>0$.
%Our estimator is comparable to the projection estimator from \cite{Gouic2020Projectionfairnessstatisticallearning} when we let $\lambfair\to 1$.
We apply our two-stage estimator with the following natural estimators for $\theta = (\beta, \mu)$ in the first stage
\vspace{-0.2cm}
\begin{align*}
    \widehat \beta &\in \argmin_{\beta\in \RR^d} \frac{1}{n}\norm{\X \beta -y}_2^2 +136\sigma \sqrt{\frac{\log d}{n}} \norm{\beta}_1, \\
    \muhat &= \frac{1}{n+N} \sum_{i=1}^{n+N} A_i X_i,
\end{align*}
\vspace{-0.2cm}
and using linear scalarization, so that $\tsEstimatedMinimizerWeighted$ solves
\begin{equation}
    \label{eq:ts-fairnessCorollary}
    \min_{\minimizer\in \RR^d} \weightrisk \frac{1}{n}\norm{(\identity_d+\muhat\muhat^\top)^{\frac{1}{2}}(\minimizer-\betahat)}_2^2 
    + \weightfair \inner{\minimizer}{\muhat}^2.
\end{equation}   
%\end{equation}
The following corollary applies \cref{thm:MinimaxRateStrongConvexity} to this setting, and its proof can be found in \cref{subsec:proof-FairnessCorollary}.
\begin{restatable}{cor}{FairnessCorollary}
\label{cor:FairnessCorollary}
    In the setting of \cref{ex:fairness-risk-trade-off}, assume that $n\gtrsim \sigma^2 s\log d$. Then the two-stage estmator $\tsEstimatedMinimizerWeighted$ from \eqref{eq:ts-fairnessCorollary} achieves for all $\weightTuple\in\Delta^2$ with $\weightrisk>0$
    \begin{equation*}
        \norm{\tsEstimatedMinimizerWeighted-\minimizerWeighted}_2\lesssim \sigma\sqrt{\frac{s\log d}{n}}+\frac{1}{\weightrisk}\sqrt{\frac{d}{n+N}}
    \end{equation*}
    with probability at least $1-cd^{-3}$, where $c>0$ is some universal constant.
\end{restatable}
If $\mu$ is fixed and known, the matching minimax lower bound on the first term follows from the same arguments as in \cref{cor:MultiDistributionOurs} without any further assumptions.
%Hence, for the estimator to work well, we need to make sure to observe enough samples from a minority group $a$.

\subsection{Necessity of unlabeled data}
\label{subsec:necessity-unlabeled-data}
% Our upper bounds from \cref{thm:MinimaxRateStrongConvexity,prop:LipschitzBound} imply that a small estimation error of all parameters leads to a small estimation error of the Pareto set. In the previous section, we saw two concrete examples of this. 
\Cref{cor:MultiDistributionOurs,cor:FairnessCorollary} showed that provided with sufficiently much unlabeled data\textemdash so that we can estimate the covariance matrices well\textemdash and if the signal is sparse, we can learn the Pareto set efficiently. The corresponding lower bounds show that this rate is tight when the covariance matrices are known. But is the unlabeled data really necessary?
%We now show that indeed, estimating the marginal distribution well using enough unlabeled data can indeed be necessary.
%In \cref{ex:fixed-design-linear-regression} and hence also \cref{prop:InsufficiencyPluginRegularization,prop:TwostageRateFixedDesign}, we ``know the covariate distribution'', as we are in a fixed-design setting. 
%This allows us to highlight the insufficiency of direct regularization, but also hides an important component that we now observed in \cref{cor:MultiDistributionOurs,cor:FairnessCorollary}: The role of unlabeled data.
%In both corollaries, we saw that the two-stage estimator only benefits from the sparsity, provided there is enough unlabeled data. 

We now show that this phenomenon is not an artifact of our bounds
%In an instance of  \cref{ex:multiple-linear-regression}, we show that having enough unlabeled data is \emph{necessary} for mitigating a sample complexity that scales linearly in $d$.
%for benefiting from the sparse ground truths at all, by considering a case of \cref{ex:multiple-linear-regression}.
%\cref{prop:NecessityUnlabeledData} implies that 
using an instance of \cref{ex:multiple-linear-regression}.
In particular, we prove that even when %we know the ground truths 
the ground truths $\beta_k$ are precisely known, having enough unlabeled data to accurately estimate the covariance matrices $\covariance_k$ 
is crucial to efficiently estimate the Pareto set in high dimensions. 
To gain some intuition for this,
we illustrate the effects of both regularization and the estimation of the covariance matrix in \cref{fig:covariance-effect}\subref{fig:covariance-intuition}. Notice how the estimators $\tsEstimatedMinimizerWeighted$ are only close to the Pareto set if both $\beta_\objectiveindex$ and $\covariance_k$ are estimated well.
\begin{restatable}[Necessity of unlabeled data]{prop}{NecessityUnlabeledData}
\label{prop:NecessityUnlabeledData}
    Consider the special case of the statistical model $\distributionSetTuple$ in \cref{ex:multiple-linear-regression}, by restricting $\beta_1=-\beta_2=\beta$ for some $\beta$ with $\norm{\beta}_2=1$ that is fixed and known. 
    Further, for $k\in \mathset{1,2}$, let 
    %The samples are independent draws from 
    $\distributionXxY^k$ be such that $\distributionX^k = \cN(0,\covariance_k)$
    %\times \distributionX^2=\cN(0,\covariance_1) \times \cN(0,\covariance_2)$, 
    with symmetric, unknown covariance matrices $\covariance_k$ satisfying $\frac{1}{2}\I_d  \preceq \covariance_k \preceq \frac{3}{2} \I_d$. %, but unknown.
    Consider $\weightTuple=(1/2,1/2)$ and we observe $n_k=n$ labeled and $N_k=N$ unlabeled datapoints each. Then, if $d\geq 3$ and $\sqrt{d/(512(n+N))} \leq 1/(4e)$, it holds that
    \begin{equation*}
        \minimax(\distributionSetTuple)\gtrsim \sqrt{\frac{d}{n+N}}.
    \end{equation*}
\end{restatable}

Note that, of course, if the covariance matrices have sparse structure, then the amount of labeled data necessary also reduces \citep[Section 6.5]{Wainwright2019}.
The proof of \cref{prop:NecessityUnlabeledData} is provided in \cref{subsec:proof-NecessityUnlabeledData} and follows a similar idea to \cref{prop:InsufficiencyPluginRegularization}, visualized in \cref{fig:proof_illustration}.

%Importantly, this is possible with unlabeled data only.
%Note that this is in stark contrast to single-objective learning, where unlabeled data does not necessarily help the learning process unless there are additional assumptions on the model \cite{?}.

\section{EXPERIMENTS}
\label{sec:experiments}
\begin{figure*}
    \centering
    \subfloat{ % []
        \includegraphics[height=0.175\linewidth]{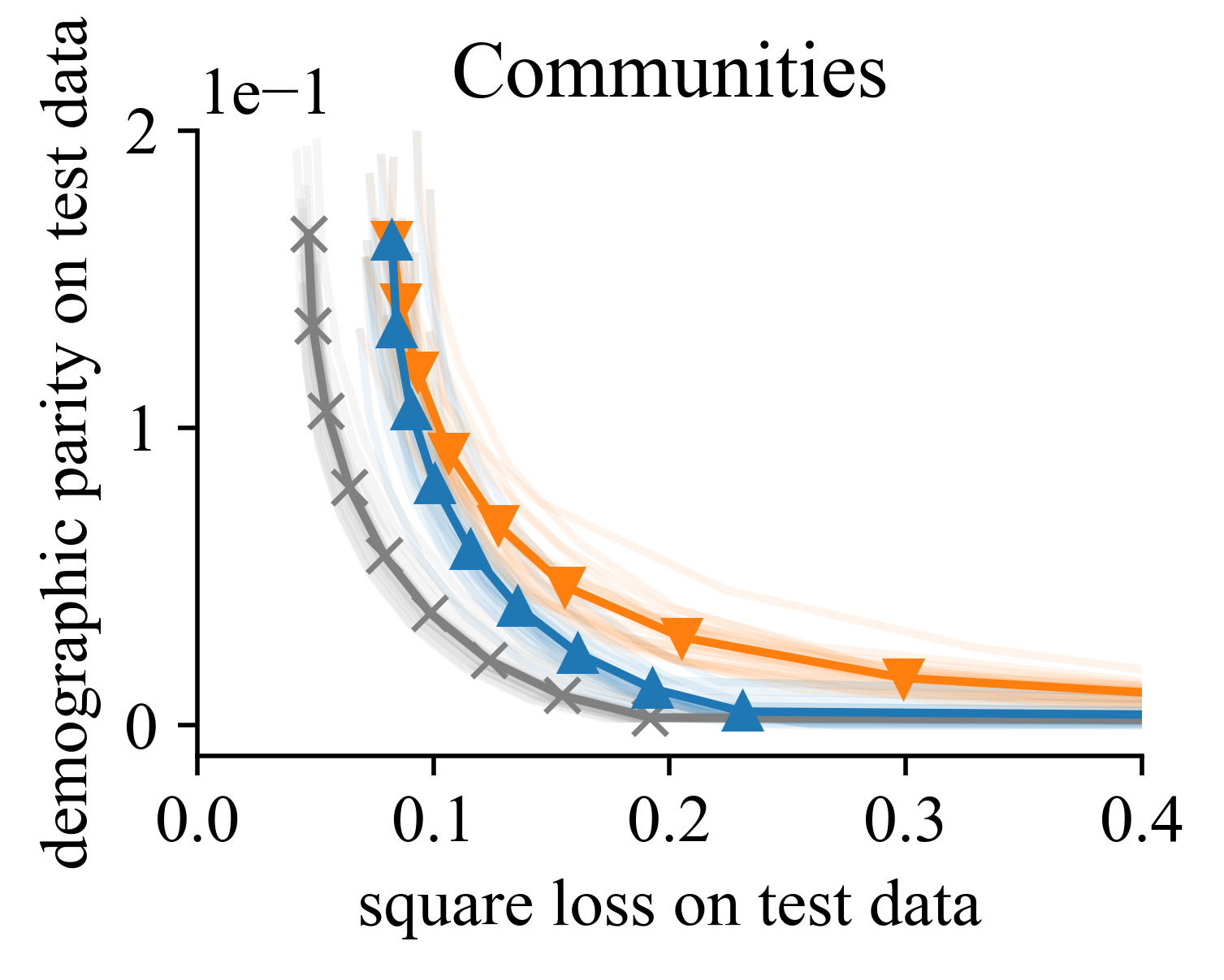}
        \label{fig:crimes-communities}
    }\hspace{-1em}
    \subfloat{ % []
        \includegraphics[height=0.175\linewidth]{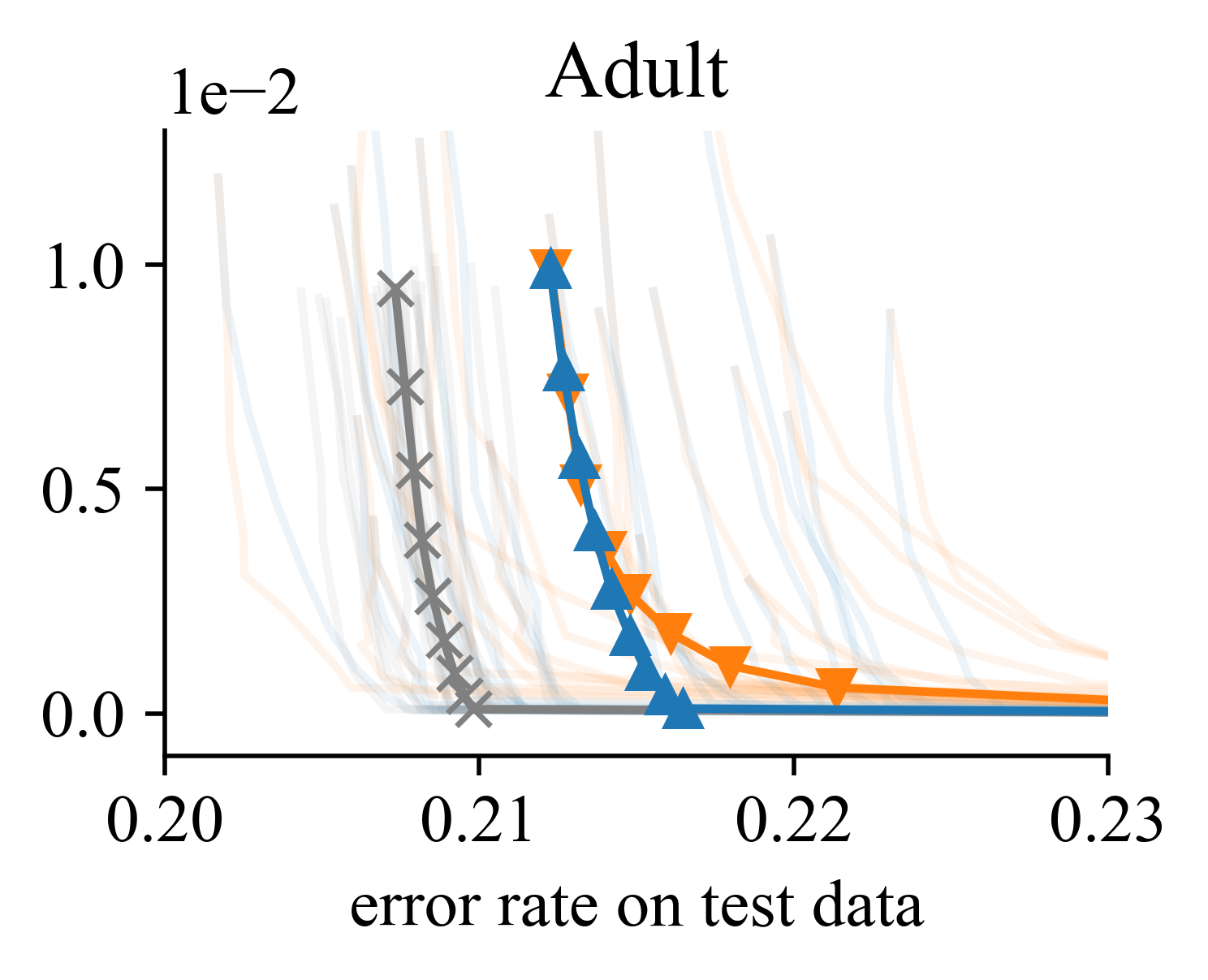}
        \label{fig:adult}
    }\hspace{-1em}
    \subfloat{ %[]
        \includegraphics[height=0.175\linewidth]{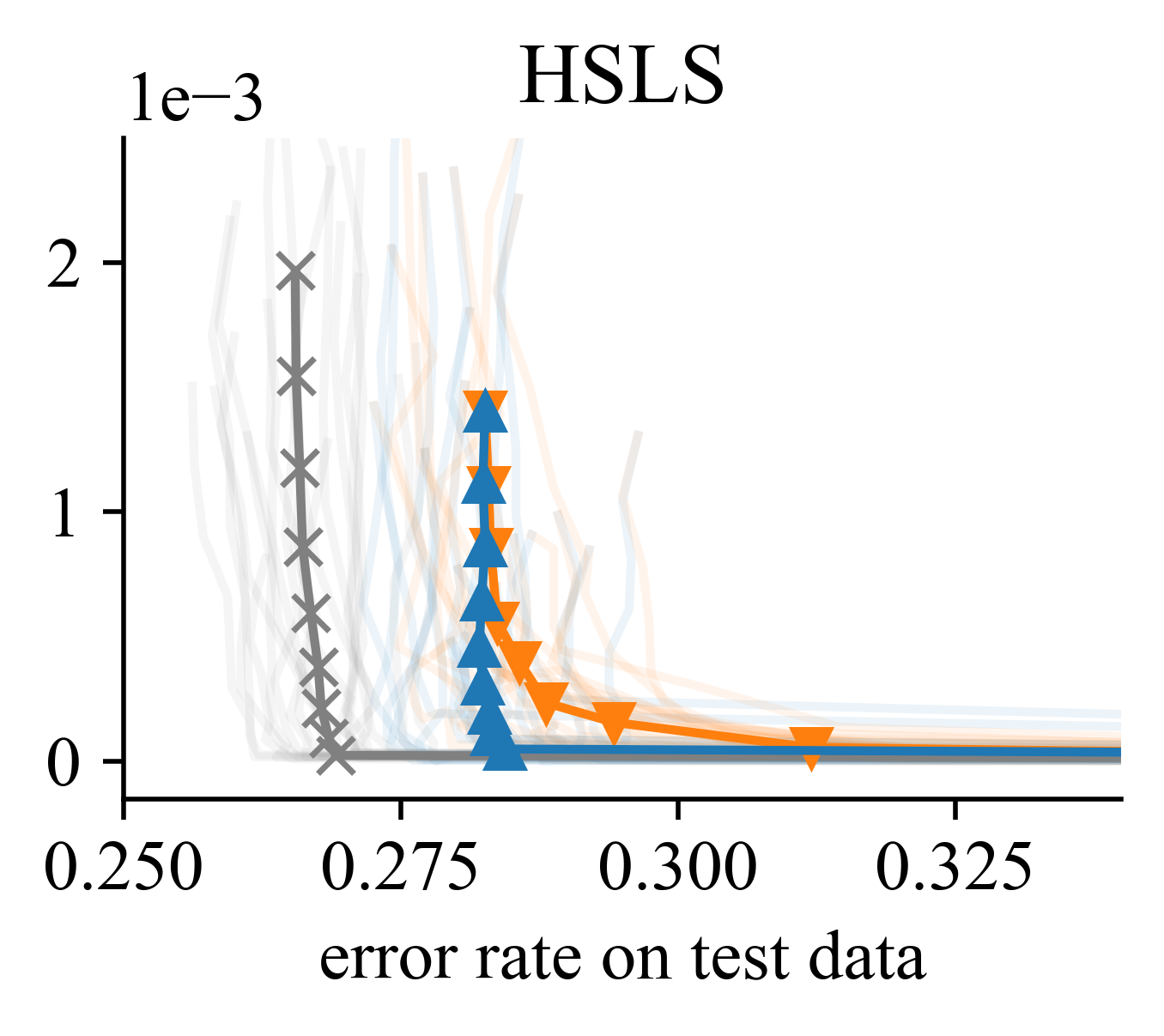}
        \label{fig:hsls}
    }\hspace{-1em}
    \subfloat{ %[ \hspace{2cm} ]
        \includegraphics[height=0.175\linewidth]{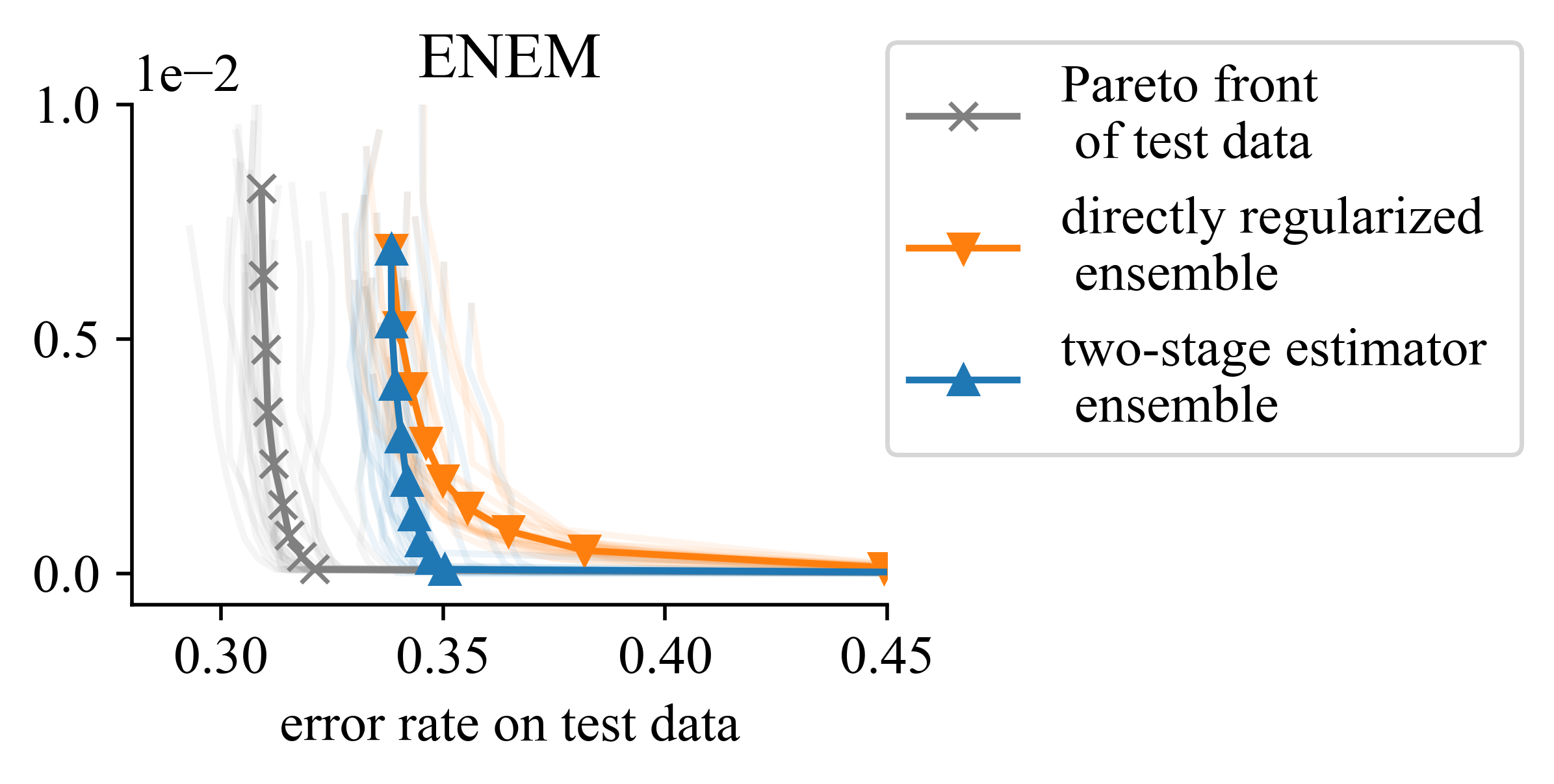}
        \label{fig:enem}
    }
    \caption{\small{Pareto fronts on test data and their estimates using direct regularization (orange) and our method (blue) 
    for the fairness experiments described in \cref{subsec:fairness-risk-experiment,subsec:fairness-datasets}, using data from \citet{Redmond2002Communities,Becker1996Adult,Jeong2022fairness,Alghamdi2022beyond}. 
    Each experiment is repeated 20 times and we plot the results (transparent), as well as their average (thick lines).}
    %\protect\subref{fig:synthetic-multi-dist}: Discussed in \cref{subsec:}
    %\protect\subref{fig:synthetic-multi-dist-2}: 
    %\protect\subref{fig:crimes-communities}: 
    %\protect\subref{fig:kernel-regression}:}
    }
    \label{fig:experiments}
\end{figure*}

In this section, we present some experiments on synthetic  data for \cref{ex:multiple-linear-regression} and real data for \cref{ex:fairness-risk-trade-off} to illustrate the results from \cref{subsec:examples-theory}.

\paragraph{Ensembles and hypernetworks.} So far, we have considered families of estimators $\{\estimatedMinimizerWeighted | \weightTuple\in \Delta^K \}$. In practice, computing such families is not possible, as it would require storing an infinite amount of estimators. Instead, in our experiments, we use two methods for approximating these sets. For one, we use what we call an \emph{ensemble}, which is a finite set of estimators $\{\estimatedMinimizerWeighted | \weightTuple\in \Lambda \}$, where $\Lambda$ is some discretization of the simplex; we use $\abs{\Lambda}=10$.
Secondly, we also use so-called \emph{hypernetworks} \citep{Navon2021learning}, which are neural networks that learn the Pareto set as a function $\widehat h:\Delta^K\to \RR^d$, $\weightTuple\mapsto \estimatedMinimizerWeighted$. For example, for linear models, a prediction on $x$ is then given by $\langle\widehat{h}(\weightTuple),x\rangle$. We give implementation details of hypernetworks for the directly-regularized and two-stage estimators in \cref{sec:implementation-hypernetworks}.

\subsection{Multiple sparse linear regression}
\label{subsec:multi-dist-regression}

The first simulation is on synthetic data in the setting of \cref{ex:multiple-linear-regression} for two objectives. We present two experiments.
In the first experiment, we fix $\weight_k=1/2$, $d=50$, $n_k=15$ and two arbitrarily chosen covariance matrices $\covariance_1,\covariance_2$. The covariates are then sampled from Gaussians. We vary the sparsity $s$ of two random ground-truths (normalized in $\ell_2$-norm) between $5$ and $45$, and the number of additional unlabeled datapoints $N_i$ between $15$ and $50$. For each configuration, we repeat the experiment $10$ times and show the resulting average log-excess scalarized loss $\log \excessscalarization$ from \eqref{eq:excess-scalarized-loss} of the two-stage estimator with appropriately chosen $\ell_1$-penalty in \cref{fig:covariance-effect}\subref{fig:covariance-simulation}.
The smaller the number of non-zero entries $s$ and the more unlabeled data are available, the better the estimator performs, as predicted by \cref{cor:MultiDistributionOurs}.
In the second experiment, we compare our estimator with the directly regularized plug-in estimator for fixed dimension $d=80$ and fixed sample sizes $n_i=25$, $N_i = 60$, with two randomly chosen $1$-sparse ground truths and covariance matrices. 
\cref{fig:covariance-effect}\subref{fig:multiple_regression} shows the Pareto fronts of 50 different runs and their point-wise average. 
%\footnote{See \cite{Bassi2021Uncertainty} on quantifying uncertainty for Pareto fronts.} 
As expected, the biggest benefit of our method lies in the cases where $\weight_1\approx \weight_2$.

%\vspace{-0.2in}
\subsection{Fairness-risk trade-off in linear regression}
\label{subsec:fairness-risk-experiment}

We also apply our estimator to four fairness datasets described in \cref{subsec:fairness-datasets}: The Communities and Crime dataset \citep{Redmond2002Communities}, The Adult dataset \citep{Becker1996Adult}, The HSLS dataset \citep{Ingels2011high,Jeong2022fairness}, and the ENEM dataset \citep{Alghamdi2022beyond}. 
To simulate the (moderately) high-dimensional regime, we both subsample and add noisy features, as described in \cref{subsec:fairness-datasets}. 
Our two-stage estimator and the directly regularized estimator \eqref{eq:penalized-scalarization} are then applied using an $\ell_1$-penalty.
We repeat all experiments $20$ times and show the resulting estimated Pareto fronts, as well as their point-wise average in \cref{fig:experiments}.
On all datasets the two-stage estimator outperforms the directly regularized estimator in parts of the Pareto front where the weights are bounded away from $1$.

%\vspace{-0.1in}
\section{RELATED WORK}
\label{subsec:related-work}

Multi-objective learning is a rich field of research \citep{Jin2007multi,Jin2008Pareto}, rooted in multi-objective optimization \citep{Deist2023Multiobjective,Shah2016Pareto,Duh2012Learning,Hu2023Revisiting} and with connections to many adjacent fields of machine learning such as 
fairness \citep{Yaghini2023Learning, tifrea2023, Martinez2020Minimax},
%efficient communication in 
federated or distributed learning \citep{Kairouz2021Advances,Li2021Ditto,Wang2017Efficient,Wang2020Tackling,Lee2017Communication}, 
%avoidance of catastrophic forgetting in 
continual learning \citep{Parisi2019Continual,Wang2022SparCL}, reinforcement learning \citep{Hayes2022Practical,VanMoffaert2014Novel}, transfer learning \citep{Li2021Transfer} and OOD generalization \citep{Chen2023pair}. 
%. Many works focus on trade-offs between specific objectives (cf.\ \Cref{sec:intro}), and others focus on leveraging particular multi-objective optimization techniques in a learning setting \citep{Deist2021Multiobjectivelearningpredictpareto,Deist2023Multiobjective,Shah2016Pareto,Duh2012Learning,Hu2023Revisiting, tifrea2023}.
%Just to name some examples, 
%In the latter category, 
MOL is closely related to, but distinct from, 
multi-\emph{task} learning. The latter, for example, considers settings where sharing parameters across multiple learning tasks is \emph{helpful} for training task-specific models \citep{Caruana1997Multitask,Baxter2000Model,Sener2018Multitask,Li2020Knowledge}. 
The use of sparsity for multi-task learning has been studied, e.g., in \cite{Lounici2009Taking,Guo2011Sparse}.
%people look at sparsity in context of MOL, not for learning gain, but for 
%Ideas from MOL also frequently appear in adjacent learning fields with sparsity studied for various reasons \fy{didn't understand from "with ..."}, such as efficient communication in federated or distributed learning \citep{Kairouz2021Advances,Li2021Ditto,Wang2017Efficient,Wang2020Tackling,Lee2017Communication}, avoidance of catastrophic forgetting in continual learning \citep{Parisi2019Continual,Wang2022SparCL}, reinforcement learning \citep{Hayes2022Practical,VanMoffaert2014Novel}, and transfer learning \citep{Li2021Transfer}. 
While several works consider special cases of MOL, such as multi-distribution learning  \citep{Haghtalab2023Unifying,Zhang2024Optimal} or fairness \citep{xian23},
%statistical guarantees for MOL have only been 
%The question of provable 
generalization of MOL in its general form is still rather poorly studied, 
with few exceptions like \cite{Sukenik2024generalization,Cortes2020Agnostic,Chamon2020probably}. These works however do not yield meaningful bounds in the high-dimensional regime that we study in this paper. 

Finally, we discuss the body of work on \emph{Pareto set learning} (PSL) \citep{Navon2021learning,Lin2019Pareto,Lin2022pareto,Dimitriadis2023pareto,Chen2024efficient,Tang2024towards} where the word ``learning'' does not refer to generalization as in statistical learning: 
%In this line of work, the work "learning" 
%PSL has somewhat of a misleading name: 
Albeit often being applied in machine learning settings, 
the performance of PSL is usually not measured by comparing the ``learned'' estimators with the ``true'' Pareto-optimal solutions that optimize trade-offs on the test data. 
%the aim of PSL is usually not to improve the generalization error. 
Instead, it aims to reduce the memory or runtime of MOO methods by harnessing, e.g., neural networks to approximate some \emph{deterministic} Pareto set. %Usually, off-the-shelf PSL only uses direct (or no) regularization, and hence is not expected to have a statistical advantage. 
That said, most PSL methods can be used in our two-stage framework in stage 2
%it can be combined with our two-stage method by using the objectives from stage 2 
(cf.\ \cref{subsec:multi-dist-regression,sec:implementation-hypernetworks}).

\section{CONCLUSIONS \& FUTURE WORK}
\label{sec:discussion}
In this work, we propose an estimator for the Pareto front of MOL problems that performs well in the high-dimensional regime by leveraging the sparsity of distributional parameters and unlabeled data. We investigate the estimator theoretically, with a key strength of the analysis being that it is agnostic to the estimators of the distributional parameters.
Further, we prove the optimality of the proposed estimator under certain conditions. 
While the focus of this work is primarily on sparsity, the estimator can also, in principle, exploit other forms of low-dimensional structure. Through synthetic and real experiments, we demonstrate the good performance of our estimator in applications.

Our work opens up many exciting future research directions.
For one, we leave it as future work to derive a more general framework for obtaining lower bounds in the MOL setting beyond identifiability.
We also note that other results from the stability literature could relax the convexity in \Cref{ass:StrongConvexitySmoothness}. Moreover, there are other choices of objectives where our theory can be applied, e.g., in the robustness-accuracy trade-off \citep{Raghunathan2020Understanding}. Finally, the seemingly important role of unlabeled data in multi-objective learning deserves to be investigated further.
%It remains an exciting direction for future research to demonstrate that the proposed estimator performs well in other multi-objective problems. 

\section*{Acknowledgments}
We thank Junhyung Park for valuable feedback on the manuscript.
AT was supported by a PhD fellowship from the Swiss Data Science Center. 
TW was supported by the SNF Grant 204439. This work was done in part while TW 
and FY were visiting the Simons Institute for the Theory of Computing.

\bibliographystyle{abbrvnat}
{
\small
\bibliography{references}
}

\newpage
\section*{Checklist}
\begin{comment}
% %%% BEGIN INSTRUCTIONS %%%
The checklist follows the references. For each question, choose your answer from the three possible options: Yes, No, Not Applicable.  You are encouraged to include a justification to your answer, either by referencing the appropriate section of your paper or providing a brief inline description (1-2 sentences). 
Please do not modify the questions.  Note that the Checklist section does not count towards the page limit. Not including the checklist in the first submission won't result in desk rejection, although in such case we will ask you to upload it during the author response period and include it in camera ready (if accepted).

\textbf{In your paper, please delete this instructions block and only keep the Checklist section heading above along with the questions/answers below.}
% %%% END INSTRUCTIONS %%%
\end{comment}

 \begin{enumerate}

 \item For all models and algorithms presented, check if you include:
 \begin{enumerate}
   \item A clear description of the mathematical setting, assumptions, algorithm, and/or model. Yes (\cref{sec:setting})
   \item An analysis of the properties and complexity (time, space, sample size) of any algorithm. Yes (\cref{sec:theoretical-guarantees})
   \item (Optional) Anonymized source code, with specification of all dependencies, including external libraries. No
 \end{enumerate}

 \item For any theoretical claim, check if you include:
 \begin{enumerate}
   \item Statements of the full set of assumptions of all theoretical results. Yes (Assumptions \ref{ass:Parameterization},\ref{ass:StrongConvexitySmoothness},\ref{ass:Injectivity})
   \item Complete proofs of all theoretical results. Yes (\cref{sec:deferred-proofs})
   \item Clear explanations of any assumptions. Yes (\cref{sec:theoretical-guarantees})  
 \end{enumerate}

 \item For all figures and tables that present empirical results, check if you include:
 \begin{enumerate}
   \item The code, data, and instructions needed to reproduce the main experimental results (either in the supplemental material or as a URL). Yes
   \item All the training details (e.g., data splits, hyperparameters, how they were chosen). Yes (\cref{sec:experiments})
   \item A clear definition of the specific measure or statistics and error bars (e.g., with respect to the random seed after running experiments multiple times). Yes (\cref{sec:experiments})
   \item A description of the computing infrastructure used. (e.g., type of GPUs, internal cluster, or cloud provider). No (The experiments are very small-scale, run on a standard MacBook Pro in under 5 minutes)
 \end{enumerate}

 \item If you are using existing assets (e.g., code, data, models) or curating/releasing new assets, check if you include:
 \begin{enumerate}
   \item Citations of the creator If your work uses existing assets. Yes (\cref{sec:experiments})
   \item The license information of the assets, if applicable. Not Applicable (We are not releasing new or existing assets)
   \item New assets either in the supplemental material or as a URL, if applicable. Not Applicable (We are not releasing new or existing assets)
   \item Information about consent from data providers/curators. Not Applicable (All datasets are licensed under a Creative Commons Attribution 4.0 International (CC BY 4.0) license.)
   \item Discussion of sensible content if applicable, e.g., personally identifiable information or offensive content. Not Applicable
 \end{enumerate}

 \item If you used crowdsourcing or conducted research with human subjects, check if you include:
 \begin{enumerate}
   \item The full text of instructions given to participants and screenshots. Not Applicable
   \item Descriptions of potential participant risks, with links to Institutional Review Board (IRB) approvals if applicable. Not Applicable
   \item The estimated hourly wage paid to participants and the total amount spent on participant compensation. Not Applicable
 \end{enumerate}

\end{enumerate}

%\usepackage{aistats2025}
% If your paper is accepted, change the options for the package
% aistats2025 as follows:
%
%\usepackage[accepted]{aistats2025}
%
% This option will print headings for the title of your paper and
% headings for the authors names, plus a copyright note at the end of
% the first column of the first page.

% If you set papersize explicitly, activate the following three lines:
%\special{papersize = 8.5in, 11in}
%\setlength{\pdfpageheight}{11in}
%\setlength{\pdfpagewidth}{8.5in}

% If you use natbib package, activate the following three lines:
%\usepackage[round]{natbib}
%\renewcommand{\bibname}{References}
%\renewcommand{\bibsection}{\subsubsection*{\bibname}}

% If you use BibTeX in apalike style, activate the following line:
%\bibliographystyle{apalike}
\onecolumn

\begin{appendices}

\crefalias{section}{appendix}
\crefalias{subsection}{appendix}

\counterwithin{assumption}{section}
\counterwithin{lemma}{section}
\counterwithin{claim}{section}

\section{PRELIMINARIES FROM CONVEX OPTIMIZATION}
\label{sec:basics-convex-optimization}

In this section, we briefly recall some basic concepts from convex optimization. A general introduction to convex optimization can be found, for example, in \cite{Boyd2004Convex, Bubeck2015Convex}.

Let $f:\RR^d\to \RR$ denote a differentiable function with gradient $\nabla f:\RR^d\to\RR^d$.
We recall the definitions of convexity, strict convexity, strong convexity and smoothness:
For some $\strongConvexityParam\geq 0$, the function $f$ is called \emph{$\strongConvexityParam$-strongly convex}, if for all $x,y\in\RR^d$ it holds that $f(x)-f(y)-\inner{\nabla f(y)}{x-y}\geq \frac{\strongConvexityParam}{2}\norm{x-y}_2^2$. If $f$ is $0$-strongly convex, it is simply convex.
For some $0\leq \nu<\infty $, the function $f$ is called \emph{$\nu$-smooth}, if its gradient is $\nu$-Lipschitz, that is, for all $x,y\in\RR^d$ it holds that $\norm{\nabla f(x)-\nabla f(y)}_2\leq \nu\norm{x-y}_2$.

We use the following well-known facts multiple times throughout the proofs of our results in \cref{sec:deferred-proofs}.
Let $f_1,\dots,f_K:\RR^d\to\RR$ be differentiable, ${\weightTuple}\in\Delta^K$ and denote $f_{\weightTuple} = \sum_{\objectiveindex=1}^K \weight_\objectiveindex f_\objectiveindex$.
\begin{enumerate}[label=(A.\arabic*), ref=(A.\arabic*), leftmargin=1cm]
\setlength\itemsep{0em}
    \item \label{item:implication-strong-convexity} If $f$ is $\strongConvexityParam$-strongly convex, then $\forall x,y\in \RR^d,\ \norm{\nabla f(x)-\nabla f(y)}_2\geq \strongConvexityParam \norm{x-y}_2$.
    \item \label{item:implication-smoothness} If $f$ is $\nu$-smooth, then $\forall x,y\in \RR^d,\ \abs{f(x)-f(y)-\inner{\nabla f(y)}{x-y}}\leq \frac{\nu}{2}\norm{x-y}_2^2$.
    \item \label{item:weighted-sum-strongly-convex} If $f_1,\dots,f_K$ are $\strongConvexityParam_k$-strongly convex functions, then the function $f_{\weightTuple}$ is $\sum_{\objectiveindex=1}^K\weight_\objectiveindex\strongConvexityParam_\objectiveindex$-strongly convex. 
    \item \label{item:weighted-sum-smooth} If $f_1,\dots,f_K$ are $\nu_\objectiveindex$-smooth functions, then the function $f_{\weightTuple}$ is $\sum_{\objectiveindex=1}^K\weight_\objectiveindex\nu_\objectiveindex$-smooth.
    \item \label{item:mixture-quadratics} If for $k\in[K]$ we have $f_k(x)= (x-y_k)^\top M_k (x-y_k)$ with $M_k\in\RR^{d\times d}$ and $y_k\in \RR^d$, then it holds that
    \begin{equation*}
        \argmin_{x\in \RR^d} f_{{\weightTuple}}(x) = \mathset{\pr{\sum_{k=1}^K\weight_k M_k}^{\dagger} \pr{\sum_{k=1}^K\weight_kM_ky_k}+z\setmid z\in\ker\pr{\sum_{k=1}^K\weight_k M_k}},
    \end{equation*}
    where $\dagger$ denotes the Moore–Penrose inverse.
\end{enumerate}

\section{DETAILS FOR THE FAIRNESS-RISK TRADEOFF IN LINEAR REGRESSION}
\label{sec:fairness-regression-appendix}

In this section, we describe in more detail the fairness metric from \cref{ex:fairness-risk-trade-off} that used in \cref{subsec:examples-theory} and \cref{sec:experiments}, and elaborate on its interpretation.

Recall the setting described in \cref{ex:fairness-risk-trade-off}:
The random variables $Y,X,A$ are distributed according to $Y=\inner{X}{\beta}+\xi$ with an $s$-sparse ground-truth $\beta\in\RR^d$, where $\xi\sim \cN(0,\sigma^2)$, and for $a\in\mathset{-1,1}$ we have $(X|A=a)\sim \cN(a\mu,\identity_d)$, where $A$ is a Rademacher random variable that is uniformly distributed on $\mathset{-1,1}$. $A$ represents an observed binary protected group attribute (such as gender or ethnicity), of two groups that have different covariate means, which for simplicity we model as $\EE\br{X|A=\pm 1}=\pm \mu$.

In this context, a commonly used fairness metric is called \emph{demographic parity}, which asserts that predictions of the machine learning model $f_\vartheta$ are independent of the group attribute $A$, that is, the model $f_\vartheta$ satisfies demographic parity, if
\begin{equation*}
    (f_\vartheta(X)|A=1) \stackrel{\law}{=} (f_\vartheta(X)|A=-1)
\end{equation*}
where $\stackrel{\law}{=}$ denotes equality in distribution. However, the assertion of demographic parity is quite strong, as it imposes an exact equality. Instead, often it is preferred to measure how much the predictor $f_\vartheta$ violates this constraint.

To measure this, we consider a notion of unfairness introduced in a recent line of work \citep{Gouic2020Projectionfairnessstatisticallearning,Chzhen2022,Fukuchi2024Demographic}. The demographic parity score that is based on the $2$-Wasserstein distance. Most of this appendix section is based on \cite{Chzhen2022} and references therein.
Denote $\cP_2(\RR)$ the space of probability distributions $\nu$ on $\RR$ with bounded second moment, that is, if $X\sim \nu$, then $\EE\br{X^2}<\infty$.
The $2$-Wasserstein distance on $\cP_2(\RR)$, denoted $W_2^2$, is defined as
\begin{equation*}
    W_2^2(\mu,\nu)=\inf_{\pi\in \Pi(\mu,\nu)} \int_{\RR^2} \abs{x-y}^2 d\pi(x,y)
\end{equation*}
where $\Pi(\mu, \nu)$ denotes the set of measures on $\RR^2$ with marginals $\mu$ and $\nu$. See \cite[Appendix A.1]{Chzhen2022} for more details.
The fairness notion that measures violation of demographic parity is then defined as the $2$-Wasserstein distance between the group-wise distributions of $\inner{X}{\vartheta}|A=a$, denoted $\law\pr{\inner{X}{\vartheta}|A=a}$ and their barycenter;
\begin{equation}
    \Lfair(\vartheta,\distributionXxY) = \min_{\nu\in \cP_2(\RR)} \Big\{\frac{1}{2} W_2^2\pr{\law(\inner{X}{\vartheta}\setmid A=1), \nu} + \frac{1}{2}W_2^2\pr{\law(\inner{X}{\vartheta}\setmid A=-1), \nu}  \Big\}.
    \label{eq:bary_fairness}
\end{equation}
We now show that under our assumptions, this fairness metric can be rewritten in the simple form $\Lfair(\vartheta,\distributionXxY) =\inner{\mu}{\vartheta}^2$. 
\begin{lemma}
\label{lem:fairness-objective-rewrite}
    Under the distribution described above, we have that
    \begin{equation*}
        \Lfair(\vartheta,\distributionXxY) = \inner{\mu}{\vartheta}^2.
    \end{equation*}
\end{lemma}
\begin{proof}
    Recall that for $a\in\mathset{-1,1}$ we have $(X|A=a) \sim \cN(a\mu,\identity_d)$, and hence $(\inner{X}{\vartheta}|A=a)\sim \cN(\inner{a\mu}{\vartheta},\norm{\vartheta}_2^2) $. Therefore, by \citet[Lemma A.2]{Chzhen2022}, the optimization problem in \cref{eq:bary_fairness}
    % \begin{equation*}
    %     \min_{\nu\in \distributionSetTuple_2(\RR)} \Big\{\frac{1}{2} W_2^2\pr{\law(\inner{X}{\vartheta}\setmid A=1), \nu} + \frac{1}{2}W_2^2\pr{\law(\inner{X}{\vartheta}\setmid A=2), \nu}  \Big\}
    % \end{equation*}
    is solved by $\nu = \cN(0,\norm{\vartheta}_2^2)$. Further, by \citet[Lemma A.1]{Chzhen2022}, we can plug this into $\Lfair$ and get
    \begin{align*}
        \Lfair(\minimizer, \distributionXxY)= &= \min_{\nu\in \cP_2(\RR)} \Big\{\frac{1}{2} W_2^2\pr{\law(\inner{X}{\vartheta}\setmid A=1), \nu} + \frac{1}{2}W_2^2\pr{\law(\inner{X}{\vartheta}\setmid A=-1), \nu}  \Big\}\\
        &= \frac{1}{2}W_2^2\pr{\cN(\inner{\mu}{\vartheta},\norm{\vartheta}_2^2), \cN(0,\norm{\vartheta}_2^2)} + \frac{1}{2}W_2^2\pr{\cN(\inner{-\mu}{\vartheta},\norm{\vartheta}_2^2), \cN(0,\norm{\vartheta}_2^2)} \\
        &= \frac{1}{2}\inner{\mu}{\vartheta}^2 + \frac{1}{2} \inner{-\mu}{\vartheta}^2 \\
        &= \inner{\mu}{\vartheta}^2,
    \end{align*}
    which concludes the proof.
\end{proof}
Based on this reformulation, it is easy to see that unless $\inner{\mu}{\minimizer}=0$, there is a trade-off between fairness and risk, since the risk is minimized at $\minimizer$ and the demographic parity score is minimized by any vector $\vartheta$ so that $\inner{\mu}{\vartheta}=0$. 
Moreover, this reformulation significantly speeds up the computation, which is why we use it in our experiments in \cref{sec:experiments}. Hence, we are implicitly modeling the data to come from a Gaussian model as described above.

\subsection{Description of fairness datasets}
\label{subsec:fairness-datasets}

We briefly describe the four datasets and their usage in our experiments in \cref{subsec:fairness-risk-experiment}.

\begin{itemize}
    \item For the Communities and Crime dataset \citep{Redmond2002Communities}, the task is to predict the number of violent crimes in a community. The protected attribute is a quantization of the share of white residents of the community.
    To simulate the (moderately) high-dimensional regime for the Communities and Crime dataset (data dimension $d=145$), we subsample uniformly $n=150$ labeled and $N=350$ unlabeled datapoints, and use the remaining samples as test samples to estimate risk and fairness scores from \cref{subsec:examples-theory}.
    \item For the Adult dataset \citep{Becker1996Adult}, the task is to predict the income of individuals, and the protected attribute is their gender. Since the Adult dataset only has dimension $d=13$, additionally to subsampling, we add $1000$ noisy features (sampled from a Gaussian) to artificially increase the data dimension to $d=1013$. We then uniformly sample $n=1000$ labeled and $N=2000$ unlabeled examples, with the remaining samples serving as the test set.
    \item For the high-school longitudinal study (HSLS) dataset \citep{Ingels2011high,Jeong2022fairness}, the target is to predict math test performance of 9th-grade high-school students, and the protected attribute is the students’ ethnicity. As the data dimension is $d=59$, we subsample $n=1000$ labeled and $N=4000$ unlabeled datapoints.
    \item In the ENEM dataset \citep{Alghamdi2022beyond}, the aim is to predict Brazilian college entrance exam scores based on the students' demographic information and socio-economic questionnaire answers, and the protected attribute is also their ethnicity. Here, the dimension is $d=139$, and we subsample $n=2000$ labeled and $N=8000$ unlabeled datapoints.
\end{itemize}

Note that the risk in \cref{ex:fairness-risk-trade-off,subsec:examples-theory,subsec:fairness-risk-experiment} is for regression. The datasets Adult, HSLS and ENEM are usually classification tasks.
Hence, for comparability with other works, we report the error rate rather than the mean squared error in \cref{fig:experiments}, while using the squared loss risk during training\textemdash similar to previous work such as \cite{Berk2017Convexframeworkfairregression}. The benefit of the two-stage estimator empirically transfers between these two metrics.

\section{IMPLEMENTATION OF HYPERNETWORKS}
\label{sec:implementation-hypernetworks}

The hypernetworks $h$ used in \cref{sec:experiments} all have the following architecture. The first layer is linear, where the weight matrix has dimension $128\times K$, the second layer is a component-wise ReLU layer, and the third layer is again linear of dimension $d\times 128$. To select the regularization strengths, we perform validation on held-out sets.
To train the hypernetworks, we use an adapted version of Algorithm 1 in \cite{Navon2021learning}.  We specify the two versions below (\cref{alg:dr-hypernetwork,alg:ts-hypernetwork}), where $\text{Dir}(\cdot)$ denotes the Dirichlet distribution. We use the PyTorch implementation of Adam \citep{Kingma2014adam} to optimize.

\begin{minipage}{0.48\textwidth}
    \begin{algorithm}[H]
    \caption{Training directly regularized hypernetwork}\label{alg:dr-hypernetwork}
        \begin{algorithmic}[1]
           \State Input: number of iterations $T$, candidate regularization strength $\alpha$.
            \For{$t=1,\dots, T$}
            \State Sample $\weightTuple\sim \text{Dir}(\frac{1}{K},\dots, \frac{1}{K})$
            \State Adam step on $\scalarizationObjectiveComposition(h(\weightTuple),\empdistributionXxYTuple) +\alpha\norm{h(\weightTuple)}_1$ 
            \Statex \hspace{1.1em} with respect to the weights of $h$, using learning
            \Statex \hspace{1.1em} rate $10^{-3}$.
            \EndFor \\
            \Return hypernetwork $h$
        \end{algorithmic}
    \end{algorithm}
\end{minipage}
\hfill
\begin{minipage}{0.48\textwidth}
    \begin{algorithm}[H]
    \caption{Training two-stage hypernetwork}\label{alg:ts-hypernetwork}
        \begin{algorithmic}[1]
           \State Input: number of iterations $T$, estimates $\estimatedParameterTuple$.
            \For{$t=1,\dots, T$}
            \State Sample $\weightTuple\sim \text{Dir}(\frac{1}{K},\dots, \frac{1}{K})$
            \State Adam step on $\scalarizationObjectiveComposition(h(\weightTuple),\estimatedParameterTuple)$ with respect
            \Statex \hspace{1.1em} to the weights of $h$, using learning rate $10^{-3}$.
            \EndFor \\
            \Return hypernetwork $h$
        \end{algorithmic}
    \end{algorithm}
\end{minipage}

\section{DEFERRED PROOFS}
\label{sec:deferred-proofs}

In this section we provide the deferred proofs.

\subsection{Proof of Proposition \ref{prop:HypervolumeBound}}
\label{subsec:proof-HypervolumeBound}
\HypervolumeBound*
\begin{proof}[Proof of \cref{prop:HypervolumeBound}]   
    The proof of the first two claims are straight-forward consequences of smoothness and the definition of $G_k$. For the last bound, we use a representation via an expected random hypervolume scalarization from \cite{Zhang2020random}.

    We begin with the first bound. By \ref{item:weighted-sum-strongly-convex}, when $\scalarization$ is linear, that is, $\scalarization(x)= \sum_{i=1}^K \weight_i x_i$, $\smoothnessParam_k$-smoothness of $\objectiveindexed$ implies $\scalarization(\boldsymbol{\smoothnessParam})$-smoothness of $\scalarizationObjectiveComposition$ and so by \cref{item:implication-smoothness} it holds
    \begin{align*}
        \excessscalarization(\estimatedMinimizerWeighted)&=\scalarizationObjectiveComposition(\estimatedMinimizerWeighted)-\scalarizationObjectiveComposition(\minimizerWeighted) \\
        &\leq \inner{\nabla_\vartheta\scalarizationObjectiveComposition(\minimizerWeighted)}{\estimatedMinimizerWeighted-\minimizerWeighted}+\frac{\scalarization(\boldsymbol{\smoothnessParam})}{2}\norm{\estimatedMinimizerWeighted-\minimizerWeighted}_2^2 \\
        &= \frac{\scalarization(\boldsymbol{\smoothnessParam})}{2}\norm{\estimatedMinimizerWeighted-\minimizerWeighted}_2^2 \\
        &= \eps(0,\scalarization(\boldsymbol{\smoothnessParam}),\weightTuple).
    \end{align*}
    where we used the stationarity condition $\nabla_\vartheta\scalarizationObjectiveComposition(\minimizerWeighted)=0$.
    
    Smoothness, the definition of $G_k$, and Cauchy-Schwarz also imply that
    \begin{align*}
        \objectiveindexed(\estimatedMinimizerWeighted)-\objectiveindexed(\minimizerWeighted) 
        &\leq \inner{\nabla_\vartheta\objectiveindexed(\minimizerWeighted)}{\estimatedMinimizerWeighted-\minimizerWeighted}+\frac{\smoothnessParam_k}{2}\norm{\estimatedMinimizerWeighted-\minimizerWeighted}_2^2 \\
        &\leq G_k\norm{\estimatedMinimizerWeighted-\minimizerWeighted}_2+\frac{\smoothnessParam_k}{2}\norm{\estimatedMinimizerWeighted-\minimizerWeighted}_2^2 \\
        &=:\eps(G_k,\smoothnessParam_k,\weightTuple).
    \end{align*}

    Finally, we use a representation of hypervolume as an expected random scalarizations to prove the bound on the hypervolume. To that end, we use a version of the arguments in \cite{Zhang2020random}, which we prove after finishing the main proof.
    \begin{lemma}
    \label{lem:hypervolume-scalarization}
        Denote the positive $(K-1)$-dimensional sphere as $\SSS^{K-1}_+=\mathset{v\in\RR^K\setmid \norm{v}_2=1, \forall i\in[K]\; v_i\geq 0 }$. Moreover, define $U$ to be the uniform probability measure on $\SSS_+^{K-1}$ with Borel $\sigma$-algebra (also known as Haar measure), that is, for all Borel sets $A\subset \SSS_+^{K-1}$ we define $U(A)=c_K^{-1}\cdot \mu_K\pr{\mathset{tv \setmid t\in [0,1], v\in A}}$ where $\mu_K$ is the Lebesgue measure on $\RR^K$ and $c_K = \pi^{K/2}/(2^K\Gamma(K/2+1))$. Here $\Gamma$ denotes the gamma function.
        Let $r>0$ and $\cS\subset [0,r]^K$. Then
        \begin{equation*}
            \hypervolume_r(\cS) = c_K \EE_{u\sim U}\br{\max_{s\in \cS} \min_{k\in [K]}\pr{\frac{r-s_k}{u_k}}^K}.
        \end{equation*}
    \end{lemma}
    Plugging our sets $\PF$ and $\PFhat$ in as $\cS$ from \cref{lem:hypervolume-scalarization}, we have for the constant $c_K$ from \cref{lem:hypervolume-scalarization}  and $\delta = \min_{k\in[K]}\min_{\weightTuple\in\Delta^K}(r-\objectiveindexed(\estimatedMinimizerWeighted))/(r-\objectiveindexed(\minimizerWeighted))$ that 
    \begin{align*}
        \hypervolume_r(\PFhat)  &= c_K \EE_{u\sim U}\br{\max_{\weightTuple\in\Delta^K}\min_{k\in[K]} \pr{\frac{r-\objectiveindexed(\estimatedMinimizerWeighted)}{u_k}}^K} \\
        %&= c_K \EE_{u\sim U}\br{\max_{\weightTuple\in\Delta^K}\min_{k\in[K]} \pr{\frac{r-\objectiveindexed(\minimizerWeighted)}{u_k}}^K\pr{\frac{r-\objectiveindexed(\estimatedMinimizerWeighted)}{r-\objectiveindexed(\minimizerWeighted)}}^K} \\
        &\geq c_K \EE_{u\sim U}\br{\max_{\weightTuple\in\Delta^K}\min_{k\in[K]} \pr{\frac{r-\objectiveindexed(\minimizerWeighted)}{u_k}}^K}\delta^K \\
        &= \hypervolume_r(\PF)\delta^K.
    \end{align*}
    By definition of $r$, which implies that $ r-\objectiveindexed(\minimizerWeighted)\geq r/2$, and the definition of $\eps(G,\smoothnessParam,\weightTuple)$ we have that
    \begin{align*}
        \delta &= \min_{k\in[K]}\min_{\weightTuple\in\Delta^K}\frac{r-\objectiveindexed(\estimatedMinimizerWeighted)}{r-\objectiveindexed(\minimizerWeighted)} \\
        &\geq \min_{k\in[K]}\min_{\weightTuple\in\Delta^K}\frac{r-\objectiveindexed(\minimizerWeighted)-\eps(G_k,\smoothnessParam_k,\weightTuple)}{r-\objectiveindexed(\minimizerWeighted)} \\
        &\geq 1-\max_{k\in[K]}\max_{\weightTuple\in\Delta^K}\frac{2\eps(G_k,\smoothnessParam_k,\weightTuple)}{r}\\
        &= 1-\frac{2 \eps_{\max}}{r}.
    \end{align*}
    Plugging this in yields the last bound and finishes the proof.
\end{proof}

\begin{proof}[Proof of \cref{lem:hypervolume-scalarization}]
    The proof is analogous to the proof of \citet[Lemma 5]{Zhang2020random} with minor adjustments.

    Considering the hyper-rectangle $R_s$ with corners $\rbold = (r,\dots,r)^\top \in \RR^K$ and $s\in \cS$, that is, $R_s=\bigtimes_{k=1}^K [s_k, r]$. Take any $v\in \SSS_+^{k-1}$ and consider the ray $\mathset{\rbold-t v\setmid t\geq 0}$. Let $p \in \RR^K$ be the point where the ray exits the rectangle. It is easy to see that $p=\rbold-\min_{k\in [K]}\pr{\frac{r_k-s_k}{v_k}}v$ because $p=r-tv$ and $t$ must be maximal so that $p_i = r-tv_i \geq s_i$. Now, if we extend this argument to the entire dominated set, 
    \begin{equation*}
        \cD_{\cS}=\mathset{x\in[0,r]^K\setmid \exists s\in \cS: x \succeq s} = \bigcup_{s\in \cS} R_s,
    \end{equation*}
    which is the union over such rectangles, we see that the point where the ray exits is given by $p=\rbold-t_v v$, where
    \begin{equation*}
        t_v = \max_{s\in \cS} \min_{k\in [K]} \pr{\frac{r_k-s_k}{v_k}} \quad \text{and hence} \quad  \ind_{\cD_\cS}(\rbold - t v) = 
        \begin{cases}
            1 & \text{if } t\in [0,t_v], v\in \SSS_+^{K-1}, \\
            0 & \text{else}.
        \end{cases}
    \end{equation*}

    Denote $\mu_K$ the Lebesgue measure on $\RR^K$. By \citet[Theorem 2.49]{Folland1999real}, the Borel measure $\sigma$ on $\SSS^{K-1}$, defined for any Borel measurable $A\subset \SSS^{K-1}$ as $\sigma (A)= K \cdot \mu_K\pr{\mathset{tv\setmid t\in (0,1], v\in A}}$\textemdash an unnormalized uniform measure on $\SSS^{K-1}$\textemdash satisfies
    \begin{align*}
        \mu_K\pr{\cD_\cS}&=\int_{\RR^K} \ind_{\cD_\cS}(x) d\mu_K(x) \\
        &= \int_{\RR^K} \ind_{\cD_\cS}(\rbold-x) d\mu_K(x) \\
        &= \int_{(0,\infty)} \int_{\SSS^{K-1}} t^{K-1}\ind_{\cD_\cS}(\rbold-tv) d\sigma(v) d\mu_1(t) \\
        &= \int_{\SSS^{K-1}_+} \int_0^{t_v} t^{K-1} d\mu_1(t) d\sigma(v) \\
        &= \frac{1}{K}\int_{\SSS^{K-1}_+} \max_{s\in \cS}\min_{k\in [K]}\pr{\frac{r_k-s_k}{v_k}}^{K} d\sigma(v).
    \end{align*}
    Now, since $\sigma (\SSS^{K-1})= 2\pi^{K/2}/ \Gamma(K/2)$ \citep[Proposition 2.54]{Folland1999real}, we have that $\sigma (\SSS_+^{K-1})= 2\pi^{K/2}/ (\Gamma(K/2)2^K)$. Hence, $U = \sigma \cdot \Gamma(K/2)2^{K-1} / \pi^{K/2} $ is a probability measure on $\SSS_+^{K-1}$, and we can write
    \begin{equation*}
        \mu_K\pr{\cD_\cS} = \frac{  \pi^{K/2}}{K\Gamma(K/2)2^{K-1} }\EE_{S\sim U} \pr{ \max_{s\in \cS}\min_{k\in [K]}\pr{\frac{r_k-s_k}{v_k}}^{K}}
    \end{equation*}
    Noting that by $\Gamma(x+1)=x\Gamma(x)$ we see that $\frac{ \pi^{K/2}}{K \Gamma(K/2)2^{K-1} } = \frac{ \pi^{K/2}}{\Gamma(K/2+1)2^{K} }=c_K$, and since $\hypervolume_r(\cS)=\mu_K\pr{\cD_\cS} $, this finishes the proof. 
\end{proof}

\subsection{Proof of Proposition \ref{prop:InsufficiencyPluginRegularization}}
\label{subsec:proof-InsufficiencyPluginRegularization}

\InsufficiencyPluginRegularization*

\begin{proof}
    We prove this lower bound by showing that any estimator in the form of \cref{eq:penalized-scalarization} is equivalent to a penalized least-squares estimator in a Gaussian sequence model over the $\ell_2$-ball in $d$ dimensions. Using this reduction, we can use standard lower bounds on this sequence model, for example from \cite{Neykov2023Minimax}. Importantly, this is \emph{not} a minimax lower bound on the original problem and only applies to the directly regularized estimator $\drEstimatedMinimizerWeighted$ because we first prove the reduction.

    We begin the proof by stating the following auxiliary fact that is also visualized in \cref{fig:proof_illustration}, which we prove after concluding the main proof. Denote $B_2^d\subset \RR^d$ the set of vectors $v\in\RR^d$ with $\norm{v}_2\leq 1$.
    \begin{claim}
    \label{claim:auxiliary-fact}
        For all $v\in B_2^d$, $\weight_1,\weight_2>0$ with $\weight_1+\weight_2=1$, $\gamma>1$ and $d,n\in\NN$ with $n\geq d$, there exist matrices $\X_1,\X_2\in\fX(\gamma)\subset \RR^{n\times d}$ and $1$-sparse vectors $\beta_1,\beta_2\in\Gamma\subset\RR^d$ so that for $\covariance_\objectiveindex:= \frac{1}{n}\X^\top_\objectiveindex \X_\objectiveindex$, and $\minimizerWeighted=\argmin_\minimizer \sum_{i=k}^K \lambda_k \norm{\X_k(\vartheta-\beta_k)}_2^2$ (from \cref{eq:scalarization}) it holds
        \begin{enumerate}
            \item $\weight_1\covariance_1+\weight_2\covariance_2=\identity_d$,
            \item $\weight_1\covariance_1\beta_1 +\weight_2\covariance_2\beta_2 = v = \minimizerWeighted $.
        \end{enumerate}
    \end{claim}
    
    % We confirm the first claim, that $\weight_1\covariance_1+\weight_2\covariance_2 = \identity_d$:
    % \begin{align*}
    %     \weight_1\covariance_1+\weight_2\covariance_2 &= \weight_1\pr{\weight_2\frac{\gamma-1}{\gamma} vv^\top + \identity_d} + \weight_2\pr{-\weight_1\frac{\gamma-1}{\gamma}vv^\top + \identity_d} \\
    %     &=\weight_1\weight_2\frac{\gamma-1}{\gamma} vv^\top +\weight_1 \identity_d -\weight_1\weight_2\frac{\gamma-1}{\gamma} vv^\top+\identity_d \\
    %     &=(\weight_1+\weight_2)\identity_d. 
    % \end{align*}
    % Further, calculations show the second claim that $\weight_1\covariance_1\beta_1 +\weight_2\covariance_2\beta_2 = v$, as
    % \begin{align*}
    %     \weight_1\covariance_1\beta_1 &= \weight_1 \pr{\weight_2\frac{\gamma-1}{\gamma} vv^\top + \identity_d}\frac{1}{\weight_1v_1}\frac{\gamma}{\gamma-1}\e_1 = \frac{\weight_2}{v_1}vv^\top \e_1 +\frac{1}{v_1}\frac{\gamma}{\gamma-1}\e_1 =\weight_2 v+\frac{1}{v_1}\frac{\gamma}{\gamma-1}\e_1, \\
    %     \weight_2\covariance_2\beta_2 &= -\weight_2 \pr{-\weight_1\frac{\gamma-1}{\gamma} vv^\top + \identity_d}\frac{1}{\weight_2v_1}\frac{\gamma}{\gamma-1}\e_1 = \frac{\weight_1}{v_1}vv^\top \e_1 -\frac{1}{v_1}\frac{\gamma}{\gamma-1}\e_1 =\weight_1 v-\frac{1}{v_1}\frac{\gamma}{\gamma-1}\e_1.
    %\end{align*}
    % \fy{i cut all the algebra since nothing is happening but actual simple algebra}
    % and consequently also $(\weight_1\covariance_1+\weight_2\covariance_2)^{1/2}\minimizerWeighted=v$
    % which is exactly what was claimed. This concludes the proof of the auxiliary fact.
    
    We use this claim in the main part of the proof. 
    Fix some $\weight_1,\weight_2>0$ with $\weight_1+\weight_2 = 1$. For an arbitrary choice of $v\in B_2^d$,
    take the corresponding design matrices $\X_\objectiveindex\in\fX(\gamma)\subset \RR^{n\times d}$, where $n\geq d$, and $1$-sparse ground-truths $\beta_1,\beta_2\in \Gamma$ from \cref{claim:auxiliary-fact}. Denote $\covariance_\objectiveindex =\frac{1}{n}\X_\objectiveindex^\top\X_\objectiveindex$ and recall that the eigenvalues of $\covariance_\objectiveindex$ are bounded to lie in $[\gamma^{-1},\gamma]$.
    
    %on population level we can write \todo{clean this up using the adversarial matrices directly}
    %\begin{align*}
    %    \scalarizationObjectiveComposition(\vartheta,\distributionXxY)
    %    &=\weight_1\frac{1}{n}\norm{\X_1(\vartheta-\beta_1)}_2^2 +\weight_2\frac{1}{n}\norm{\X_2(\vartheta-\beta_2)}_2^2 \\
    %    &= \weight_1(\vartheta-\beta_1)^\top \covariance_1(\vartheta-\beta_1) +\weight_2(\vartheta-\beta_2)^\top\covariance_2(\vartheta-\beta_2) \\
    %    &= \vartheta^\top \pr{\weight_1\covariance_1+\weight_2\covariance_2}\vartheta-2\vartheta^\top \pr{\weight_1\covariance_1\beta_1+\weight_2\covariance_2\beta_2} +\weight_1\beta_1^\top\covariance_1\beta_1 + \weight_2\beta_2^\top\covariance_2\beta_2\\
    %    &= \vartheta^\top \vartheta-2\vartheta^\top \minimizerWeighted +\minimizerWeighted^\top\minimizerWeighted \underbrace{-\minimizerWeighted^\top\minimizerWeighted +\frac{1}{2}\beta_1^\top\covariance_1\beta_1 + \frac{1}{2}\beta_2^\top\covariance_2\beta_2}_{=:C}\\
    %    &=\norm{\vartheta-\minimizerWeighted}_2^2 +C
    %\end{align*}
    %where $C$ is a constant independent of $\vartheta$. This also confirms that indeed, $\minimizerWeighted$ minimizes the scalarized objective.

    Recalling that we use linear scalarization, and using the auxiliary fact that $\weight_1\covariance_1+\weight_2\covariance_2=\identity_d$ and $\weight_1\covariance_1\beta_1 +\weight_2\covariance_2\beta_2 = v = \minimizerWeighted$, the scalarized empirical objective of \cref{eq:penalized-scalarization} reduces to
    \begin{align*}
        \scalarizationObjectiveComposition(\vartheta,\empdistributionXxYTuple)
        &= \weight_1\frac{1}{n}\norm{\X_1\vartheta-y^1}_2^2+\weight_2\frac{1}{n}\norm{\X_2\vartheta-y^2}_2^2 \\
        &=\vartheta^\top \vartheta - 2\vartheta^\top\pr{\frac{\weight_1}{n}\X_1^\top y^1+\frac{\weight_2}{n}\X_2^\top y^2} + \frac{\weight_1}{n}\norm{y^1}_2^2+ \frac{\weight_2}{n}\norm{y^2}_2^2  \\
        &=\vartheta^\top \vartheta - 2\vartheta^\top\pr{ \minimizerWeighted +\frac{\weight_1}{n}\X_1^\top\xi^1+\frac{\weight_2}{n}\X_2^\top\xi^2} + C(\xi^1,\xi^2) 
    \end{align*}
    where $C(\xi^1,\xi^2)$ is a (random) constant independent of $\vartheta$. Since the following random vector is Gaussian
    \begin{equation*}
        \xi := \frac{\weight_1}{n}\X_1^\top\xi^1+\frac{\weight_2}{n}\X_2^\top\xi^2 \sim\cN\pr{0,\M} \quad \text{with} \quad \M = \frac{\sigma^2}{n} \pr{\weight_1^2\covariance_1+\weight_2^2\covariance_2},
    \end{equation*}
    we can define $ y = \minimizerWeighted + \xi$ and it follows that we can write
    \begin{equation*}
        \scalarizationObjectiveComposition(\vartheta,\empdistributionXxYTuple)=\vartheta^\top \vartheta - 2\vartheta^\top  y +  C(\xi^1,\xi^2) =\norm{\vartheta -y}_2^2+ C'(\xi^1,\xi^2).
    \end{equation*}

    As the auxiliary fact shows, $v$ could lie anywhere in $B_2^d$, and hence, from the point of view of the empirical scalarized objective, the problem has fully reduced to a Gaussian sequence model.
    In particular, the empirical objective from \cref{eq:penalized-scalarization} is that of a penalized least-squares estimator in a Gaussian sequence model in $d$ dimensions with ground-truth $v=\minimizerWeighted$, using the noisy observations
    \begin{equation}
    \label{eq:Gaussian-sequence-model-1}
        y=v+\xi  \quad \text{with} \quad  v\in B_2^d,\quad \xi \sim \cN(0,\M).
    \end{equation}
    
    If $\weight_1=\weight_2=1/2$, then $ \xi$ is isotropic (since $\M=\sigma^2\identity_d/(2n)$) and we could directly apply known lower bounds for the minimax rate in this setting \citep{Neykov2023Minimax}, but in the general form we use the following bound on the eigenvalues of $\M$
    \begin{equation*}
        \widetilde\sigma^2:= \frac{\sigma^2}{2n\gamma} \preceq \M %\preceq \frac{\sigma^2\gamma}{n}.
    \end{equation*}
    This can easily be seen by using $\weight_1^2+\weight_2^2 \geq 1/2$ and hence
    \begin{equation*}
        \eigenvalueMin\pr{\M} = \frac{\sigma^2}{n}\eigenvalueMin\pr{\weight_1^2\covariance_1+\weight_2^2\covariance_2} \geq \frac{\sigma^2}{n}\frac{\weight_1^2+\weight_2^2}{\gamma}\geq \frac{\sigma^2}{2n\gamma}
        %\eigenvalueMax\pr{\M} &= \frac{\sigma^2}{n}\eigenvalueMax\pr{\weight_1^2\covariance_1+\weight_2^2\covariance_2} \leq \frac{\sigma^2\gamma}{n}(\weight_1^2+\weight_2^2)\leq \frac{\sigma^2\gamma}{n}.
    \end{equation*}
    
    We show that one may bound the minimax rate in the model \eqref{eq:Gaussian-sequence-model-1} from below in terms of the minimax rate of the Gaussian sequence model
    \begin{equation}
    \label{eq:Gaussian-sequence-model-2}
        \widetilde y=v+\widetilde\xi  \quad \text{with} \quad v\in B_2^d,\quad \widetilde\xi \sim \cN(0,\widetilde \sigma^2\identity_d)
    \end{equation}
    where $\widetilde\sigma^2= \sigma^2 /(2n\gamma)$.

    To see this, first, note that we may write $\xi = Z + W$, where $Z\sim \cN(0,\eigenvalueMin(\M)\I_d)$ and $W\sim \cN(0, \M-\eigenvalueMin(\M) \I_d)$
    are two independent Gaussian vectors. Then, for any estimator $\vhat(v+Z+W)$ in the Gaussian sequence model \eqref{eq:Gaussian-sequence-model-1}, we can define a corresponding estimator
    $\vprime : \R^d\to \R^d$ given by
    \begin{equation*}
        \vprime(v+Z) = \EE_W\br{\widehat v(v+Z+W)}
    \end{equation*}
    that is an estimator in the the Gaussian sequence model with only noise $Z$.
    %\fy{the reason i did it this way is because in the vtilde case there is no random real joint distribution over W but only Z? so the W is kind of "invented" for the sake of computing?}
    By a standard bias-variance decomposition, 
    the following
    lower bound holds for all $v\in B_2^d, z\in\RR^d$:
    \begin{equation*}
    %\label{eq:comparison}
        \EE \br{\norm{\widehat v(v+Z+W)-v}_2^2 \mid Z=z} \geq  \norm{\vprime(v+z)-v}_2^2.
    \end{equation*}
    The claim then directly follows by taking expectation with respect to $Z$, taking the supremum over $v$ and infimum over estimators, and using $\eigenvalueMin(\M)\geq \widetilde \sigma^2$.
    Therefore, the error of any estimator $\drEstimatedMinimizerWeighted$ of the form \eqref{eq:penalized-scalarization} using any penalty $\rho$ is lower bounded by the minimax rate in the Gaussian sequence model of \eqref{eq:Gaussian-sequence-model-2}. 
    
    If $\widetilde\sigma^2\leq 1/(d+1)$, the minimax rate in this setting is lower bounded by 
    $\widetilde\sigma^2 d$ (up to constant factors), see for example \cite[Corollary 3.3]{Neykov2023Minimax}. Hence, from plugging in the definition of $\widetilde\sigma^2 = \sigma^2/(2n\gamma)$, we see that under our assumption that $\sigma^2 \leq 2n\gamma/(d+1)$, the minimax rate is lower bounded by $\sigma^2d/(n\gamma)$ (up to constant factors). Putting everything together, we have for any estimator $\drEstimatedMinimizerWeighted$ in the form of \eqref{eq:penalized-scalarization}
    \begin{align*}
        \sup_{\substack{\beta_1,\beta_2\in \Gamma \\\X_1,\X_2\in \fX(\gamma)}} \EE\br{\norm{\drEstimatedMinimizerWeighted( y)-\minimizerWeighted}_2^2}&\geq \sup_{v\in B_2^d} \EE\br{\norm{\drEstimatedMinimizerWeighted( y)-v}_2^2}  &&(\text{\cref{claim:auxiliary-fact}})\\
        &\geq \inf_{\widehat v}\sup_{v \in B_2^d } \EE\br{\norm{\widehat v(\widetilde y)-v}_2^2} &&(\text{lower bound from \eqref{eq:Gaussian-sequence-model-2}}) \\
        &\gtrsim \widetilde{\sigma}^2d= \frac{\sigma^2 d}{n\gamma}. && \text{\cite[Corollary 3.3]{Neykov2023Minimax}}
    \end{align*}
    This concludes the lower bound.
\end{proof}

\begin{proof}[Proof of \cref{claim:auxiliary-fact}]
    To see this, first note that the optimization problem is minimized by
    \begin{equation}
    \label{eq:varthetacalc}
        \minimizerWeighted=(\weight_1\covariance_1+\weight_2\covariance_2)^{-1}(\weight_1\covariance_1\beta_1+\weight_2\covariance_2\beta_2)
    \end{equation}
    because the problem is strongly convex \ref{item:weighted-sum-strongly-convex} and the first-order stationarity condition 
    \begin{equation*}
        \nabla_\vartheta \pr{\frac{\weight_1}{n}\norm{\X_1(\vartheta-\beta_1)}_2^2+\frac{\weight_2}{n}\norm{\X_2(\vartheta-\beta_2)}_2^2} = 2\weight_1\covariance_1(\vartheta-\beta_1)+2\weight_2\covariance_2(\vartheta-\beta_2)=0
    \end{equation*}
    is satisfied by $\minimizerWeighted$, cf. \ref{item:mixture-quadratics}.
    Moreover, because for every symmetric positive definite $\covariance$ with eigenvalues in $[\gamma^{-1},\gamma]$, there exists a $\X\in \fX(\gamma)$ so that $\covariance =\frac{1}{n}\X^\top\X$, we may simply construct $\covariance_1,\covariance_2$ directly. This is because $n\geq d$ and we may choose 
    \begin{equation*}
        \X = \sqrt{n}
        \begin{pmatrix}
            \covariance^{1/2} \\
            \mathbf{0}_{(d-n)\times d}
        \end{pmatrix}
        \implies \frac{1}{n}\X^\top \X = \covariance.
    \end{equation*}

    The proof of the auxiliary fact then follows by explicit construction:
    Let $\e_i$ denotes the $i$-th standard basis vector.
    Fix any $v\in\RR^d$ with $\norm{v}_2\leq 1$ and, without loss of generality, assume that $v_1\neq 0$ (where $v=(v_1,\dots,v_d)^\top$). Otherwise change index $1$ with any other index (the case $v=0$ is trivially solved by $\beta_1=\e_1,\beta_2=-(\weight_1/\weight_2)\e_1$ and $\covariance_1=\covariance_2=\identity_d$).
    We choose the vectors 
    \begin{equation*}
        \beta_1 = \frac{1}{v_1 \weight_1} \frac{\gamma}{\gamma-1} \e_1 \quad \text{and} \quad \beta_2 = -\frac{1}{v_1 \weight_2} \frac{\gamma}{\gamma-1} \e_1,
    \end{equation*}
    and we further choose the matrices
    \begin{equation*}
        \covariance_1 = \weight_2\frac{\gamma-1}{\gamma} vv^\top + \identity_d \quad \text{and} \quad \covariance_2 = -\weight_1\frac{\gamma-1}{\gamma}vv^\top + \identity_d.
    \end{equation*}
    
    The covariance matrices and vectors satisfy all requirements, as $\beta_1$ and $\beta_2$ are clearly $1$-sparse, and the matrices $\covariance_1,\covariance_2$ are symmetric, and their eigenvalues are 
    \begin{align*}
        \eigenvalueMin(\covariance_1) &= 1>\gamma^{-1}, &&\eigenvalueMax(\covariance_1) = 1+\weight_2\norm{v}_2^2\frac{\gamma-1}{\gamma}\leq \frac{2\gamma-1}{\gamma}\leq \gamma, \\
        \eigenvalueMin(\covariance_2) &= 1-\weight_1\norm{v}_2^2\frac{\gamma-1}{\gamma}\geq \frac{\gamma-\gamma+1}{\gamma} = \gamma^{-1}, &&\eigenvalueMax(\covariance_2) = 1 \leq \gamma,
    \end{align*}
    also implying they are positive definite. The claims 
    then follow directly for these choices of $\Sigma_k, \beta_k$. Specifically, we have by \cref{eq:varthetacalc} that
    \begin{equation*}
        \minimizerWeighted=(\weight_1\covariance_1+\weight_2\covariance_2)^{-1}(\weight_1\covariance_1\beta_1+\weight_2\covariance_2\beta_2)= \identity_d^{-1}  v=v
    \end{equation*}
    and the proof is complete.
\end{proof}

\subsection{Proof of Proposition \ref{prop:TwostageRateFixedDesign}}
\label{subsec:proof-TwostageRateFixedDesign}

\TwostageRateFixedDesign*

\begin{proof}[Proof of \cref{prop:TwostageRateFixedDesign}]
    The bound follows by combining known bounds on the estimation error of the LASSO, e.g., from \cite{Bickel2009simultaneous,Wainwright2019}, with a closed-form solution of the multi-objective optimization step. 
    
    Specifically, note that similar to our previous derivations \ref{item:mixture-quadratics}, by first-order optimality, our estimator $\tsEstimatedMinimizerWeighted$ takes the form
    \begin{equation*}
        \tsEstimatedMinimizerWeighted = (\weight_1\covariance_1+\weight_2\covariance_2)^{-1}\pr{\weight_1\covariance_1\betahat_1+\weight_2\covariance_2\betahat_2}
    \end{equation*}
    where $\betahat_1,\betahat_2$ are the LASSO estimates with $\ell_1$-penalty.
    Since the eigenvalues of $\covariance_\objectiveindex$ are larger than $\gamma^{-1}$, the eigenvalues of $(\weight_1\covariance_1+\weight_2\covariance_2)^{-1}$ are smaller than $\gamma$, and so
    \begin{align}
        \norm{\tsEstimatedMinimizerWeighted-\minimizerWeighted}_2^2 &= \norm{(\weight_1\covariance_1+\weight_2\covariance_2)^{-1}\pr{\weight_1\covariance_1\betahat_1+\weight_2\covariance_2\betahat_2}-(\weight_1\covariance_1+\weight_2\covariance_2)^{-1}(\weight_1\covariance_1\beta_1+\weight_2\covariance_2\beta_2)}_2^2 \nonumber\\
        &\leq \norm{(\weight_1\covariance_1+\weight_2\covariance_2)^{-1}}_2^2\norm{\weight_1\covariance_1(\beta_1-\betahat_1)+\weight_2\covariance_2(\beta_2-\betahat_2)}_2^2 \nonumber\\
        &\leq \gamma\pr{\weight_1\norm{\covariance_1}_2\norm{\beta_1-\betahat_1}_2+\weight_2\norm{\covariance_2}_2\norm{\beta_2-\betahat_2}_2}^2\nonumber\\
        &\leq \gamma^3\pr{\norm{\beta_1-\betahat_1}_2+\norm{\beta_1-\betahat_1}_2}^2 \nonumber \\
        &\leq 2\gamma^3\pr{\norm{\beta_1-\betahat_1}_2^2+\norm{\beta_2-\betahat_2}_2^2}. \label{eq:proof-direct-bound-lasso} 
    \end{align}
    To bound the error on the LASSO estimates, we can invoke standard results, such as \citet[Theorem 7.13(a) or Theorem 7.19]{Wainwright2019}. To do so, we verify that the matrix $\X_k$ satisfies the restricted eigenvalue condition with constant $\gamma^{-1}$, and is column-normalized with constant $\gamma$: For any ($1$-sparse) vector $\beta_\objectiveindex \in \RR^d$, we have
    \begin{equation*}
        \frac{1}{n}\norm{\X_\objectiveindex\beta_\objectiveindex}_2^2 =\beta_\objectiveindex^\top \covariance_\objectiveindex \beta_\objectiveindex \geq \gamma^{-1}\norm{\beta_\objectiveindex}_2^2 
    \end{equation*}
    because we have assumed that $\X_\objectiveindex\in\fX(\gamma)$,
    and for the standard basis vectors $\e_i$ we have
    \begin{equation*}
        \frac{1}{n}\norm{\X_\objectiveindex \e_i}_2^2 =\e_i^\top \covariance_\objectiveindex \e_i \leq \gamma \norm{\e_i}_2^2 = \gamma.
    \end{equation*}
    Furthermore, choosing the penalty strength
    \begin{equation*}
        \alpha = 6\gamma \sigma \sqrt{\frac{2\log d}{n}}
    \end{equation*}
    so that the event $\mathset{\frac{1}{n}\norm{\X_k \xi^k}_\infty \leq \alpha/2}$ has probability at least at least $1-2d^{-4}$ \cite[Example 7.14 with $t=2\sqrt{2\log (d)/n}$]{Wainwright2019}. 
    Hence, by \citet[Theorem 7.13(a)]{Wainwright2019}, we get that
    \begin{equation*}
        \norm{\beta_\objectiveindex-\betahat_\objectiveindex}_2 \leq 3\gamma \alpha = 18\gamma^2\sigma \sqrt{\frac{2\log d}{n}}
    \end{equation*}
    with probability at least $1-2d^{-4}$. 
    %Therefore, by standard tail integration, we have that
    %\begin{align*}
    %    \EE\br{\norm{\minimizer_\objectiveindex-\estimatedMinimizer_\objectiveindex}_2^2} = \int_{0}^\infty \PP\pr{\norm{\minimizer_\objectiveindex-\estimatedMinimizer_\objectiveindex}_2^2\geq t}dt \lesssim \frac{\sigma^2 \log d}{n}
    %\end{align*}
    %for $\objectiveindex=1,2$. 
    Putting this together with \eqref{eq:proof-direct-bound-lasso}, we get our upper bound
    \begin{equation*}
        \norm{\tsEstimatedMinimizerWeighted-\minimizerWeighted}_2^2 \leq 2\gamma^3 \pr{36\gamma^2\sigma\sqrt{\frac{2\log d}{n}}}^2 = 2592 \frac{\gamma^7\sigma^2\log d}{n}
    \end{equation*}
    with probability at least $1-4d^{-4}$ where we take the union bound for both estimators $k\in\mathset{1,2}$.
    Note that the upper bound can also be shown indirectly using \cref{thm:MinimaxRateStrongConvexity} with known $\unlabeledTuple$.
    This concludes the proof of \cref{prop:InsufficiencyPluginRegularization}.
\end{proof}

\subsection{Proof of Theorem \ref{thm:MinimaxRateStrongConvexity}}
\label{subsec:proof-MinimaxRateStrongConvexity}

%\StrongConvexitySmoothness*

%\LocallyLipschitz*

%\Injectivity*

\MinimaxRateStrongConvexity*

\begin{proof}
    The proof of the theorem follows from the regularity assumptions and Pareto stationarity.
    It is based on arguments akin to using the Implicit Function Theorem for the stability of minimizers, similar to \cite[Theorem 3.1]{Shvartsman2012stability} or \cite[Proposition 4.32]{Bonnans2000perturbation}. 
    
    By \cref{ass:StrongConvexitySmoothness} and \ref{item:weighted-sum-strongly-convex}, we know that $\vartheta\mapsto \scalarizationObjectiveComposition(\vartheta,\parameterTuple)$ is $\scalarization(\strongConvexityParamTuple)$-strongly convex for all $\parameterTuple\in \parameterSpaceTupleLarge$, and because there is one $j$ such that $\weight_j\strongConvexityParam_j>0$, we have $\scalarization(\strongConvexityParamTuple)>0$. Hence, the map $\parameterTuple\mapsto \minimizerWeighted(\parameterTuple) = \argmin_{\vartheta}\scalarizationObjectiveComposition (\vartheta,\parameterTuple)$ is well-defined on $\parameterSpaceTupleLarge$. We now show that this function is locally Lipschitz. 

    Define the function $F_{{\weightTuple}}:\RR^m\times \RR^{K p}\to \RR^m$
    \begin{equation*}
        F_{\weightTuple}(\vartheta,\parameterTuple) = \nabla_\vartheta \scalarizationObjectiveComposition(\vartheta,\parameterTuple) = \sum_{\objectiveindex=1}^K \weight_\objectiveindex \nabla_\vartheta \objectiveindexed(\vartheta,\parameter_\objectiveindex).
    \end{equation*}
    By first-order Pareto stationarity \citep{Roy2023optimizationparetosetstheory}, we know that $F_{\weightTuple}(\minimizerWeighted(\parameterTuple),\parameterTuple)=0$ for all $\parameterTuple\in\parameterSpaceTupleLarge$.
    Let $\parameterTuple,\parameterTuple'\in \parameterSpaceTupleLarge \subset \RR^{K p}$, so that $F_{\weightTuple}(\minimizerWeighted(\parameterTuple),\parameterTuple)=0$ and $F_{\weightTuple}(\minimizerWeighted(\parameterTuple'),\parameterTuple')=0$
    %. Therefore, $F_{\weightTuple}(\minimizerWeighted(\parameterTuple),\parameterTuple) = F_{\weightTuple}(\minimizerWeighted(\parameterTuple'),\parameterTuple')$ 
    and hence
    \begin{equation*}
        \norm{F_{\weightTuple}(\minimizerWeighted(\parameterTuple),\parameterTuple)-F_{\weightTuple}(\minimizerWeighted(\parameterTuple'),\parameterTuple)}_2 = \norm{F_{\weightTuple}(\minimizerWeighted(\parameterTuple'),\parameterTuple') - F_{\weightTuple}(\minimizerWeighted(\parameterTuple'),\parameterTuple)}_2.
    \end{equation*}
    Using $\scalarization(\strongConvexityParamTuple)$-strong convexity of $\vartheta\mapsto \scalarizationObjectiveComposition(\vartheta,\parameterTuple)$ and \ref{item:implication-strong-convexity}, we can now lower-bound the left-hand side as
    \begin{align*}
        \norm{F_{\weightTuple}(\minimizerWeighted(\parameterTuple),\parameterTuple)-F_{\weightTuple}(\minimizerWeighted(\parameterTuple'),\parameterTuple)}_2 \geq  \scalarization(\strongConvexityParamTuple) \norm{\minimizerWeighted(\parameterTuple)-\minimizerWeighted(\parameterTuple')}_2
    \end{align*}
    and using local Lipschitz-continuity of $\parameter_\objectiveindex\mapsto \nabla_\vartheta\objectiveindexed(\vartheta,\parameter_\objectiveindex)$ (\cref{ass:LocallyLipschitz}), we can upper bound the right-hand side as
    \begin{align*}
        \norm{F_{\weightTuple}(\minimizerWeighted(\parameterTuple'),\parameterTuple') - F_{\weightTuple}(\minimizerWeighted(\parameterTuple'),\parameterTuple)}_2 &= \norm{\sum_{\objectiveindex=1}^K \weight_\objectiveindex \pr{\nabla_\vartheta \objectiveindexed(\minimizerWeighted(\parameterTuple'),\parameter_k')- \nabla_\vartheta \objectiveindexed(\minimizerWeighted(\parameterTuple'),\parameter_k)}}_2 \\
        &\leq \sum_{\objectiveindex=1}^K \weight_\objectiveindex \zeta_k(\minimizerWeighted(\parameterTuple')) \norm{\parameter_k'-\parameter_k}_2
    \end{align*}
    where $\zeta_k(\minimizerWeighted(\parameterTuple'))$ denotes the local Lipschitz constant. Combining the two and letting $\zeta(\minimizerWeighted(\parameterTuple'))=\max_k\zeta_k(\minimizerWeighted(\parameterTuple'))$, we get that
    \begin{equation*}
        \norm{\minimizerWeighted(\parameterTuple)-\minimizerWeighted(\parameterTuple')}_2 \leq \frac{\zeta(\minimizerWeighted(\parameterTuple'))}{\scalarization(\strongConvexityParamTuple)} \sum_{\objectiveindex=1}^K \weight_\objectiveindex \norm{\parameter_k'-\parameter_k}_2.
    \end{equation*}
    Since we assumed that the true parameters satisfy $\parameterTuple\in \parameterSpaceTuple\subset\parameterSpaceTupleLarge$ 
    %\fy{hm this is now a specific one? notational overload} 
    and the estimators satisfy $\estimatedParameterTuple\in\parameterSpaceTupleLarge$ we can plug them in, that is, $\tsEstimatedMinimizerWeighted = \minimizerWeighted(\estimatedParameterTuple)$ and $\minimizerWeighted = \minimizerWeighted(\parameterTuple)$, we get
    \begin{equation*}
        \norm{\tsEstimatedMinimizerWeighted-\minimizerWeighted}_2\leq \frac{\zeta(\minimizerWeighted)}{\scalarization(\strongConvexityParamTuple)} \sum_{\objectiveindex=1}^K \weight_\objectiveindex \norm{\estimatedParameter_k-\parameter_k}_2
    \end{equation*}
    which is the desired upper bound.
\end{proof}

\subsection{Proof of Proposition \ref{prop:LipschitzBound}}
\label{subsec:proof-LipschitzBound}

\LipschitzBound*

\begin{proof}
We use the standard uniform learning bound
    \begin{align*}
        \excessscalarization(\tsEstimatedMinimizerWeighted) &=\scalarizationObjectiveComposition(\tsEstimatedMinimizerWeighted,\parameterTuple)-\scalarizationObjectiveComposition(\minimizerWeighted,\parameterTuple)\\
        &= \scalarizationObjectiveComposition(\tsEstimatedMinimizerWeighted,\parameterTuple) - \scalarizationObjectiveComposition(\tsEstimatedMinimizerWeighted,\estimatedParameterTuple) +\underbrace{\scalarizationObjectiveComposition(\tsEstimatedMinimizerWeighted,\estimatedParameterTuple) -\scalarizationObjectiveComposition(\minimizerWeighted,\estimatedParameterTuple)}_{\leq 0}+\scalarizationObjectiveComposition(\minimizerWeighted,\estimatedParameterTuple)-\scalarizationObjectiveComposition(\minimizerWeighted,\parameterTuple)\\
        &\leq 2\sup_{\vartheta\in \RR^d} \abs{\scalarizationObjectiveComposition(\vartheta,\parameterTuple)-\scalarizationObjectiveComposition(\vartheta,\estimatedParameterTuple)}\\
        &=2\sup_{\vartheta\in \RR^d} \abs{\norm{{\weightTuple}\odot\objectivevector(\vartheta,\parameterTuple)}-\norm{{\weightTuple}\odot\objectivevector(\vartheta,\estimatedParameterTuple)}}.
    \end{align*}
    Using the reverse triangle inequality, we can further bound this as
    \begin{equation*}
        \excessscalarization(\tsEstimatedMinimizerWeighted)\leq 2\sup_{\vartheta\in \RR^d} \norm{{\weightTuple}\odot\objectivevector(\vartheta,\parameterTuple)-{\weightTuple}\odot\objectivevector(\vartheta,\estimatedParameterTuple)} =2\sup_{\vartheta\in \RR^d} \norm{{\weightTuple}\odot\pr{\objectivevector(\vartheta,\parameterTuple)-\objectivevector(\vartheta,\estimatedParameterTuple)}} \leq 2\norm{{\weightTuple}\odot\pr{\Phi(\parameter_\objectiveindex,\estimatedParameter_\objectiveindex)}_{\objectiveindex=1}^K}
    \end{equation*}
    where the last inequality holds by assumption, and hence, we have the result.
\end{proof}

\subsection{Proof of Theorem \ref{thm:GeneralLowerBound}}
\label{subsec:proof-GeneralLowerBound}

\GeneralLowerBound*

\begin{proof}[Proof of \cref{thm:GeneralLowerBound}]
    The proof of \cref{thm:GeneralLowerBound} makes use of the identifiability assumption from \cref{ass:Injectivity}. We prove the lower bound by contradiction: Assuming that we could estimate $\minimizerWeighted$ better than the stated bound (in a minimax sense), we show that the identifiability assumption implies that we could also estimate one of the parameters $\parameter_k$ faster than the minimax rate $\delta_k$. As this is impossible by definition, the lower bound follows.
    
    Note that by first-order optimality, we have that
    \begin{equation*}
        \nabla_\vartheta \scalarizationObjectiveComposition(\minimizerWeighted,\parameterTuple )=0,
    \end{equation*}
    which is known as Pareto stationarity \citep{Roy2023optimizationparetosetstheory}.
    Expanding and rearranging yields for every $k$
    \begin{equation*}
        \nabla_\vartheta \Loss_k(\minimizerWeighted,\parameter_k)=-\frac{1}{\weight_k}\sum_{i\neq k}\weight_i\nabla_\vartheta \Loss_i(\minimizerWeighted,\parameter_i).
    \end{equation*}
    Since by \cref{ass:Injectivity}, the map $g_k(\cdot;\minimizerWeighted):\parameter_k\mapsto\nabla_\vartheta \objectiveindexed(\minimizerWeighted,\parameter_k)$ is injective for every $\minimizerWeighted$, we can write
    \begin{equation*}
        \parameter_k=g_k^{-1}\pr{ -\frac{1}{\weight_k}\sum_{i\neq k}\weight_i\nabla_\vartheta \Loss_i(\minimizerWeighted,\parameter_i);\minimizerWeighted}.
    \end{equation*}
    Now, given any estimators $\estimatedMinimizerWeighted$ of $\minimizerWeighted$ and $\estimatedParameterTuple$ of $\parameterTuple$ %. Based on these estimators, 
    we define the new estimator
    \begin{equation*}
        \estimatedParameter_k^{\text{new}}=g_k^{-1}\pr{-\frac{1}{\weight_k}\sum_{i\neq k}\weight_i\nabla_\vartheta \Loss_i(\estimatedMinimizerWeighted,\estimatedParameter_i);\estimatedMinimizerWeighted}.
    \end{equation*}
    We can bound the error of this new estimator using the Lipschitz-continuity of the left-inverse (\cref{ass:Injectivity}) and get
    \begin{align*}
        \norm{\estimatedParameter_k^{\text{new}}-\parameter_k} &\leq \eta_{k}'\pr{\norm{\estimatedMinimizerWeighted-\minimizerWeighted}_2+\norm{ \frac{1}{\weight_k}\sum_{i\neq k}\weight_i\pr{\nabla_\vartheta \Loss_i(\estimatedMinimizerWeighted,\estimatedParameter_i)-\nabla_\vartheta \Loss_i(\minimizerWeighted,\parameter_i)}}_2} \\
        &\leq \eta_{k}'\pr{\norm{\estimatedMinimizerWeighted-\minimizerWeighted}_2+\frac{1}{\weight_k}\sum_{i\neq k}\weight_i\norm{\nabla_\vartheta \Loss_i(\estimatedMinimizerWeighted,\estimatedParameter_i)-\nabla_\vartheta \Loss_i(\minimizerWeighted,\parameter_i)}_2}.
    \end{align*}
    Further, by \cref{ass:Injectivity} we have for all estimators $\estimatedParameter_i$
    \begin{equation*}
        \norm{\nabla_\vartheta \Loss_i(\estimatedMinimizerWeighted,\estimatedParameter_i)-\nabla_\vartheta \Loss_i(\minimizerWeighted,\parameter_i)}_2
        \leq \eta_{\objectiveindex}\pr{\norm{\estimatedParameter_i-\parameter_i}+\norm{\estimatedMinimizerWeighted-\minimizerWeighted}_2}.
    \end{equation*}
    Consequently, we get
    \begin{align*}
        \norm{\estimatedParameter_k^{\text{new}}-\parameter_k} &\leq 
        \eta_{k}'\pr{\norm{\estimatedMinimizerWeighted-\minimizerWeighted}_2+\frac{1}{\weight_k}\sum_{i\neq k}\weight_i\norm{\nabla_\vartheta \Loss_i(\estimatedMinimizerWeighted,\estimatedParameter_i)-\nabla_\vartheta \Loss_i(\minimizerWeighted,\parameter_i)}_2}\\
        &\leq \eta_{k}'\pr{\norm{\estimatedMinimizerWeighted-\minimizerWeighted}_2+\frac{1}{\weight_k}\sum_{i\neq k}\weight_i\eta_{i}\pr{\norm{\estimatedParameter_i-\parameter_i}+\norm{\estimatedMinimizerWeighted-\minimizerWeighted}_2}} \\
        &\leq (1+\scalarization(\etabold))\frac{\eta_k'}{\weight_k}\norm{\estimatedMinimizerWeighted-\minimizerWeighted}_2+\frac{\eta_k'}{\weight_k}\sum_{i\neq k}\weight_i\eta_{i}\norm{\estimatedParameter_i-\parameter_i}
    \end{align*}
    which we can rearrange to get
    \begin{align*}
        \norm{\estimatedMinimizerWeighted-\minimizerWeighted}_2 \geq \pr{1+\scalarization(\etabold)}^{-1}\pr{\frac{\weight_k}{\eta_{k}'}\norm{\estimatedParameter_k^{\text{new}}-\parameter_k}-\sum_{i\neq k}\eta_{i}\weight_i\norm{\estimatedParameter_i-\parameter_i}}.
    \end{align*}
    Notice that this lower bound holds for \emph{all} choices of estimators $\estimatedParameterTuple$ and $\estimatedMinimizerWeighted$. Choosing $\estimatedParameter_\objectiveindex$ to be minimax optimal for all $k\in[K]$, we then obtain for all $k\in[K]$ that
    %that is, we choose
    \begin{equation*}
        \sup_{\distributionXxYTuple\in\distributionSetTuple}\EE\norm{\estimatedParameter_k^{\text{new}}-\parameter_k}\geq \sup_{\distributionXxYTuple\in\distributionSetTuple} \EE\br{\norm{\estimatedParameter_\objectiveindex-\parameter_\objectiveindex}}=\delta_\objectiveindex.
    \end{equation*}
    %and also notice that by definition, 
    % we then obtain
    % \begin{equation*}
    %     \sup_{\distributionXxYTuple\in\distributionSetTuple}\EE\norm{\estimatedParameter_k^{\text{new}}-\parameter_k}_2 \geq \delta_k.
    % \end{equation*}
    Hence, using the notation $\estimatedParameter^{\text{new}}_k = \estimatedParameter^{\text{new}}_k(\estimatedMinimizerWeighted)$ to highlight the dependence, this yields the lower bound
    \begin{align*}
        \inf_{\estimatedMinimizerWeighted} \sup_{\distributionXxYTuple\in\distributionSetTuple} \EE\br{\norm{\estimatedMinimizerWeighted-\minimizerWeighted}_2} &\geq \inf_{\estimatedMinimizerWeighted} \sup_{\distributionXxYTuple\in\distributionSetTuple} \EE\br{\pr{1+\scalarization(\etabold)}^{-1}\pr{\frac{\weight_k}{\eta_{k}'}\norm{\estimatedParameter_k^{\text{new}}(\estimatedMinimizerWeighted)-\parameter_k}-\sum_{i\neq k}\eta_{i}\weight_i\norm{\estimatedParameter_i-\parameter_i}}} \\
        &\geq \pr{1+\scalarization(\etabold)}^{-1}\pr{\frac{\weight_k}{\eta_{k}'} \inf_{\estimatedMinimizerWeighted} \sup_{\distributionXxYTuple\in\distributionSetTuple} \EE\br{\norm{\estimatedParameter_k^{\text{new}}(\estimatedMinimizerWeighted)-\parameter_k}}-\sum_{i\neq k}\eta_{i}\weight_i \sup_{\distributionXxYTuple\in\distributionSetTuple}\EE\br{\norm{\estimatedParameter_i-\parameter_i}}} \\
        &\geq \pr{1+\scalarization(\etabold)}^{-1}\pr{\frac{\weight_k}{\eta_{k}'}\delta_k-\sum_{i\neq k}\eta_{i}\weight_i \delta_i}.
    \end{align*}
    And of course, $0$ is a trivial lower bound. 
    %Therefore, by Jensen's inequality, we get that
    %\begin{equation*}
    %    \inf_{\estimatedParameter} \sup_{\distributionXxYTuple\in\distributionSetTuple} \EE\br{\norm{\estimatedParameter-\minimizerWeighted}_2^2} \geq \inf_{\estimatedParameter} \sup_{\distributionXxYTuple\in\distributionSetTuple} \pr{\EE\br{\norm{\estimatedParameter-\minimizerWeighted}_2}}^2 \geq \pr{1+\scalarization(\nu)}^{-2}\abs{\frac{\weight_k}{C_{k}'}\delta_k-\sum_{i\neq k}C_{i}\weight_\objectiveindex \delta_\objectiveindex}_+^2
    %\end{equation*}
    Since the argument was valid for any $k$, we can take the maximum over $k$. This concludes the proof of \cref{thm:MinimaxRateStrongConvexity}.
\end{proof}

\subsection{Proof of Corollary \ref{cor:MultiDistributionOurs}}
\label{subsec:proof-MultiDistributionOurs}

%\MultipleLinearRegression*

\begin{assumption}[Scalings]
\label{ass:technical-assumptions-multi-distribution}
\begin{enumerate}
\setlength\itemsep{0em}
    \item We let $b\lesssim 1 \lesssim \sigma$.
    \item For all $k\in[K]$, $n_\objectiveindex+N_\objectiveindex\gtrsim d(B^4/b^4)$.
    \item For all $k\in[K]$, $n_k\gtrsim (B^2/b^2)\sigma^2 s\log d$.
    \item For the lower bound we assume that for some large enough universal constant $C>0$ and some $k\in[K]$, it holds that $\weight_k \frac{b^3}{B^2} n_k^{-1/2} \geq C \sum_{i\neq k} \weight_i n_i^{-1/2}$. This assumption corresponds to \eqref{eq:sufficient-condition-minimax}.
\end{enumerate}
\end{assumption}

\MultiDistributionOurs*

\begin{proof}
    %We note that this corollary could be proved directly. However, we use \cref{thm:MinimaxRateStrongConvexity} to demonstrate its applicability.\\
    To prove the upper bound of \cref{cor:MultiDistributionOurs}, we apply \cref{thm:MinimaxRateStrongConvexity}, to prove the lower bound, we apply \cref{thm:GeneralLowerBound}.
    
    By definition of $\distributionSetTuple$ in \cref{ex:multiple-linear-regression}, we have (neglecting vectorization) that $\parameterSpaceTuple=\parameterSpace^K$, where
    \begin{equation*}
        \parameterSpace = \mathset{(\beta,\covariance)\in \RR^{d}\times \RR^{d\times d}\setmid \norm{\beta}_0\leq s,\norm{\beta}_2\leq 1, \covariance \text{ is symmetric and } b^2 \identity_d \preceq \covariance \preceq B^2  \identity_d}.
    \end{equation*}
    The lower bound on $\covariance$ was assumed, and the upper bound holds because $X$ was assumed to be $B^2$-sub-Gaussian, implying that the largest eigenvalue is bounded as
    \begin{equation*}
        \eigenvalueMax(\EE\br{XX^\top}) =  \sup_{\norm{v}_2=1} v^\top \EE\br{XX^\top} v  = \sup_{\norm{v}_2=1}\EE\br{\inner{X}{v}^2}\leq B^2.
    \end{equation*}

    To apply \cref{thm:MinimaxRateStrongConvexity,thm:GeneralLowerBound}, we show that \cref{ass:StrongConvexitySmoothness}, \ref{ass:LocallyLipschitz} and \ref{ass:Injectivity} hold with the sets $\parameterSpaceTupleLarge =\parameterSpaceLarge^K$, where
    \begin{align*}
        \parameterSpaceLarge &= \mathset{(\beta,\covariance) \in \RR^d \times \RR^{d\times d}\setmid \norm{\beta}_2\leq 2, \covariance \text{ is symmetric and } (1/2) b^2 \identity_d \preceq \covariance \preceq 2B^2  \identity_d}, 
    \end{align*}
    and then use that our estimates lie in these sets with high probability.
    
    \textbf{Upper bound.} 
    %\begin{quote}
    %\vspace{-0.8cm}
    %\StrongConvexitySmoothness*
    %\end{quote}
    It is easy to show that \cref{ass:StrongConvexitySmoothness} holds.
    Recall that the objectives take the form
    \begin{equation*}
        \objectiveindexed(\vartheta,\beta_k,\covariance_k)=\norm{\covariance_\objectiveindex^{1/2}(\vartheta-\beta_\objectiveindex)}_2^2+\sigma^2 \quad \text{where} \quad (1/2)b^2\identity_d \preceq \covariance_k\preceq 2B^2\identity_d.
    \end{equation*}
    Hence $\vartheta\mapsto \objectiveindexed(\vartheta,\minimizer_\objectiveindex,\covariance_\objectiveindex)$ is strongly convex with parameter $\strongConvexityParam_k=(1/2)b^2$ \cite[\S 3.4]{Bubeck2015Convex} (and smooth with parameter $4B^2$). Recall \cref{ass:LocallyLipschitz}.
    %\begin{quote}
    %    \vspace{-0.8cm}
    %    \LocallyLipschitz*
    %\end{quote}
    Now, the map
    $(\beta_\objectiveindex,\covariance_\objectiveindex)\mapsto \nabla_\vartheta\objectiveindexed(\vartheta,\beta_\objectiveindex,\covariance_\objectiveindex)$ is locally Lipschitz on $\parameterSpaceLarge$, since
    \begin{align*}
        \norm{\nabla_\vartheta\objectiveindexed(\vartheta,\beta_\objectiveindex,\covariance_\objectiveindex) - \nabla_\vartheta\objectiveindexed(\vartheta,\beta_\objectiveindex',\covariance_\objectiveindex')}_2 
        &= 2\norm{\covariance_\objectiveindex(\vartheta-\beta_\objectiveindex)-\covariance_\objectiveindex'(\vartheta-\beta_\objectiveindex')}_2\\
        &= 2\norm{\pr{\covariance_\objectiveindex-\covariance_\objectiveindex'}(\vartheta-\beta_\objectiveindex)+\covariance_\objectiveindex'(\beta_\objectiveindex-\beta_\objectiveindex')}_2\\ 
        &\leq 2\norm{\vartheta-\beta_k}_2\norm{\covariance_\objectiveindex-\covariance_\objectiveindex'}_2+2\norm{\covariance_\objectiveindex'}_2\norm{\beta_\objectiveindex-\beta_\objectiveindex'}_2 \\
        &\leq 2\max\mathset{\norm{\vartheta-\beta_k}_2,\norm{\covariance_k'}_2} \pr{\norm{\covariance_k-\covariance_k'}_2+\norm{\beta_k-\beta_k'}_2} \\
        &\leq \underbrace{2\max\mathset{\norm{\vartheta}_2+2,2B^2}}_{\zeta_k(\vartheta)}\pr{\norm{\covariance_k-\covariance_k'}_2+\norm{\beta_k-\beta_k'}_2}
    \end{align*}
    where the local Lipschitz constant $\zeta_k(\vartheta)$ depends on $\vartheta$ and $B$. Hence, \cref{ass:LocallyLipschitz} is satisfied with the norm $\norm{\theta_k}=\norm{(\beta_k,\covariance_k)}=\norm{\beta_k}_2+\norm{\covariance_k}_2$.

    We now prove that the estimated parameters $\estimatedParameter_k=(\betahat_k,\estimatedCovariance_k)$, where $\estimatedCovariance_k$ is the sample covariance matrix and $\betahat_k$ the LASSO estimate, are contained in $\parameterSpaceLarge$ with high probability for all $k$.
    Recall the definition of $\estimatedCovariance_k$: Denote $\X_k$ the covariate sample matrix that has the labeled samples $X^\objectiveindex_i$, $i=1,\dots,n_k$ as its rows, $\widetilde\X_\objectiveindex$ the sample matrix that has the unlabeled samples $X^\objectiveindex_i$, $i=n_k+1,\dots,N_k$ as its rows, and $y^k=(Y^\objectiveindex_1,\dots,Y^\objectiveindex_{n})^\top$ according to \eqref{eq:sparse-multi-distribution-model}.
    We define the sample covariance matrix as 
    \begin{equation*}
        \estimatedCovariance_\objectiveindex = \frac{1}{n_k+N_k}
        \begin{pmatrix}
            \X_\objectiveindex^\top & \widetilde\X_\objectiveindex^\top
        \end{pmatrix} 
        \begin{pmatrix}
            \X_\objectiveindex \\
            \widetilde\X_\objectiveindex
        \end{pmatrix}
        .
    \end{equation*}
    We let $\betahat_k$ be the LASSO estimates with $\ell_1$-penalty $\alpha\norm{\cdot}_1$, where $\alpha_k =136 B\sigma\sqrt{\log (d) / n_k}$, that is,
    \begin{equation*}
        \betahat_\objectiveindex = \argmin_{\beta\in\RR^d}\frac{1}{n_k}\norm{\X_\objectiveindex \beta-y^\objectiveindex}_2^2 +\alpha_k\norm{\beta}_1.
    \end{equation*}
    For some universal constants $c_1>0$, we define the event
    \begin{equation*}
        \cE_1 = \mathset{ \forall k\in[K]:\  \norm{\estimatedCovariance_\objectiveindex-\covariance_\objectiveindex}_2 \leq c_1 B^2\sqrt{\frac{d}{n_k+N_k}}}.
    \end{equation*}
    By \citet[Theorem 6.5]{Wainwright2019}, or \citet[Corollary 5.50]{Vershynin2010introduction}, and a union bound, $\cE_1$ has probability at least $1-c_2K\cdot \exp\pr{-c_3 B^{4}d}$, where $c_2,c_3>0$ are two more universal constants. On $\cE_1$, it holds that
    \begin{align*}
        \eigenvalueMin(\estimatedCovariance_k) &\geq \eigenvalueMin(\covariance_k)-\norm{\estimatedCovariance_k-\covariance_k}_2 \geq b^2-c_1 B^2\sqrt{\frac{d}{n_k+N_k}} \geq \frac{b^2}{2}, \\
        \eigenvalueMax(\estimatedCovariance_k) &\leq \eigenvalueMax(\covariance_k)+\norm{\estimatedCovariance_k-\covariance_k}_2 \leq B^2+c_1 B^2\sqrt{\frac{d}{n_k+N_k}} \leq 2B^2,
    \end{align*}
    where we used the assumption that $c_1 B^2\sqrt{d/(n_k+N_k)}\leq b^2/2 \wedge B^2 $ (the second follows from $n_k+N_k \geq  d / c_1^2$).
    We also define for some universal constant $c_4>0$ the event
    \begin{equation*}
        \cE_2 = \mathset{\forall k\in [K]: \norm{\betahat_k-\beta_k}_2 \leq c_4 \frac{B\sigma }{b} \sqrt{\frac{s\log d}{n_k}} }
    \end{equation*}
    By \citet[Theorem 7.19]{Wainwright2019}, this event holds whenever for all $k\in [K]$ we have
    \begin{equation}
    \label{eq:lasso-penalty-random}
        \frac{2}{n}\norm{\X_k^\top\xi^k}_\infty \leq \alpha_k = 136 B\sigma \sqrt{\frac{\log d}{n_k}}.
    \end{equation}
    which, we now show holds with probability at least $1-K/(4d^3)$. 
    Since $\X_k$ and $\xi^k$ are independent, we have that $(\X_k^\top \xi^k | \X_k)  \sim \cN(0,\sigma^2 \norm{X^k_{:,j}}_2^2)$ (where $X^k_{:,j}$ is the $j$-th column of $\X_k$), and denoting $\kappa =\max_{j\in[d]} \norm{X^k_{:,j}}_2/\sqrt{n_k}$, we get
    \begin{equation*}
        \PP\pr{\frac{2}{n}\norm{\X_k^\top \xi^k}_\infty \leq 8\sigma  \pr{\frac{\kappa}{\sqrt{n_k}}} \sqrt{\frac{\log d}{n_k}}\setmid \X_k} \geq 1-\frac{1}{8d^3}.
    \end{equation*}
    It remains to bound $\kappa /\sqrt{n_k}$, which we can do using $\rho(\covariance_k)=\max_{j\in [d]}\sqrt{(\covariance_k)_{jj}}\leq \sqrt{\eigenvalueMax(\covariance_k)}\leq B$ as
    \begin{equation*}
        \PP\pr{\frac{\kappa}{\sqrt{n_k}} \leq B\pr{1+16\sqrt{\frac{ \log d}{n_k}}}}\geq 1- d\exp\pr{-4\sqrt{n\log d}},
    \end{equation*}
    cf.\ \citet[Equation (23) and Appendix I]{Raskutti2011minimax}. Therefore, if $n_k\geq \log d$, we have with probability at least $1-1/(4d^3)$ that \cref{eq:lasso-penalty-random} holds for every $k\in[K]$ separately, and taking union bound yields $\PP(\cE_2)\geq 1-K/(4d^3)$.
    
    Thus, on $\cE_2$, we also have that for all $k\in [K]$, $\|\betahat_k\|_2\leq 2$, because we assumed that $c_4 \frac{B\sigma }{b} \sqrt{\frac{s\log d}{n_k}}\leq 1$.

    Consequently, on $\cE_1\cap \cE_2$ we have that for all $k\in[K]$, $\estimatedParameter_k=(\betahat_k,\estimatedCovariance_k)\in \parameterSpaceLarge$. Hence, we may apply the upper bound from \cref{thm:MinimaxRateStrongConvexity} with $\scalarization(\strongConvexityParamTuple)=(1/2)b^2$, that is,
    \begin{align*}
        \norm{\tsEstimatedMinimizerWeighted-\minimizerWeighted}_2&\leq \frac{2\zeta(\minimizerWeighted)}{b^2} \sum_{\objectiveindex=1}^K \weight_k\pr{\|\betahat_\objectiveindex-\beta_\objectiveindex\|_2+\norm{\estimatedCovariance_\objectiveindex-\covariance_\objectiveindex}_2} \\
        &\lesssim \frac{\zeta(\minimizerWeighted)}{b^2} \sum_{\objectiveindex=1}^K \weight_k\pr{\frac{B\sigma }{b} \sqrt{\frac{s\log d}{n_k}} +  B^2\sqrt{\frac{d}{n_k+N_k}}}
    \end{align*}
    
    Notably, we can turn it into an explicit uniform upper bound by bounding $\minimizerWeighted=\minimizerWeighted(\parameterTuple)$ as
    \begin{equation*}
        \sup_{\parameterTuple\in \parameterSpaceTuple} \zeta(\minimizerWeighted) \leq \sup_{\parameterTuple\in \parameterSpaceTuple} 2\max\mathset{\norm{\minimizerWeighted(\parameterTuple)}_2+2,2B^2} =2\max\mathset{\frac{B^2}{b^2} + 2, 2B^2}
    \end{equation*}
    where we used the simple bound from \ref{item:mixture-quadratics}
    \begin{equation*}
        \sup_{\parameterTuple\in \parameterSpaceTuple}\norm{\minimizerWeighted(\parameterTuple)}_2 = \sup_{\parameterTuple\in \parameterSpaceTuple} \norm{\pr{\sum_{k=1}^K\weight_k \covariance_k}^{-1}\pr{\sum_{k=1}^K\weight_k \covariance_k\beta_k}}_2 \leq \frac{B^2}{b^2}. 
    \end{equation*}
    Hence, using $2\max\mathset{\frac{B^2}{b^2} + 2, 2B^2}/b^2\leq 6\frac{B^2}{b^2}\max\mathset{\frac{1}{b^2},1}=  6\frac{B^2}{b^4}$ since $b\leq 1$, the bound becomes
    \begin{align*}
        \norm{\tsEstimatedMinimizerWeighted-\minimizerWeighted}_2&\lesssim \frac{B^2}{b^4} \sum_{\objectiveindex=1}^K \weight_k\pr{\frac{B\sigma }{b} \sqrt{\frac{s\log d}{n_k}} +  B^2\sqrt{\frac{d}{n_k+N_k}}} \\
        &\leq \frac{B^4}{b^4} \sum_{\objectiveindex=1}^K \weight_k\pr{\frac{\sigma }{b^2} \sqrt{\frac{s\log d}{n_k}} + \sqrt{\frac{d}{n_k+N_k}}}
    \end{align*}
    
    which concludes the first part of the proof.

    \textbf{Lower bound.}
    Throughout, denote $\covarianceTuple$ as the tuple of fixed covariance matrices that are known to the algorithm, each satisfying $b^2\identity_d\preceq \covariance_k \preceq B^2\identity_d$. We first have to check that \cref{ass:Injectivity} holds. 
    %\begin{quote}
    %\vspace{-0.8cm}
    %\Injectivity*
    %\end{quote}
    Note that $\nabla_\vartheta\objectiveindexed(\vartheta,\beta_\objectiveindex,\covariance_\objectiveindex) = 2\covariance_k(\vartheta-\beta_\objectiveindex)$, so that $g^{-1}_\objectiveindex(y,\vartheta) = \vartheta -\frac{1}{2}\covariance_\objectiveindex^{-1} y$
    and hence
    \begin{align*}
        \norm{g_\objectiveindex(\beta_k,\vartheta)-g_\objectiveindex(\beta_k',\vartheta')}_2&= \norm{2\covariance_k(\vartheta-\beta_\objectiveindex)-2\covariance_k(\vartheta'-\beta_\objectiveindex')}_2\leq 2B^2 \pr{\norm{\vartheta-\vartheta'}_2+\norm{\beta_k-\beta_k'}_2},\\
        \norm{g^{-1}_\objectiveindex(y,\vartheta)-g^{-1}_\objectiveindex(y',\vartheta')}_2 &= \norm{\vartheta -\frac{1}{2}\covariance_\objectiveindex^{-1} y-\vartheta' +\frac{1}{2}\covariance_\objectiveindex^{-1} y'}_2 \leq \norm{\vartheta-\vartheta'}_2 + \frac{1}{2b^2}\norm{y-y'}_2,
    \end{align*}
    and therefore \cref{ass:Injectivity} holds with $\eta_k=2B^2$ and $\eta_k'=\max\mathset{1,1/(2b^2)}=1/(2b^2)$ as we assume $b^2\leq 1/2$.
    We can hence apply the lower bound from \cref{thm:MinimaxRateStrongConvexity}. For that, we use \citet[Example 15.16]{Wainwright2019}, which is an application of Fano's method \citep{Yu1997assouad} with local packings. Conditioned on $\X_k$, it yields the minimax lower bound
    \begin{equation*}
        \inf_{\betahat_k}\sup_{\substack{\norm{\beta_k}_0\leq s, \\ \norm{\beta_k}_2\leq 1} } \EE\br{\norm{\betahat_k-\beta_k}_2\setmid \X_k} \gtrsim \frac{\sigma}{\gamma(\X_k)} \sqrt{\frac{s\log d}{n_k}}
    \end{equation*}
    where $\gamma(\X_k)=\max_{\abs{S}=2s} \sigma_{\max}((\X_k)_S/\sqrt{n})$ and $\sigma_{\max}$ denotes the largest singular value. For any $\abs{S}=2s$, we can bound this using \cite[Equation (6.16)]{Wainwright2019}
    \begin{equation*}
        \EE\br{\sigma_{\max}((\X_k)_S/\sqrt{n})} \leq \sqrt{\eigenvalueMax((\covariance_k)_S)}+\sqrt{\frac{\tr{(\covariance_k)_S}}{n_k}}
    \end{equation*}
    As $\covariance_k\preceq B^2\identity_d$, this is bounded by $B+B\sqrt{s/n_k}=B(1+\sqrt{s/n_k})$. Hence, if $n_k\geq s$, we have by Jensen's inequality
    \begin{equation*}
        \inf_{\betahat_k}\sup_{\substack{\norm{\beta_k}_0\leq s, \\ \norm{\beta_k}_2\leq 1} } \EE\br{\EE\br{\norm{\betahat_k-\beta_k}_2\setmid \X_k}} \gtrsim \frac{\sigma}{\EE\br{\gamma(\X_k)}} \sqrt{\frac{s\log d}{n_k}} \gtrsim \frac{\sigma}{B} \sqrt{\frac{s\log d}{n_k}}.
    \end{equation*}
    Therefore, combining this lower bound with the upper bound on the minimax rate from the previous section of the proof, \cref{thm:MinimaxRateStrongConvexity} yields that if for some $k$, with large enough constant $C>0$ it holds
    \begin{equation*}
        \weight_k b^2 \frac{\sigma}{B} \sqrt{\frac{s\log d}{n_k}} \geq C \sum_{i\neq k} \weight_i  \frac{\sigma B}{b} \sqrt{\frac{s\log d}{n_i}},
    \end{equation*}
    we get using $(1+\scalarization(\etabold))^{-1}\geq 1/(1+2B^2)\gtrsim 1/B^2$ that
    \begin{equation*}
        \minimax(\distributionSetTuple) \gtrsim  \max_{k\in[K]}\weight_k\frac{b^2\sigma}{B^3} \sqrt{\frac{s\log d}{n_k}}.
    \end{equation*}
    As we assumed the above condition, up to canceling and rearranging terms, this finished the proof.
\end{proof}

\subsection{Proof of Corollary \ref{cor:FairnessCorollary}}
\label{subsec:proof-FairnessCorollary}

%\StrongConvexitySmoothness*

%\LocallyLipschitz*

\FairnessCorollary*

\begin{proof}     
        As before, to be able to apply the upper bound from \cref{thm:MinimaxRateStrongConvexity}, we need to check that \cref{ass:StrongConvexitySmoothness} and \ref{ass:LocallyLipschitz} hold. 
         
        First, notice that in the specific distribution that we are considering we have
        \begin{align*}
            \EE\br{XX^\top} &= \frac{1}{2}\EE\br{XX^\top | A=1} + \frac{1}{2}\EE\br{XX^\top | A=-1} \\
            &=\frac{1}{2}\pr{\EE\br{(X-\mu)(X-\mu)^\top | A=1}+ \mu_1\mu_1^\top} + \frac{1}{2}\pr{\EE\br{(X+\mu)(X+\mu)^\top | A=-1}+ \mu\mu^\top} \\
            &= \identity_d + \mu\mu^\top.
        \end{align*}
        and hence we can write
        \begin{align*}
            \Lrisk(\minimizer,\beta,\mu) = \norm{(\identity_d+\mu\mu^\top)^{1/2}(\minimizer-\beta)}_2^2+\sigma^2 \implies \nabla_\minimizer \Lrisk(\vartheta,\beta,\mu) = 2(\identity_d+\mu\mu^\top)(\minimizer-\beta).
        \end{align*}
        Moreover, by \cref{lem:fairness-objective-rewrite}, we can write $\Lfair(\minimizer,\distributionXxY)$ as
        \begin{equation*}
            \Lfair(\vartheta,\mu) = \inner{\mu}{\vartheta}^2 \implies \nabla\Lfair(\vartheta,\mu)=2\inner{\mu}{\vartheta}\mu.
        \end{equation*}
        We may define $\parameterSpaceTuple=\parameterSpace_{\riskOp}\times \parameterSpace_{\fairOp}$ and $\parameterSpaceTupleLarge=\parameterSpaceLarge_{\riskOp}\times \parameterSpaceLarge_{\fairOp}$, where
        \begin{align*}
            \parameterSpaceTuple &= \mathset{(\parameter_{\riskOp},\parameter_{\fairOp})=((\beta,\mu),\mu)\in \RR^{3d} \setmid \norm{\beta}_0\leq s, \norm{\beta}_2\leq 1, \norm{\mu}_2\leq 1}, \\
            \parameterSpaceTupleLarge &= \mathset{(\parameter_{\riskOp},\parameter_{\fairOp})=((\beta,\mu),\mu)\in \RR^{3d} \setmid \norm{\beta}_2\leq 2, \norm{\mu}_2\leq 2}.
        \end{align*}
        
        Clearly, $\vartheta\mapsto \Lrisk(\vartheta,\beta,\mu)$ is $1$-strongly convex \cite[\S3.4]{Bubeck2015Convex} and continuously differentiable for all $((\beta,\mu),\mu)\in \parameterSpaceTupleLarge$, cf. proof of \cref{cor:MultiDistributionOurs}. Moreover, the map $(\beta,\mu)\mapsto \nabla_\vartheta \Lrisk(\vartheta,\beta,\mu)$ is locally Lipschitz continuous on $\parameterSpaceTupleLarge$, as
        \begin{align*}
            \norm{\nabla_\vartheta \Lrisk(\vartheta,\beta,\mu)-\nabla_\vartheta\Lrisk(\vartheta,\beta',\mu')}_2 &=\norm{2(\identity_d+\mu\mu^\top)(\minimizer-\beta)-2(\identity_d+\mu'\mu'^\top)(\minimizer-\beta')}_2 \\
            &= 2\norm{(\mu\mu^\top-\mu'\mu'^\top)(\minimizer-\beta) + (\identity_d+\mu'\mu'^\top)(\beta'-\beta)}_2 \\
            &\leq 2(\norm{\mu\mu^\top-\mu'\mu'^\top}_2\norm{\minimizer-\beta}_2 + \norm{\identity_d+\mu'\mu'^\top}_2\norm{\beta'-\beta}_2) \\
            &\leq \underbrace{(16+8\norm{\vartheta}_2)}_{=:\zeta_{\riskOp}(\vartheta)}\pr{\norm{\mu-\mu'}_2+\norm{\beta-\beta'}_2}
        \end{align*}
        where we used $\norm{\vartheta-\beta}_2\leq 2+\norm{\vartheta}_2$ and $\norm{\identity_d+\mu'\mu'^\top}_2\leq 3$, as well as $\norm{\mu\mu^\top-\mu'\mu'^\top}_2 \leq \norm{\mu-\mu'}_2(\norm{\mu}_2+\norm{\mu'}_2)\leq 4\norm{\mu-\mu'}_2$.
        Therefore, the risk objective satisfies the conditions from \cref{ass:StrongConvexitySmoothness} and, in particular, is strongly convex with $\weightrisk>0$ (by assumption), so that the fairness objective does not need to be strongly convex.
        
         %Hence, the optimization problem from \eqref{eq:stage-2} is independent of $\estimatedMinimizer_{\text{fair}}$ and so is the bound from \cref{thm:MinimaxRateStrongConvexity}. 
        %The reader may think of this as setting $\minimizer_{\text{fair}}=0$ (which is always a minimizer) and also later on always choose $\estimatedMinimizer_{\text{fair}}=0$.

        Now, notice that $\vartheta\mapsto \Lfair(\vartheta,\mu)$ is clearly convex and twice continuously differentiable, however, it is not strongly convex. Moreover, the map
        $\mu\mapsto \nabla_\vartheta\Lfair(\vartheta,\mu)=2\inner{\mu}{\vartheta}\mu$ is locally Lipschitz, since
        \begin{equation*}
            \norm{\nabla_\vartheta\Lfair(\vartheta,\mu) - \nabla_\vartheta\Lfair(\vartheta,\mu')}_2 = 2\norm{\inner{\mu}{\vartheta}\mu-\inner{\mu'}{\vartheta}\mu'}_2 \leq \underbrace{8\norm{\vartheta}_2}_{\zeta_{\fairOp}(\vartheta)}\norm{\mu-\mu'}_2.
        \end{equation*}
    Therefore, \cref{ass:StrongConvexitySmoothness} and \ref{ass:LocallyLipschitz} are satisfied.

    It remains to show that $\estimatedParameterTuple = ((\betahat,\muhat),\muhat) \in \parameterSpaceTupleLarge$ with high probability. 
    Recall that the two-stage estimator uses the estimator
    $\estimatedParameterTuple$ chosen like this: Denoting $\X$ as the design matrix and $y$ as the vector of noisy responses, $\betahat$ is defined as the LASSO
    \begin{equation*}
        \estimatedMinimizer \in \argmin_{\vartheta\in\RR^d}\norm{\X \vartheta-y}_2^2 +136\sigma\sqrt{\frac{\log d}{n}}\norm{\vartheta}_1.
    \end{equation*}
    To estimate $\mu$ the two-stage estimator uses the standard mean estimation
    \begin{equation*}
        \muhat : = \frac{1}{n+N}\sum_{i=1}^{n+N} A_i X_i.
    \end{equation*}
    Now, define the event
    \begin{equation*}
        \cE_1 = \mathset{\norm{\muhat-\mu}_2\leq c_1\sqrt{\frac{d}{n+N}}}
    \end{equation*}
    which holds with probability at least $1-c_2\exp\pr{-c_3 d}$ by concentration of a $\chi_d^2$-distribution with $d$ degrees of freedom. Furthermore, define the event
    \begin{equation*}
        \cE_2 = \mathset{ \norm{\betahat-\beta}_2 \leq c_4 \sigma\sqrt{\frac{s\log d}{n}} }
    \end{equation*}
    which holds with probability at least $1-c_5d^{-3}$ by the derivations in the proof of \cref{cor:MultiDistributionOurs}.
    Hence, if $n+N\gtrsim d$, on $\cE_1 \cap \cE_2$ we have that $\estimatedParameterTuple = ((\betahat,\muhat),\muhat) \in \parameterSpaceTupleLarge$.
    
    Consequently, on $\cE_1\cap \cE_2$, 
    noting that $\mu$ is shared across objectives (i.e., we can combine the two error terms), the upper bound in \cref{thm:MinimaxRateStrongConvexity} applies and since $\weightrisk+\weightfair=1$, we obtain 
    \begin{align*}
        \norm{\tsEstimatedMinimizerWeighted-\minimizerWeighted}_2 &\leq\frac{16+8\norm{\minimizerWeighted}_2}{\weightrisk} \pr{\weightrisk\norm{\betahat-\beta}_2 + \norm{\muhat-\mu}_2} \\
        &\lesssim \pr{1+\norm{\minimizerWeighted}_2}\pr{\sigma\sqrt{\frac{s\log d}{n}}+\frac{1}{\weightrisk}\sqrt{\frac{d}{n+N}}},
    \end{align*}
    which holds with probability at least $1-c_2\exp\pr{-c_3d}-c_5d^{-3}\geq 1-c d^{-3}$ by the union bound.
    Finally, we conclude the proof since by \ref{item:mixture-quadratics} we have
    \begin{equation*}
        \norm{\minimizerWeighted}_2 = \norm{\pr{\weightfair\mu\mu^\top +\weightrisk(\identity_d+\mu\mu^\top)}^{-1}\weightrisk(\identity_d+\mu\mu^\top)\beta}_2\leq 1.
    \end{equation*}
    %This concludes the proof.
\end{proof}

\subsection{Proof of Proposition \ref{prop:NecessityUnlabeledData}}
\label{subsec:proof-NecessityUnlabeledData}
\NecessityUnlabeledData*
\begin{proof}[Proof of \cref{prop:NecessityUnlabeledData}]  
    The proof follows from Fano's method \citep{Fano1961transmission,Yang1999information,Yu1997assouad} using local packings, applied to our setting. Specifically, similar to the proof of \cref{prop:InsufficiencyPluginRegularization}, we will use that for any vector $v$ in $C\delta B_2^d$ (for some constant $C>1$), we can find covariance matrices $\covariance_k$ so that the corresponding solution $\minimizerWeighted$ equals $v$. We then consider a $\delta$-packing of $C\delta B_2^d$ and apply Fano's method: To that end, we bound the mutual information between Gaussian distributions with the covariance matrices that correspond to the packing. Some calculations yield the lower bound.

    We write $N$ instead of $n+N$ for brevity throughout the proof. 
    For $t\in (0,1/4]$ recall the definition of the scaled $\ell_2$-ball $tB_2^d=\mathset{v\in\RR^d \setmid \norm{v}_2\leq t}$.
    Then for all $v,v'\in tB_2^d$ it holds $-t\leq\inner{v}{\beta}\leq t$ and $\norm{v-v'}_2\leq 2t$. Let $v_1\dots,v_M$ be a $\delta$-packing of $tB_2^d$ in $\ell_2$-norm. 
    By standard arguments \citep[Example 5.8]{Wainwright2019}, there exists such a packing with $\log M \geq d\log \pr{t/\delta}$.
    
    For any vector $v\in tB_2^d$, we now define the matrix
    \begin{equation*}
        A(v) = v\beta^\top + \beta v^\top - \inner{v}{\beta}\beta\beta^\top
    \end{equation*}
    and notice that $A(v)\beta = v + \inner{v}{\beta} \beta -\inner{v}{\beta}\beta = v$.
    Based on $A(v)$ we define the pairs of covariance matrices $\covarianceTuple^j = (\covariance_1^j,\covariance_2^j)$ with
    \begin{equation*}
         \covariance_1^j =  \identity_d + A(v_j) \qquad \text{and}\qquad \covariance_2^j = \identity_d-A(v_j).
    \end{equation*}
    For $\covarianceTuple = (\covariance_1,\covariance_2)$, recall from \ref{item:mixture-quadratics} that the multi-objective solution is given by 
    \begin{equation*}
        \minimizerWeighted(\covarianceTuple) = \pr{\covariance_1+\covariance_2}^{-1}\pr{\covariance_1\beta-\covariance_2\beta}.
    \end{equation*}
    and hence we immediately see that
    \begin{align*}
        \minimizerWeighted(\covarianceTuple^j) &= (\covariance_1^j+\covariance_2^j)^{-1}(\covariance_1^j\beta-\covariance_2^j\beta) \\
        &= (\identity_d+A(v_j)+\identity_d-A(v_j))^{-1}((\identity_d+A(v_j))\beta-(\identity_d-A(v_j))\beta) \\
        &= (2\identity_d)^{-1}(2v_j) \\
        &= v_j
    \end{align*}
    as well as the eigenvalues of $\covariance_1^j,\covariance_2^j$ lying within $1+2\inner{\beta}{v_j}\leq 1+2t\leq 3/2$ and $1-\inner{v_j}{\beta}\geq 1-2t\geq 1/2$.
    Hence, for every $i\in [M]$, there exists a pair of distributions in $\distributionSetTuple$ so that $\minimizerWeighted= v_i$. 

    We now apply a version of Fano's method \citep[Section 15.3]{Wainwright2019} that is based on local packing.
    Define $J\sim U([M])$ and $(Z| J=j) \sim \cN(0,\covariance_1^j)\times \cN(0,\covariance_2^j)$. By \citet[Proposition 15.12]{Wainwright2019}, we have for any $\delta>0$ that
    \begin{equation*}
        \minimax(\distributionSetTuple)=\inf_{\estimatedParameter}\sup_{\distributionXxYTuple\in\distributionSetTuple} \EE_{\distributionXxYTuple}\br{\norm{\estimatedParameter-\minimizerWeighted}_2} \geq \delta \pr{1-\frac{I(Z;J)+\log 2}{\log M}}
    \end{equation*}
    and $I(Z;J)$ is the mutual information between $Z$ and $J$.
    We bound $I(Z;J)$ using \citet[Lemma 15.17 and Equation (15.45)]{Wainwright2019}:
    \begin{align*}
        I(Z;J) &\leq  \frac{N}{2}\pr{\log\det\cov(Z)-\frac{1}{M}\sum_{i=1}^M\log \det(\covariance_1^i \otimes \covariance_2^i)}.
    \end{align*}
    Note that unconditionally, $Z\sim \cN(0, \frac{1}{M}\sum_{i=1}^M\covariance_1^i\otimes \covariance_2^i)$, where $A\otimes B$ denotes the matrix with the blocks $A$ and $B$ on its diagonal. Hence,
    \begin{align*}
        &\log\det\cov(Z)-\frac{1}{M}\sum_{i=1}^M\log \det(\covariance_1^i \otimes \covariance_2^i) \\
        &= \log \det\pr{\frac{1}{M}\sum_{i=1}^M\covariance_1^i\otimes \covariance_2^i} -\frac{1}{M}\sum_{i=1}^M\log \det(\covariance_1^i \otimes \covariance_2^i) \\
        &=  \log \det\pr{\frac{1}{M}\sum_{i=1}^M\covariance_1^i}+\log \det\pr{\frac{1}{M}\sum_{i=1}^M \covariance_2^i} -\frac{1}{M}\sum_{i=1}^M\log \det(\covariance_1^i ) -\frac{1}{M}\sum_{i=1}^M\log \det( \covariance_2^i) 
        %&=\pr{\log \det\pr{\frac{1}{M}\sum_{i=1}^M\covariance_1^i} -\frac{1}{M}\sum_{i=1}^M\log \det(\covariance_1^i )}+\pr{\log \det\pr{\frac{1}{M}\sum_{i=1}^M \covariance_2^i}  -\frac{1}{M}\sum_{i=1}^M\log \det( \covariance_2^i)}
    \end{align*}
    We now apply the following lemma, which is proved after concluding the main proof.
    \begin{lemma}
    \label{lem:mutual-information-bound}
        For $\norm{\beta}_2=1$, $v\in tB_2^d$ with $t\in (0,1/4]$, define the matrix
        \begin{equation*}
            A(v) = v\beta^\top + \beta v^\top -\inner{v}{\beta} \beta \beta^\top
        \end{equation*}
        and define for $v_i\in tB_2^d$, $i\in\mathset{1,\dots,M}$ the matrices $\covariance^i_1 = \identity_d + A(v_i)$ and $\covariance_2^i = \identity_d - A(v_i)$.
        Then, for $j\in\mathset{1,2}$, it holds that
        \begin{equation*}
            0\leq \log\det\pr{\frac{1}{M}\sum_{i=1}^M \covariance_j^i}-\frac{1}{M}\sum_{i=1}^M \log\det\pr{\covariance_j^i} \leq 32 t^2.
        \end{equation*}
    \end{lemma}
    Therefore, we know that $I(Z,J)\leq 32 N t^2$.
    Combining this with $\log M \geq d\log\pr{t/\delta}$ we have that
    \begin{align*}
        \inf_{\estimatedMinimizerWeighted}\sup_{\distributionXxYTuple\in\distributionSetTuple} \EE_{\distributionXxYTuple}\br{\norm{\estimatedMinimizerWeighted-\minimizerWeighted}_2} &\geq \delta \pr{1-\frac{I(Z;J)+\log 2}{\log M}} \\
        &\geq \delta \pr{1-\frac{32 N t^2 +\log 2}{d\log(t/\delta)}}
    \end{align*}
    Choosing $\delta = \sqrt{\frac{d}{512 N }}$ and $t=e\delta$ (which, by assumption, is smaller than $1/4$), since $d\geq 3$ we have
    \begin{align*}
        \inf_{\estimatedMinimizerWeighted}\sup_{\distributionXxYTuple\in\distributionSetTuple} \EE_{\distributionXxYTuple}\br{\norm{\estimatedMinimizerWeighted-\minimizerWeighted}_2} &\geq \delta \pr{1-\frac{32 N e^2\frac{d}{512N}}{d}-\frac{\log 2}{d}} \\
        &\geq \delta \pr{1-\frac{e^2}{16}-\frac{1}{4}} \\
        &\geq \frac{1}{4}\delta \asymp \sqrt{\frac{d}{N}}
    \end{align*}
    which concludes the proof by recalling that we replaced $N+n$ by just $N$.
\end{proof}

\begin{proof}[Proof of \cref{lem:mutual-information-bound}]
Denote the function $f(A)=\log\det A$. 

We begin by noting that $f$ is concave on the positive semidefinite cone \citep{Rivin2002another}, directly implying the lower bound, as it is merely Jensen's inequality. 

We now show the upper bound. 
Since $f$ is twice differentiable on the set of positive definite matrices, the second-order Taylor expansion in integral form yields that for any two such matrices $A$ and $B$,
\begin{equation*}
    f(A)=f(B)+\inner{\nabla f(B)}{A-B} + \int_0^1 (1-s)\nabla^2f(B+s(A-B))[A-B,A-B]ds.
\end{equation*}
A computation shows that the gradient and Hessian of $f$ are given by
\begin{equation*}
    \nabla f(A) = A^{-1} \qquad \text{and} \qquad \nabla^2 f(A)[H,H] = -\tr{A^{-1} H A^{-1} H}.
\end{equation*}

If we take $A = \covariance_1^i$, $B = \bar{\covariance}_1$ and $H = \covariance_1^i- \bar{\covariance}_1$, we get that
\begin{equation}
\label{eq:bound-log-det-1}
\begin{aligned}
    f(\covariance_1^i) &= f(\bar{\covariance}_1) +\inner{\nabla f(\bar{\covariance}_1)}{\covariance_1^i- \bar{\covariance}_1} \\
    &\hspace{2cm }+\int_{0}^1 (1-s) \tr{(\bar\covariance_1+s(\covariance_1^i-\bar\covariance_1))^{-1} (\covariance^i_1-\bar \covariance_1) (\bar\covariance_1+s(\covariance_1^i-\bar\covariance_1))^{-1} (\covariance^i_1-\bar \covariance_1)}ds 
\end{aligned}
\end{equation}
Notice how $\frac{1}{M}\sum_{i=1}^M \inner{\nabla f(\bar \covariance_1)}{\covariance_1^i-\bar \covariance_1}=\inner{\nabla f(\bar \covariance_1)}{\bar \covariance_1-\bar \covariance_1}=0$. Hence, averaging over $i$ on both sides of \cref{eq:bound-log-det-1}, we get that
\begin{equation}
\label{eq:bound-log-det-2}
\frac{1}{M}\sum_{i=1}^M f(\covariance_1^i) = f(\bar{\covariance}_1)  - \frac{1}{M}\sum_{i=1}^M \int_0^1 (1-s) \tr{(\bar\covariance_1+s(\covariance_1^i-\bar\covariance_1))^{-1} (\covariance^i_1-\bar \covariance_1) (\bar\covariance_1+s(\covariance_1^i-\bar\covariance_1))^{-1} (\covariance^i_1-\bar \covariance_1)} ds
\end{equation}
Therefore, to prove our bound, we need to bound the integral.
To that end, we bound the trace as
\begin{align*}
    \tr{(\bar\covariance_1+s(\covariance_1^i-\bar\covariance_1))^{-1} (\covariance^i_1-\bar \covariance_1) (\bar\covariance_1+s(\covariance_1^i-\bar\covariance_1))^{-1} (\covariance^i_1-\bar \covariance_1)} \leq \norm{(\bar\covariance_1+s(\covariance_1^i-\bar\covariance_1))^{-1}}^2_F \norm{\covariance^i_1-\bar \covariance_1}_F^2.
\end{align*}

We bound this Frobenius norm for our specific choice of covariance matrices. Writing
$v = \inner{v}{\beta}\beta + v_\perp$ with   $\inner{v_\perp}{\beta}=0$,
one easily verifies that
\begin{equation*}
A(v) = v\beta^\top + \beta v^\top - \inner{v}{\beta}\beta\beta^\top
= (\inner{v}{\beta}\beta + v_\perp)\beta^\top + \beta v^\top - \inner{v}{\beta}\beta\beta^\top = v_\perp \beta^\top + \beta v^\top
\end{equation*}
Thus, we can bound for $v\in B_2^d$
\begin{equation*}
\norm{A(v)}_F \leq \norm{v_\perp \beta^\top}_F + \norm{\beta v^\top}_F
= \norm{v_\perp}_2\norm{\beta}_2 + \norm{\beta}_2\norm{v}_2
\leq 2\norm{v}_2\leq 2t,
\end{equation*}
and hence also $1-2t \leq\norm{\covariance_1^i}_F \leq 1+2t$ and $1-2t \leq\norm{\covariance_2^i}_F \leq 1+2t$.
Now, for any two vectors $v,w\in tB_2^d$, by linearity
\begin{equation*}
    \norm{A(v)-A(w)}_F = \norm{A(v-w)}_F \leq 2\norm{v-w}\leq 4t.
\end{equation*}
If we define $\bar{v}=\frac{1}{M}\sum_{i=1}^M v_i\in tB_2^d$ and set $\bar{\covariance}_1=\identity_d+A(\bar{v})$, then for each $i$,
\begin{equation*}
    \norm{\covariance^i_1-\bar{\covariance}_1}_F = \norm{A(v_i)-A(\bar{v})}_F \leq 2\norm{v_i-\bar{v}} \leq 4t,
\end{equation*}
and analogously for $\bar{\covariance}_2=\identity_d-A(\bar{v})$. Thus, we can bound the Hessian term, using $t\leq 1/4$, as 
\begin{align*}
    \norm{(\bar\covariance_1+s(\covariance_1^i-\bar\covariance_1))^{-1}}^2 \norm{\covariance^i_1-\bar \covariance_1}_F^2 \leq \frac{1}{(1-2t)^2} \norm{\covariance^i_1-\bar \covariance_1}_F^2 \leq   4 \norm{\covariance^i_1-\bar \covariance_1}_F^2\leq 64t^2.
\end{align*}

Hence, finally, we get from \cref{eq:bound-log-det-2} that
\begin{align*}
\frac{1}{M}\sum_{i=1}^M f(\covariance_1^i) &= f(\bar{\covariance}_1)  - \frac{1}{M}\sum_{i=1}^M \int_0^1 (1-s) \tr{(\bar\covariance_1+s(\covariance_1^i-\bar\covariance_1))^{-1} (\covariance^i_1-\bar \covariance_1) (\bar\covariance_1+s(\covariance_1^i-\bar\covariance_1))^{-1} (\covariance^i_1-\bar \covariance_1)} ds\\
&\geq f(\bar{\covariance}_1)  - \frac{64t^2}{M}\sum_{i=1}^M \int_0^1 (1-s)  ds \\
&= f(\bar{\covariance}_1)  - 32t^2
\end{align*}
or, in other words,
\begin{equation*}
    f\pr{\frac{1}{M}\sum_{i=1}^M \covariance_1^i}-\frac{1}{M}\sum_{i=1}^M f(\covariance_1^i) \leq  32 t^2
\end{equation*}
which concludes the proof.
\end{proof}

\end{appendices}

\end{document}